\documentclass[11pt]{article}

\usepackage{graphicx}
\usepackage{natbib}

\usepackage{mathtools}
\usepackage{bm}
\usepackage{bbm}

\usepackage{graphicx}
\usepackage{amssymb}
\usepackage{xcolor}

\newcommand{\tn}[1]{\|#1\|_{2}}

\newcommand{\BE}{\mathbb{E}} 
\newcommand{\BS}{\mathbb{S}} 
\newcommand{\R}{\mathbb{R}} 
\newcommand{\BT}{\mathbb{T}} 
\newcommand{\BD}{\mathbb{D}} 
\newcommand{\BC}{\mathbb{C}} 

\newcommand{\TiaQED}{\hfill $\triangle$} 

\definecolor{mygray}{gray}{0.4}

\newcommand{\CS}{\mathcal{S}} 
\newcommand{\CA}{\mathcal{A}} %
\newcommand{\CI}{\mathcal{I}} %
\newcommand{\CC}{\mathcal{C}} %
\newcommand{\CO}{\mathcal{O}} %
\newcommand{\CM}{\mathcal{M}} %
\newcommand{\CN}{\mathcal{N}} %
\newcommand{\CD}{\mathcal{D}} %
\newcommand{\CB}{\mathcal{B}} %
\newcommand{\CK}{\mathcal{K}} %
\newcommand{\CH}{\mathcal{H}} %

\newcommand{\Blkdiag}{{\rm Blkdiag}}
\newcommand{\Tr}{{\rm Tr}} %

\newcommand{\op}{{\rm op}}

\usepackage{amsmath}
\DeclareMathOperator*{\argmax}{arg\,max}
\DeclareMathOperator*{\argmin}{arg\,min}

\newcommand{\diag}{{\rm diag}} 

\newcommand{\AngBra}[1]{\langle #1 \rangle}

\usepackage{amsthm}

\theoremstyle{definition}
\newtheorem{remark}{Remark}

\newtheorem{theorem}{Theorem}
\newtheorem{definition}{Definition}
\newtheorem{assumption}{Assumption}
\newtheorem{lemma}{Lemma}

\newtheorem{approach}{Approach}

\usepackage{apptools}
\AtAppendix{\counterwithin{lemma}{section}}
\AtAppendix{\counterwithin{definition}{section}}
\AtAppendix{\counterwithin{theorem}{section}}

\newcommand*{\QE}{\hfill\ensuremath{\square}}%

\newenvironment{pfof}[1]{\vspace{1ex}\noindent{\textbf{Proof of
			#1:}}\hspace{0.5em}} {\hfill\QE\vspace{1ex}}

\newcommand{\vct}[1]{{\bm{#1}}}
\newcommand{\x}{\vct{x}}
\newcommand{\y}{\vct{y}}
\newcommand{\br}{\vct{r}}
\newcommand{\z}{\vct{z}}
\newcommand{\w}{\vct{w}}

\newcommand{\onebb}{\mathbf{1}_m}
\newcommand{\va}{\vct{a}}
\newcommand{\ab}{\vct{a}}
\newcommand{\bb}{\vct{b}}
\newcommand{\vd}{\vct{d}}
\newcommand{\vb}{\vct{b}}
\newcommand{\bC}{\vct{C}}
\newcommand{\ve}{\vct{e}}
\newcommand{\vv}{\vct{v}}

\newcommand{\bt}{{\boldsymbol{\theta}}}

\newcommand{\bEps}{{\boldsymbol{{\varepsilon}}}}
\newcommand{\bF}{{\boldsymbol{F}}}
\newcommand{\bmu}{\boldsymbol{\mu}}
\newcommand{\bR}{\boldsymbol{R}}

\newcommand{\bA}{\boldsymbol{A}}
\newcommand{\bB}{\boldsymbol{B}}
\newcommand{\bP}{\boldsymbol{P}}
\newcommand{\bV}{\boldsymbol{V}}
\newcommand{\bQ}{\boldsymbol{Q}}
\newcommand{\bM}{\boldsymbol{M}}
\newcommand{\bU}{\boldsymbol{U}}
\newcommand{\bZ}{\boldsymbol{Z}}
\newcommand{\bg}{\boldsymbol{g}}
\newcommand{\bh}{\boldsymbol{h}}
\newcommand{\bz}{\boldsymbol{z}}

\newcommand{\bNeps}{{\boldsymbol{{\epsilon}}}}

\newcommand{\eps}{\varepsilon}

\newcommand{\BMxi}{\boldsymbol{\xi}}
\newcommand{\BMtheta}{\boldsymbol{\theta}}
\newcommand{\BMphi}{\boldsymbol{\phi}}
\newcommand{\BMXi}{\boldsymbol{\Xi}}
\newcommand{\BMPsi}{\boldsymbol{\Psi}}

\newcommand{\BMPhi}{\boldsymbol{\Phi}}

\newcommand{\distas}{\overset{\text{i.i.d.}}{\sim}}

\newcommand\scalemath[2]{\scalebox{#1}{\mbox{\ensuremath{\displaystyle #2}}}}

\newcommand{\ThrLN}[1]{{\left\vert\kern-0.25ex\left\vert\kern-0.25ex\left\vert #1 
		\right\vert\kern-0.25ex\right\vert\kern-0.25ex\right\vert}}

\usepackage{xcolor}         
\usepackage[colorlinks]{hyperref}
\hypersetup{citecolor=[RGB]{5,112,176},linkcolor=[RGB]{5,112,176},urlcolor=[RGB]{5,112,176}}

\usepackage{algorithm}
\usepackage{algpseudocode}

\usepackage{subfigure}

\title{ Stochastic Contextual Bandits with Long Horizon Rewards}

\author{ Yuzhen Qin\thanks{University of California, Riverside, \textsuperscript{$\dagger$}University of Washington, \textsuperscript{$\ddagger$}University of Michigan. \newline$~~~~~~$Emails: \{yuzhenq, yli692\}@ucr.edu, fabiopas@engr.ucr.edu, mfazel@uw.edu,  oymak@umich.edu.  \newline$~~~~~~$This paper will appear at AAAI 2023. },\quad Yingcong Li,\textsuperscript{*}\quad Fabio Pasqualetti,\textsuperscript{*}\quad Maryam Fazel,\textsuperscript{$\dagger$}\quad \vspace{8pt}
	Samet Oymak\textsuperscript{*,$\ddagger$}\\
}

\date{}

\begin{document}
\maketitle	

\begin{abstract}
	The growing interest in complex decision-making and language modeling problems highlights the importance of sample-efficient learning over very long horizons. This work takes a step in this direction by investigating contextual linear bandits where the current reward depends on at most $s$ prior actions  and contexts (not necessarily consecutive), up to a time horizon of $h$. 
	In order to avoid polynomial dependence on $h$, we propose new algorithms that leverage sparsity to discover the dependence pattern and arm parameters jointly. We consider both the data-poor ($T<h$) and data-rich  ($T\ge h$) regimes, and derive  respective regret upper bounds $\tilde O(d\sqrt{sT} +\min\{ q, T\})$ and $\tilde O(\sqrt{sdT})$,  with sparsity $s$, feature dimension $d$,  total time horizon $T$, and $q$ that is adaptive to the reward dependence pattern. 
	Complementing upper bounds, we also show that learning over a single trajectory brings inherent challenges: While the dependence pattern and arm parameters form a rank-1 matrix, circulant matrices are not isometric over rank-1 manifolds and sample complexity indeed benefits from the sparse reward  dependence  structure. Our results necessitate a new analysis to address long-range temporal dependencies across data and avoid polynomial dependence on the reward horizon $h$. Specifically, we utilize connections to the restricted isometry property of circulant matrices formed by dependent sub-Gaussian vectors and establish new guarantees that are also of independent interest.
\end{abstract}



\section{Introduction}

Multi-armed bandits (MAB) serve as a prototypical model to study exploration-exploitation trade-off in sequential decision-making (e.g., see \citet{bubeck2012regret}).  The agent needs to repeatedly make decisions by interacting with an unknown environment, aiming to maximize the cumulative reward. As a generalization of MAB, the contextual bandits allow the agent to take actions based on contextual information \cite{langford2007epoch}. Extensive studies have been conducted on contextual bandits due to its wide applications such as clinical trials, recommendation, and advertising (e.g., see \citet{woodroofe1979one,chu2011contextual,li2017provably,li2010contextual,QinACC2022,Qin2022_OJCSYS}).

Most existing work on contextual bandits assume that each reward only depends on a single action and the associated context. This action can be the one just taken (instantaneous reward) or the one taken a certain number of steps before (delayed rewards). However, in realistic decision-making scenarios, the reward generating process can have a more complex, non-Markovian nature. Multiple prior actions can jointly affect the current reward. For instance, whether a learner will 
take a course recommended by an online education platform depends not only on that course, but also on what combination of courses they have taken before. 
Recommending courses in a complicated curriculum to users with diverse backgrounds and past experiences requires accounting for the combined effects of past contexts on the current recommendation. 
Similarly, the attention mechanism  \citep{vaswani2017attention} is finding increasing success in reinforcement learning and NLP applications \citep{chen2021decision,brown2020language} and it makes predictions by assessing the similarities between current and past contexts (e.g.,~that correspond to words in a sentence or frames in a video game) and creating a history-weighted adaptive context. In connection to this, the benefit of using a long context history has been well acknowledged in RL and control theory (e.g.~frame/state stacking practice \cite{hessel2018rainbow}). These observations motivates the following central question:
\begin{itemize}
	\item[\textbf{Q:}]  Can we provably and efficiently learn from long-horizon rewards? What is the role of reward dependence structure in sample efficiency?
\end{itemize}

In this work, we thoroughly address these questions for a novel variation of stochastic linear contextual bandits problems, where the current reward depends on 
a subset of prior 
contexts, up to a time horizon of $h$ (see Fig.~\ref{conceptual} for an illustration). Specifically, the reward is determined by a \textit{filtered context} that is a linear combination of prior $h$ selected contexts. 
Moreover, inspired by practical decision making scenarios, we account for sparse interactions where only $s$ ($s\ll h$) of $h$ prior contexts actually contributing to the current reward. Here $s=1$ corresponds to the special instance of {delayed rewards}. Crucially, we develop strategies that leverage this sparse dependence structure of the reward function and establish regret guarantees for long horizon rewards.

\begin{figure*}[t]
	\centering
	\includegraphics[scale=1.7]{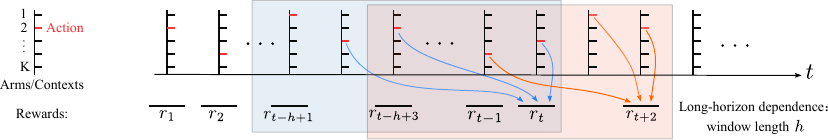}
	\caption{Contextual Bandits with Long-Horizon Rewards. The reward at each round $t$ depends on the contexts associated with latest $h$ actions ($h$ can be very large).}
	\vspace{-8pt}
	\label{conceptual}
\end{figure*}

\subsection{Related Work} 
\subsubsection{Composite anonymous rewards.} \citet{pike2018bandits} considered bandits with composite anonymous rewards, where 1) the reward that the agent receives at each round is the sum of the delayed rewards of an unknown subset of past actions, and 2) individual contributions of past actions to the reward are not discernible. \citet{cesa2018nonstochastic} generalized this setting to a case where the reward generated by an action is not simply revealed to the agent at a single instant in the future, but rather spreads over multiple  rounds.   Recent work along this line is also found in \citet{garg2019stochastic}, \citet{zhang2022gaussian}, and \citet{wang2021adaptive}. 
In this paper, we consider a contextual setting with, which is different from the above ones and poses new  challenges since each arm no longer has  a fixed reward distribution.


\subsubsection{Delayed rewards.} Bandits with delayed rewards are also related to our work. Stochastic linear bandits with random delayed rewards was studied in \citet{vernade2020linear}. \citet{li2019bandit} investigated the case where the delay is unknown. Generalized stochastic linear bandits with random delays were studied in \citet{zhou2019learning} and \citet{howson2022delayed}.  \citet{cella2020stochastic} and \citet{lancewicki2021stochastic} studied bandit problems with reward-dependent delays.  In \citet{lancewicki2021stochastic} and \citet{thune2019nonstochastic},  delays are allowed to be unrestricted. Arm-dependent delays in stochastic bandits are studied in \citet{gael2020stochastic}. Delays are also considered in adversarial bandits \citep{bistritz2019online,gyorgy2021adapting,zimmert2020optimal}. 
Recently, non-stochastic cooperative linear bandits with delays have also been studied \citep{ito2020delay,cesa2019delay}. In fact,  our setting captures  unknown fixed delays and also aggregated and anonymous delayed rewards. 

\subsubsection{Sparse parameters.} In sparse bandits, feature vectors can have large dimension $d$, but only a small subset, $s\ll d$, of them affect rewards. 
Early studies on sparse \textit{linear bandits} are found in \citet{carpentier2012bandit} and \citet{abbasi2012online}. Recent results studied both the data-poor and data-rich regimes, depending on whether the total horizon $T$ is less or larger than $d$. In the data-rich regime, \citet{lattimore2020bandit} proved a regret lower bound $\Omega(\sqrt{sdT})$. In the data-poor regime, \citet{hao2020high} showed a regret lower bound  $\Omega(s^{\frac{1}{3}}T^{\frac{2}{3}})$. A recent work used information-directed sampling techniques \citet{hao2021information}. Sparse \textit{contextual linear bandits} also receive increasing interests. \citet{kim2019doubly} proposed an algorithm that combines Lasso with doubly-robust techniques, and provided an upper bound $O(s\log(dT)\sqrt{T})$. An extended setting wherein each arm has its own parameter was studied in \citet{bastani2020online}, \citet{wang2018minimax}, where upper bounds $O(s^2\log^2(T))$ and $O(s^2\log(T))$ were shown, respectively. \citet{oh2021sparsity} proposed an exploration-free algorithm and obtained an upper bound $O(s^2\log(d)+s\sqrt{T\log(dT)})$. 
In \citet{ariu2022thresholded}, a thresholded Lasso algorithm is presented, resulting in an upper bound $O(s^2\log(d)+\sqrt{sT})$. In \citet{ren2020dynamic}, the dynamic batch learning approach was used and a upper bound $O(s\cdot{\rm polylog}(d)+\log (T) \sqrt{sT \log (d)})$ was obtained. In comparison, sparsity in our case results from the reward dependence structure.  As we will discuss in Sec.~\ref{challenges}, learning the dependence pattern is challenging since the measurements have an inherent circulant structure. 

\subsubsection{Online Convex Optimization (OCO).} Another line of research related to ours is OCO with memory where the losses depend on the past decisions taken from a convex set \citep{anava2015online,shi2020online,kumar2022online}.

\subsection{Contributions}

The contributions of this paper are summarized as follows:

1. We introduce a new contextual bandit model, motivated by realistic scenarios where rewards have a long-range and sparse dependence on prior actions and contexts. The problem of identifying the reward parameter and sparse delay pattern admits a special low-rank and sparse structure.

2. We propose two sample-efficient algorithms for the data-poor and data-rich regimes by leveraging sparsity prior. For the former, we prove a regret upper bound $O \big( d\sqrt{sT \log(dT)}+\min\{ q, T\}\big)$ that is adaptive to the reward dependence pattern described by $q$; for the latter, we obtain a regret upper bound $O(\sqrt{sdT \log(dT)})$. Note that neither of the bounds has polynominal dependence on the horizon $h$, enabling efficient learning across long horizons; and both are optimal in $T$ (up to logarithmic factors).


3. We make technical contributions to address temporal dependencies within data that has a block-Toeplitz/circulant matrix form. First, the seminal work by \citet{krahmer2014suprema} on Restricted Isometry Property (RIP) of circulant matrices assume context vectors have i.i.d.~entries. We generalize their result to milder concentrability conditions that allow dependencies. Second, we establish results that highlight the challenges of low-rank estimation unique to circulant measurements. In line with theory, numerical experiments demonstrate that our sparsity-based approach indeed outperforms low-rank ones.

\section{Problem Setting}

\textbf{Notation.}  Given $\x=[\x_1^\top,\dots,\x_h^\top]^\top \in \R^{dh}$ with each block $\x_i\in\R^d$, denote $\|\x\|_{2,1}^{(d)}:=\sum_{i=1}^h\|\x_i\|_2$; for $\bA \in \R^{m \times n}$, $\|\bA\|_\op :=\sup_{\|\x\|_2\le 1}\|\bA\x\|_2$ denotes its operator norm. Let $[n]=\{1,2,\dots,n\}$ for any integer $n$.  For any $\CS\subset [n]$, $\x_{\CS}$ denotes the sub-vector of $\x$ with entries indexed by $\CS$. Let $\AngBra{\cdot,\cdot}$ be the inner product; for $\bA$ and $\bB\in\R^{m\times n}$, $\AngBra{\bA,\bB}=\Tr(\bA^\top \bB)$. Let $\otimes$ be the Kronecker product. Given $\bA \in \R^{m\times p}$, it is said to satisfy RIP if there is $\delta\in(0,1)$ such that $(1-\delta)\|\x\|_2^2 \le \|\bA \x\|_2^2 \le (1+\delta)\|\x\|_2^2 $ holds for all $\x\in \R^p$; the smallest $\delta$ satisfying this inequality is called the RIP constant (see Appendix for more details).


\subsection{Stochastic Linear Contextual Bandits}
In this paper, we study a stochastic linear contextual bandit problem with  rewards that depend on past actions and contexts (see Fig.~\ref{conceptual} for an illustration). 
Let $K$ be the number of arms, and then the action set is $[K]$. At each round $t$, the agent observes $K$ context vectors,  $\{\x_{t,a}\in\R^d:a\in[K]\}$,  each associated with an arm and drawn i.i.d. from an unknown distribution $\nu$. It then selects an action $a_t\in[K]$ and receives a reward generated by
\begin{equation}\label{model:CB}
	r_t=\AngBra{\y_{t,a_t},\bt}+\eps_t,
\end{equation}
where
\begin{equation*}
	\y_{t,a_t}=\sum_{i=0}^{h-1}w_i\x_{t-i,a_{t-i}}, \text{ and } t \in [T].
\end{equation*}
Here $\bt\in \R^d$ is the coefficient vector, $\eps_t\in\R$ is additive noise that is zero-mean 1-sub-Gaussian. Particularly, the vector $\y_{t,a_t}$ is the \textit{filtered context}, determined by the weight vector $\w:=[w_0,w_1,\dots, w_{h-1}]^\top\in\R^h$ that describes how rewards depend on the past and current selected contexts (where $w_i\ge 0$). The range of the dependence $h$ can be very large, indicating that a reward can have a long-range contextual dependence. Assume $\x_{j,a_j}=0$ for $j\le 0$ since $a_j$'s in this case correspond to nonexistent actions. 

In this paper, we consider sparse contextual dependence, that is, the weight vector $\w$ is $s$-sparse (i.e., $\|\w\|_0\le s$) with $s\ll h$. This is particularly relevant to many realistic situations since often only a small number of past ``events'' matter.   As we mentioned before, this setting captures: a) bandits with unknown delays ($\w$ has only one non-zero entry and it is $1$-valued), and b) bandits with aggregated and anonymous rewards (all the non-zero entries of $\w$ are $1$-valued). 

The coefficient vector $\bt\in \R^d$ and the weight vector $\w \in \R^h$ are unknown. Without loss of generality, we assume that $\|\BMtheta\|_2\le 1$ and $\|\w\|_1\le 1$. We also made the mild boundedness assumption that $\x_{t,a}$ 
satisfies $\|\x_{t,a}\|_\infty \le 1$ for all $t\in[T]$ and $a\in[K]$.

The agent's objective is to maximize the cumulative reward over the course of $T$ rounds, or equivalently, to minimize the pseudo-regret defined as
\begin{equation}\label{Def:regret}
	R_T = \left [\sum_{t=1}^{T}\sum_{i=0}^{h-1}{w_i}\left( \AngBra{\x_{t,a^*_t},\bt}-\AngBra{\x_{t,a_t},\bt}\right) \right],
\end{equation}
where $a^*_t= \argmax_{a\in [K]} \AngBra{\x_{t,a},\BMtheta}$ 
defines the optimal action at round $t$. 
\begin{remark}
	The definition of the regret here is slightly different from simply summing up $\AngBra{\y_{t,a_t^*}-\y_{t,a_t},\bt}$. In fact, the two regret definitions are essentially the same. The reason is that taking an action, say $a_t$, gives the agent a total reward $\sum_{i=0}^{h-1} w_i \AngBra{\x_{t,a_t},\BMtheta}$ that spreads over the next $h$ rounds. Therefore, the agent can make decisions without knowing $\w$ if $\BMtheta$ is known as a prior: a greedy strategy seeking to maximize the instantaneous reward at each round maximizes the cumulative reward in the long run. 
	Further discussion on this point can be found in Appendix~\ref{def:regret}. \TiaQED
\end{remark}

\begin{remark}
	Although the only knowledge of $\BMtheta$ seems sufficient for our decision-making purpose,  learning $\BMtheta$ is actually challenging. This is because each reward can come in a composite manner, possibly consisting of the contributions from the latest $h$ actions.  Learning $\BMtheta$ requires to sort out the reward dependence structure. \TiaQED
\end{remark}


\subsection{Discussion of Challenges}\label{challenges}
Next, we discuss some technical challenges inherent in our problem. For this purpose, we denote $\BMxi_t=\x_{t,a_t}$ as the chosen context for each $t$ to make notation concise.

\subsubsection{Circulant design matrices in low-rank matrix recovery.}

Let $\bZ_{t}=[\BMxi_{t},\BMxi_{t-1},\dots,\BMxi_{t-h+1}]\in \R^{d\times h}$, and then \eqref{model:CB} can be rewritten as
\begin{equation}\label{model:low_rank}
	r_t=\AngBra{\bZ_{t},\BMtheta \w^\top}+\varepsilon_t. 
\end{equation}
At first glance, it seems  that the problem reduces to reconstructing the rank-$1$ and sparse matrix $\BMtheta \w^\top \in \R^{d\times h}$, and classic techniques for low-rank (and sparse) matrix recovery can be applied (e.g., \citet{richard2012estimation}, \citet{oymak2015simultaneously} \citet{davenport2016overview},  and \citet[Chap. 10]{wainwright2019high}). However, we find this is \textit{not} true due to the \textit{Toeplitz/circulant structure} of the design matrices $Z_t$.  The following lemma shows that circulant matrices, even if its first row has i.i.d. entries, do not obey RIP for rank-1 matrices with exponentially high probability (see Appendix~\ref{Failure_Matrix_RIP} for more details). 
\begin{lemma}\label{norip} 
	Let ${\bf{C}}\in\R^{n\times p}$ be a subsampled circulant matrix whose first row has i.i.d. Gaussian entries (normalized properly) and $n\leq p$. For any $\delta\leq 1$, there exists a constant $c<1$ such that with probability at least $1-c^p$, ${\bf{C}}$ does not obey RIP over rank-$1$ matrices in $\R^{p_1\times p_2}$ with $p=p_1p_2$.
\end{lemma}
We further provide numerical experiments in Fig.~\ref{low_rank} to show that circulant measurements are indeed problematic while dealing with low-rank matrix recovery.



\begin{figure}[t]
	\centering
	\includegraphics[scale=0.23]{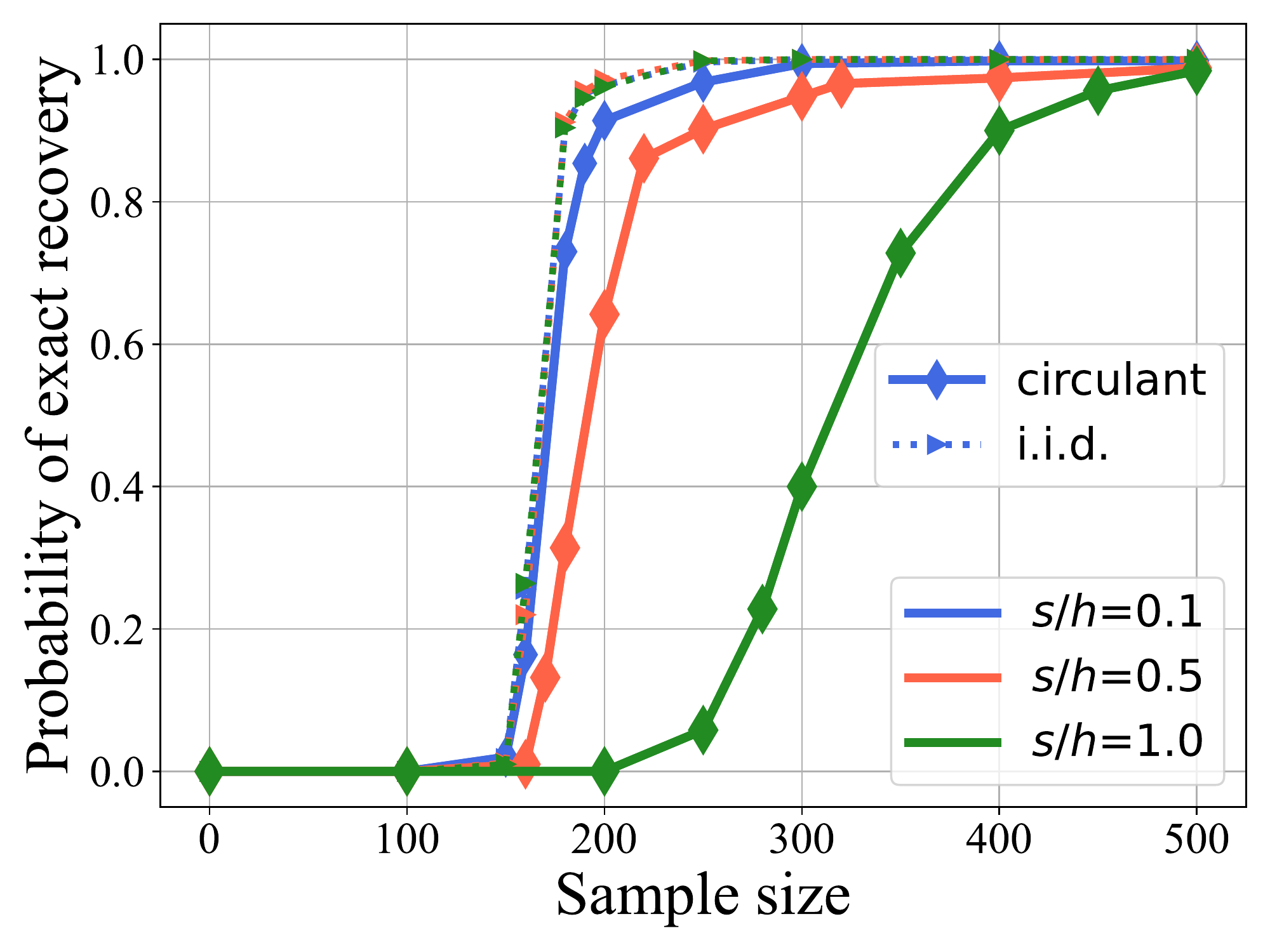}
	\caption{Probability of exact recovery  of the matrix $\Phi =  \BMtheta \w^\top  \in \R^{d\times h}$ from the noiseless measurement $\y=\AngBra{\bZ,\Phi}$ using low-rank recovery ($d=10$ and $h=100$). We compare two types of design matrices: 1) $\bZ$ with i.i.d. entries, and 2) circulant $\bZ$ generated by an i.i.d. vector. Different instances, where $\w$ has different sparsity (quantified by $s$), are considered. The experiment shows that: 1) the same amount of data is needed for all instances if $\bZ$ has  i.i.d. entries, but 2) the number of samples needed varies significantly when it comes to circulant $\bZ$. These observations indicate that for circulant measurements, the amount of data needed to recover a low-rank matrix may not simply depend on the rank, which is substantially different from i.i.d. measurements.}
	\label{low_rank}
	\vspace{-12pt}
\end{figure}

These findings indicate that tackling our problem via low-rank matrix estimation may not work. Therefore, we resort to another technique-- \textit{sparsity estimation}--  by leveraging the sparse structure in the reward dependence pattern. First, let $\va _t =\{a_t,a_{t-1},\dots,a_1\}$ denote the sequence of past actions, $\z_{\va _t}=[\BMxi_t^\top,\BMxi_{t-1}^\top, \dots,\BMxi_{t-h+1}^\top]^\top$, and $\BMphi=\w \otimes \bt \in \R^{dh}$. Then, \eqref{model:CB} can be rewritten as
\begin{equation}\label{CB:high-d}
	r_t= \AngBra{\z_{\va _t}, \BMphi}+\eps_t.
\end{equation}
Since $\w$ is $s$-sparse, reconstructing $\BMtheta$ and $\w$ becomes to estimate the $s$-block-sparse vector $\BMphi$.  

Denote $\br_t=[r_1,\dots,r_t]^\top$ and $\bEps_t=[\eps_1,\dots,\eps_t]^\top$, and it follows from \eqref{CB:high-d} that 
\begin{equation}\label{cic:equation}
	\br_t = \scalemath{0.75}{\begin{bmatrix}
			\BMxi_{1}^\top&0&\cdots&0\\
			\BMxi_{2}^\top&\BMxi_{1}^\top&\cdots&0\\
			\vdots&\vdots&\ddots&\vdots\\
			\BMxi_{h}^\top&\BMxi_{h-1}^\top&\cdots&\BMxi_{1}^\top\\
			\BMxi_{h+1}^\top&\BMxi_{h}^\top&\cdots&\BMxi_{2}^\top\\
			\vdots&\vdots&\ddots&\vdots\\
			\BMxi_{t}^\top&\BMxi_{t-1}^\top&\cdots&\BMxi_{t-h+1}^\top
	\end{bmatrix}} \BMphi +\bEps_t := \Xi_t \BMphi+ \bEps_t.
\end{equation}
One can observe that the design matrix $\Xi_t$ above also has a \textit{Toeplitz/circulant} structure. Learning the block-sparse $\BMphi$ using this special form of design matrices has some other challenges, which we discuss below. 

\subsubsection{Circulant matrices with dependent entries.} 

Estimating $\BMphi$ is a sparse regression problem. RIP and related restricted eigenvalue condition (REC) are widely used for such problems \citep{candes2007dantzig,bickel2009simultaneous}. Earlier studies show that sub-sampled circulant matrices whose first row is i.i.d. sub-Gaussian  satisfy RIP for $s$-sparse vectors if there are at least $\tilde \Omega(s\log^2(s))$ samples \citep{krahmer2014suprema}. In our case, the circulant matrix is generated by random vectors with \textit{dependent} entries (i.e., entries in each $\BMxi_t$ may be dependent). The new challenge is: how many samples are needed for such circulant measurements to satisfy RIP/REC? 

\subsection{Technical Result}\label{sec:assump}

We first present a technical result on RIP that paves the way for the analysis of our bandit problem, which is also of independent interest (see Appendix~Part I for the proof). 

\begin{theorem}\label{RIP:informal}
	Let $\BMxi_1,\dots,\BMxi_n\in \R^d$ be independent sub-Gaussian isotropic random vectors. Assume that each $\BMxi_i$ satisfies the Hanson-Wright inequality (HWI)
	\begin{align}\label{ineq:HW}
		{\Pr \left[|\BMxi_i^\top  \bA \BMxi_i -\BE (\BMxi_i^\top  \bA \BMxi_i)|\ge \eta \right]}
		{\le 2 \exp \big(- \frac{1}{c} \min \big\{ \frac{\eta^2}{k^4 \|\bA\|_F^2},\frac{\eta}{k^2 \|\bA\|_\op}\big\}\big), ~~~~~\forall \eta>0,}
	\end{align}
	for any positive semi-definite matrix $\bA\in \R^{d\times d}$, where $k$ is a constant, and $c$ is an absolute constant. Let $\BMXi\in\R^{m\times nd}$ be a matrix formed by sub-sampling \textit{any} $m$ rows from the block-circulant matrix:
	\begin{equation*}
		\bC= \scalemath{0.9}{\begin{bmatrix}
				\BMxi_{n}^\top&\BMxi_{n-1}^\top&\cdots&\BMxi_{1}^\top\\
				\BMxi_{1}^\top&\BMxi_{n}^\top&\cdots&\BMxi_{2}^\top\\
				\vdots&\vdots&\ddots&\vdots\\
				\BMxi_{n-1}^\top&\BMxi_{n-2}^\top&\cdots&\BMxi_{n}^\top
		\end{bmatrix}}.
	\end{equation*} 
	Then, for all $s$-sparse vectors, the restricted isometry constant of $\Xi$, denoted by $\delta_s$,  satisfies $\delta_s\le \delta$ if $m \ge c_1 \delta^{-2}s \log^2 (s) \log^2(nd)$ for some constant $c_1$. 
\end{theorem}

\begin{remark}
	Although we just need to reconstruct a block-sparse vector in the bandit problem, this theorem applies to \textit{general sparse} vectors. The assumption \eqref{ineq:HW} holds for many random vectors, e.g., sub-Gaussian vectors with independent entries and random vectors that obey convex concentration property \cite{adamczak2015note}. 
\end{remark}

\section{Algorithms and Main Results}\label{main_results}
Next, we present some algorithms for the bandit problem described in \eqref{model:CB}, taking into account data-poor and data-rich regimes, accompanied with their regret bounds. First, we make the following assumption.


\begin{assumption}\label{ass:cv_concentr}
	We assume that the distribution $\nu$ is such that for all $t$: (a) the context vectors, $\x_{t,a}, a\in[K]$, are i.i.d., (b) $\BE_\x=\BE[\frac{1}{K}\sum_{a=1}^K\x_{t,a} \x_{t,a} ^\top]$ satisfies $\lambda_{\min} (E_\x) \ge \sigma^2$ for some $\sigma$, and (c) $\x_{t,a}$ satisfies HWI given by \eqref{ineq:HW}. 
\end{assumption}

\begin{remark}\label{mention:weaker}
	If we let $\xi_t$ be a random vector deterministically chosen from the set $\{x_{a,t},a\in[K] \}$, Assumption~\ref{ass:cv_concentr} ensures that $\xi_t$ also satisfies the Hanson-Wright inequality (see Appendix~\ref{HW:chosen} for the proof). We will use this property with Theorem~\ref{RIP:informal} to analyze our bandit algorithms. 
	\TiaQED
	
	
\end{remark}

\subsection{Data-Poor Regime}\label{sec:poor}
First, we consider the situation where the dimension of the weight vector $\w$ is larger than the number of rounds (i.e., $T< h$). 
In this data-poor regime, one can  observe from \eqref{cic:equation} that it is impossible to reconstruct $\BMphi$. Fortunately, it is not necessary to completely learn $\BMphi$ to guide the decision-making; instead, a good estimate of $\bt$ is sufficient (see the definition of regret in \eqref{Def:regret} for the reason). Thus, we propose the following approach to \textit{partially} and \textit{gradually} learn $\BMphi$ such that  $\bt$ can be estimated and exploited at an early stage. 

\begin{algorithm}[t]
	\caption{Doubling Lasso}
	\label{alg:initial}
	\begin{algorithmic}[1]
		\State \textbf{Input:}  parameter $L$, the doubling sequence $\{T_i\}$ with $T_i=4(2^{i}-1)L$ [see Fig.~\ref{doubling} (a)], $\hat \BMtheta_0=0$.
		\For{$t=1: T$}
		\State {Observe context vectors $\{\x_{t,a}:a\in[K]\}$}.
		\State Take the greedy action $a_t\in \sup_{a\in[K]}\AngBra{\x_{t,a},\hat \BMtheta_{i-1}}$, and receive a reward $r_t$.
		\If{$t= T_i$} \\ \hfill {\color{mygray} \footnotesize \# \textit{end of the $i$th epoch, estimate a new  $\hat \BMtheta$} }
		\State Calculate $\hat \phi_{2^{i-1}L}$ according to the Lasso \eqref{Lasso:initial}.
		\State Let $\hat \BMtheta_i$ be the singular vector of $\hat \Phi_{\CS_i}$ associated with the largest singular value. 
		\EndIf
		\EndFor
	\end{algorithmic}
\end{algorithm}

\begin{approach}\label{idea:1}
	Recall that $\BMphi=\w \otimes \BMtheta$ with $\w\in\R^h$. For any integer $k\le h$, let $\w_{[k]}=[w_1,\dots,w_k]^\top$ and $\BMphi_{\CK}=\w_{[k]} \otimes \BMtheta\in \R^{kd}$.  Then, if one has learned $\hat \BMphi_\CK$ as an estimate of $\phi_\CK$, $\BMtheta$ can be estimated in the following steps:
	
	(a1) Transform vector $\hat \BMphi_{\CK}$ into a matrix $\hat \BMPhi_{\CK} \in \R^{d \times k}$ (the $i$th column of $\hat \Phi_{\CK} $ is the $i$th block of $\hat \BMphi_{\CK}$ with size $d$.
	
	(a2) Let $\hat \BMtheta$ be the left singular vector of $\hat \Phi_\CK$ associated with the largest singular value\footnote{We point out that what we estimate in this step is not exactly $\bt$, but rather its direction ${\bt}/{\|\bt\|}$. As it turns out later, a good estimate of the direction ensures a small angle between $\hat \bt$ and $\bt$, and is thus sufficient for good decision-making.}. 	\TiaQED
\end{approach}

With this approach, we use a doubling trick to design our algorithm (see Algorithm~\ref{alg:initial} and Fig.~\ref{doubling}). We select a constant $L$ satisfying $s\le L<h$  and define a sequence of sets $\CS_1,\CS_2,\dots,\CS_m$ with growing number of elements, where $\CS_i=[2^{i-1}Ld]$. Then, we aim to estimate
\begin{align*}
	{\BMphi_{\CS_1},\BMphi_{\CS_2},\BMphi_{\CS_3},\dots, \BMphi_{\CS_m}},
\end{align*}
in \textit{sequential epochs}, where $\BMphi_{\CS_i} \in \R^{2^{i-1}Ld}$ contains the first $2^{i-1}L$ blocks of $\BMphi$ and $m$ is the largest integer such that $2^{m-1}L\le h$. The main idea is to learn a small portion of $\BMphi$ when there is little data; as more data is collected, we learn a progressively larger portion.  

At each epoch $i$, the following  greedy action is repeatedly taken for $2^{i+1} L$  times (\textit{doubling trick}, see Fig.~\ref{doubling}~(a)):
\begin{equation} \label{greedy:action}
	{a_t \in \argmax\nolimits_{a\in [K]} \AngBra{\x_{t,a},\hat \BMtheta_{{i-1}}}, ~~~~~~i\ge 1},
\end{equation}
where $\hat \BMtheta_{{i-1}}$ is the estimate of $\BMtheta$ at the $(i-1)$th epoch ($\hat \BMtheta_{0}=0$).
If there are more than one greedy actions, the agent uniformly randomly picks one. Then, we collect $2^{i+1} L$ data points generated by \eqref{cic:equation}. However, we only use half of them to learn $\hat \BMtheta_{i}$. Specifically, dividing the data into four $2^{i-1}L$-dimensional chucks, we use the second and the fourth chucks (see Fig.~\ref{doubling} (a)). From \eqref{cic:equation}, the rewards in these two chucks are respectively generated by
\begin{equation}\label{data:half}
	{\br'{[i]}=\BMXi'{[i]}\BMphi +\bEps'{[i]},\hspace{10pt}
		\br''{[i]}=\BMXi''{[i]}\BMphi +\bEps''{[i]}},
\end{equation}
where $\br'{[i]},\br''{[i]} \in \R^{2^{i-1}L},\BMXi'[i],\BMXi''[i]\in \R^{2^{i-1}L \times dh}$, and $\bEps'[i],\bEps''[i] \in \R^{2^{i-1}L}$ are the corresponding reward vectors, context matrices, and noise vectors.

\begin{figure}[t]
	\centering
	\includegraphics[scale=1.5]{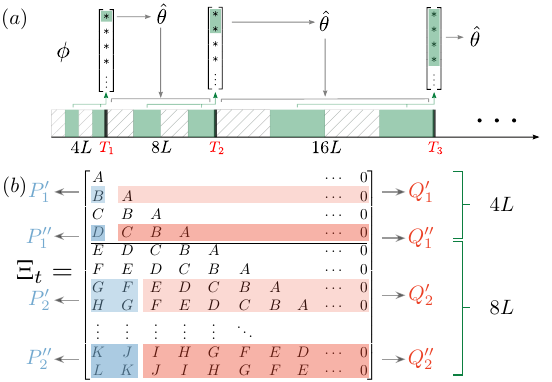}
	\caption{Illustration of Doubling Lasso. (a) In each epoch $i$, we play the greedy action in \eqref{greedy:action} for $2^{i+1}L$ rounds ($T_i$'s are the epochs' ends). Then, we use the second and fourth quarters of the collected data (green areas) to learn the first $2^{i-1}L$ block of $\BMphi$ and estimate $\BMtheta$ subsequently (see Approach~\ref{idea:1}). The learned $\hat\BMtheta$ is used for decision-making in the next epoch that has double length. (b) Illustration of how matrices in \eqref{matrices} are defined. Here, the matrix $\BMXi_t$ defined in \eqref{cic:equation} is represented in a form of $(L\times Ld)$-dimensional block matrices.
	}  
	\label{doubling}
\end{figure}

Rewrite $\BMphi=[\BMphi_{\CS_i} ^\top,  \BMphi_{\bar\CS_i}^\top]^\top$, where $\BMphi_{\CS_i}\in \R^{2^{i-1}Ld}$ is what we want to learn. Then, one can rewrite \eqref{data:half} into 
\begin{equation}\label{matrices}
	{\begin{bmatrix}
			\br'[i]\\
			\br''[i]
		\end{bmatrix}=\begin{bmatrix}
			\bP'_i& \bQ'_i\\
			\bP''_i& \bQ''_i
		\end{bmatrix} \begin{bmatrix}
			\BMphi_{\CS_i}\\
			\BMphi_{\bar \CS_i}
		\end{bmatrix} +\begin{bmatrix}
			\bEps'{[i]}\\
			\bEps''{[i]}
	\end{bmatrix}},
\end{equation}
where $\br'[i]\in\R^{2^{i-1}L}$, and $\BMXi'[i]=[\bP'_i, \bQ'_i]$ and $\BMXi''[i]=[\bP''_i, \bQ''_i]$ (see Fig.~\ref{doubling} (b) for an illustration). To learn $\BMphi_{\CS_i}$, let $\bar \br[i]=\br''[i]-\br'[i]$, $\bar \bP_i=\bP''_i-\bP'_i$, $\bar \bQ_i=\bQ''_i-\bQ'_i$, and $\bar \bEps[i]=\bEps''[i]-\bEps'[i]$, and then we have
\begin{equation}\label{parti_measure}
	{\bar \br{[i]}= \bar\bP_i \BMphi_{\CS_i} +\bar\bQ_i \BMphi_{\bar \CS_i}+\bar\bEps[i]:=\bar\bP_i \BMphi_{\CS_i} +\bNeps[i]},
\end{equation}
where the $\bar\bQ_i \BMphi_{\bar \CS_i}+\bar\bEps[i]$ is taken as the new noise $\bNeps[i]$.

Then, $\BMphi_{\CS_i}$ (which is at most $s$-block-sparse since $\BMphi$ is) is estimated by solving the block-sparsity-recovery Lasso:
\begin{equation}\label{Lasso:initial}
	{\hat \BMphi_{\CS_i}= \argmin _{\tilde \BMphi\in \R^{2^{i-1}L d}}\Big(\frac{1}{2^i L} \left\|\bar \bP_i \tilde \BMphi - \bar \br[i] \right\|^2_2 + \lambda_i\|\tilde \BMphi\|_{2,1}^{(d)} \Big)},
\end{equation}
where the regularization parameter is selected as
\begin{equation}
	{\lambda_i = cd\sqrt{\frac{2\log(2^i d L/\gamma)}{2^{i-1} L}}}
\end{equation}
for some $c>0$.
Subsequently, we use Approach~\ref{idea:1} to estimate $\BMtheta$. 
The algorithm is presented in Algorithm~\ref{alg:initial}. The following theorem provides a regret upper bound for it. 

\begin{theorem}\label{Th:data_poor}
	Consider the stochastic contextual linear bandit model with long-horizon rewards described in \eqref{model:CB}. In Algorithm~\ref{alg:initial}, choose
	$
	L= c sd \log^2 (sd) \log^2(hd),
	$
	where $c>0$ is a constant. When $T< h$, the regret satisfies
	\begin{equation}\label{upp:bound:poor}
		{R_T =  O \Big( d\sqrt{sT \log(dT)}+\min\{ q(\w), T\} \Big)},
	\end{equation}
	where $q(\w)$ is a function of the weight vector $\w$ that describes how the weights in $\w$ are distributed.  Specifically, $q(\w):=h^{\alpha (\mu)}$, where $\mu\in(0,1)$ and $\alpha(\mu)= \inf_{\alpha\in[0,1] }\|\w_{q(\w)}\|_2\ge \mu $ with $\w_{q(\w)}:=\{w_1,w_2,\dots,w_{\lceil q(\w) \rceil}\}$ and ${1}/{\mu}=\Theta(1)$.
\end{theorem}

\begin{remark}
	Notice that the following two facts in our algorithm are crucial for our analysis: 1) we use the difference between $\bP''_i$ and $\bP'_i$ (i.e., $\bar \bP_i$ in \eqref{parti_measure}) as the measurement matrix to learn $\BMphi_{\CS_i}$, ensuring that $\bar \bP_i$ has zero mean, and 2) the doubling trick and the choice of data to use ensure that $\bP''_i$ and $\bP'_i$ are \textit{non-overlapping} and \textit{independent}, and $\bar P_i$'s in different epochs are also \textit{non-overlapping} and \textit{independent} (see Fig.~\ref{doubling}~(b)).  Our analysis uses Theorem~\ref{RIP:informal} to show each $\bar \bP_i$ in \eqref{Lasso:initial} satisfies the restrictive eigenvalue condition for block-sparse vectors (see Theorem~\ref{Theorem:RE} in the Appendix). Then, we derive Theorem~\ref{bound:Lasso} that generalizes Theorem 7.13 in \citet{wainwright2019high} to complete the proof. \TiaQED
\end{remark} 

\begin{remark}
	The value of $\alpha$ describes a ``mass-like" distribution  of the weights in  $s$-sparse vector $\w$. A small $\alpha$ means that non-zero entries appear at early positions of $\w$, making it easier to learn useful information of $\BMtheta$ at an early stage than the case of a large $\alpha$. For instance, if $\alpha\le \frac{1}{2}\log_h T$ (i.e., half of the ``mass" is located at the first $\sqrt{T}$ positions of $\w$),  $R_h=\tilde  O( d\sqrt{sT})$. If $\alpha = 1$, i.e., $\w=[0,0,\dots,1]^\top$, then $R_h=  O(T)$, which is intuitive since no information can be gathered to help decision-making until the last moment of the horizon. The upper bound \eqref{upp:bound:poor} indicates that our algorithm is adaptive to different instances.  We conjecture that the dependence on $q$ is optimal; for instance, for delayed bandits, $q$ becomes the delay, which is unavoidable.
\end{remark}

\begin{figure}[t]
	\centering
	\includegraphics[scale=0.28]{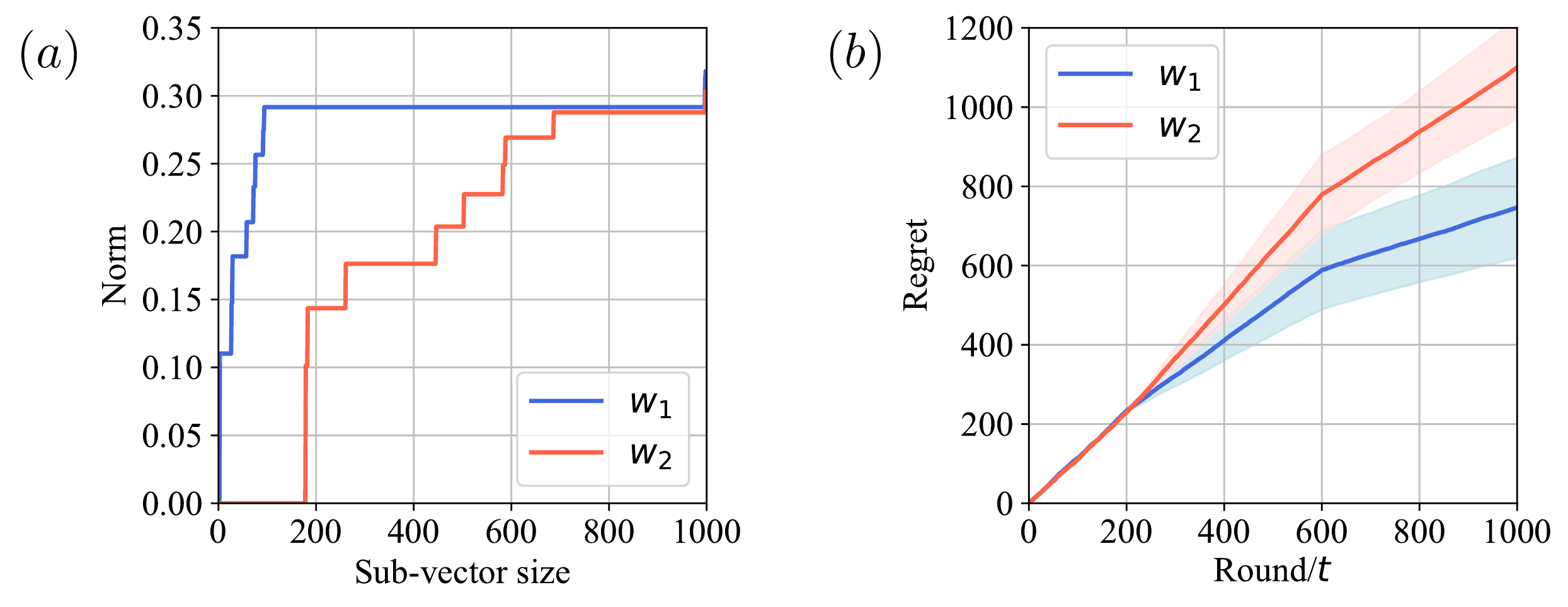}
	\caption{Regret comparison of Algorithm~\ref{alg:initial} for different instances of $\w$. (a) The $\ell_2$-norms of the sub-vectors formed by the first $k$ entries of $\w_1$ and $\w_2$, respectively ($k$ is the $x$-axis). Notice that the ``mass'' of $\w_1$ is distributed at earlier positions than $\w_2$. (b) Experiments show that Algorithm~\ref{alg:initial} results in a smaller regret for $\w_1$ than for $\w_2$, as predicted by Theorem~\ref{Th:data_poor}.  (Shaded regions show standard error in $10$ trials. Parameters: $h=1000,T=999$, and $s=10$. )}
	\label{data_poor_diff_w}
\end{figure}
\noindent\textbf{Experiments.} In Fig.~\ref{data_poor_diff_w}, we perform some experiments by considering two different $\w$'s, i.e., one with the ``mass" distributed at earlier positions and the other at later positions. As predicted by our theory, our algorithm indeed achieves a lower regret in the former case (see Fig.~\ref{data_poor_diff_w}~(b)).	

\begin{remark}\label{remark:optimal}
	Apart from the term $q(\w)=h^{\alpha}$, which is presumably \textit{unavoidable} since it measures the hardness of a problem instance, the upper bound in \eqref{upp:bound:poor} has no polynomial dependence on $h$. This means that exploring the sparsity in the reward dependence pattern is indeed beneficial especially when $sd \ll h$. \citet{hao2020high} studied a sparse linear bandit problem in the data-poor regime  and obtained an optimal bound, instantiated in our setting, $\tilde \Theta((sd)^{\frac{2}{3}} T^{\frac{2}{3}} )$. We obtained a distinct bound since we consider a different setting rather than a sparse arm parameter.    \TiaQED
\end{remark} 


\subsection{Data-Rich Regime}\label{sec:rich}
Now, we consider the situation where there are more rounds than the dimension of the weight vector $\w$, i.e., $T \ge h$. 
In this data-rich regime, we introduce an algorithm outlined in Algorithm~\ref{alg:rich} (see also Fig.~\ref{illust:rich} for an illustration). 

\begin{figure}[t]
	\centering
	\includegraphics[scale=3]{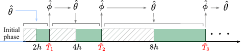}
	\caption{Illustration of AD-Lasso. For the initial phase of $h$ rounds, we use the Doubling Lasso in Algorithm~\ref{alg:initial}. For $t>h$, we also use the doubling trick, but slightly different from Algorithm~\ref{alg:initial}: 1) here $L=h$, and 2) we estimate the entire $\BMphi$ in each epoch.
	} 
	\label{illust:rich}
\end{figure}

There are two phases in this algorithm, making it adaptive: 1) in the initial $h$ rounds, we employ the Doubling Lasso (see Algorithm~\ref{alg:initial}); 2) from the $h+1$ round on, we propose another algorithm. In the second phase, we also use a doubling trick similar to Algorithm~\ref{alg:initial}. The only differences are: 1) the length of epoch $i$ is $2^{i}h$ instead of $2^{i+1}L$, 2) in each epoch, we estimate the \textit{entire} $\BMphi$ instead of a portion of it, and 3) the later half of collected data is used. 


Same as in Algorithm~\ref{alg:initial}, we  collect $2^jh$ data points in each epoch. From \eqref{cic:equation}, the $2^{j-1} h$ rewards in the later half (See Fig.~\ref{illust:rich}) are generated by
\begin{equation*}
	{\tilde \br{[j]}= \tilde \BMXi{[j]}\BMphi +\tilde \bEps{[j]}},
\end{equation*}
where $\tilde\br{[j]} \in \R^{2^j h},\tilde \BMXi[j]\in \R^{2^{j}h \times dh}$, and $\tilde \bEps[j] \in \R^{2^{j}h}$ are the corresponding reward vector, context matrix, and noise vector in the later half of the epoch $j$, respectively.

To learn $\BMphi$, we calculate the following Lasso program:
\begin{equation}\label{Lasso:rich}
	{\hat \BMphi[j]= \argmin _{ \BMphi\in \R^{h d}}\Big(\frac{1}{2^j h} \left\|\tilde \BMXi{[j]}  \BMphi - \br[j] \right\|^2_2 + \lambda_j\| \BMphi\|_{2,1}^{(d)} \Big)},
\end{equation}
where the regularization parameter is 
\begin{equation}
	{\lambda_j = 2\sqrt{\frac{2 d\log(2^{j} h/\gamma)}{2^{j-1} h}}}.
\end{equation}


\begin{algorithm}[t]
	\caption{Adaptive Doubling Lasso (AD-Lasso)}
	\label{alg:rich}
	\begin{algorithmic}[1]
		\State \textbf{Input:}  $L$ for the initial phase, the doubling sequence $\{\tilde T_j\}$ with $\tilde T_j=(2^{j+1}-1)h$ [see Fig.~\ref{illust:rich}]
		\For{$t=1: h$}
		\State Implement Algorithm~\ref{alg:initial} with parameter $L$.
		\EndFor
		\State Reset $\hat \BMtheta_0$ to the latest $\hat \BMtheta$. 
		\For{$t=h+1: T$}
		\State {Observe contexts vectors $\{\x_{t,a}:a\in[K]\}$}.
		\State Take the greedy action $a_t\in \sup_{a\in[K]}\AngBra{\x_{t,a},\hat \BMtheta_{j-1}}$, and receive a reward $r_t$.
		\If{$t = \tilde T_j$} \\ \hfill {\color{mygray} \footnotesize \# \textit{end of the $j$th epoch, estimate a new $\hat\BMtheta$}}
		\State Calculate $\hat \BMphi[j]$ according to the Lasso \eqref{Lasso:rich}.
		\State Let $\hat \BMtheta_j$ be the singular vector of $\hat \Phi$ associated with the largest singular value. 
		\EndIf 
		\EndFor
	\end{algorithmic}
\end{algorithm}

\begin{theorem}\label{Th:data_rich}
	Consider the stochastic contextual linear bandit model with long-horizon rewards described in \eqref{model:CB}. In Algoritm~\ref{alg:rich}, let $L$ be the same as in Thoerem~\ref {Th:data_poor}.
	When $T\ge h$, the regret  has the following upper bound:
	\begin{equation}\label{bound:rich}
		R_T =  {O \Big( d\sqrt{sh \log(dh)}+ \min\{q(\w),h\} +\sqrt{sdT \log(dT)} \Big)},
	\end{equation}
	where $q(\w)$ is defined in Theorem~\ref{Th:data_poor}.     
\end{theorem}

\begin{remark}
	The first two terms in \eqref{bound:rich} result from the initial phase ($t\le h$) when data is poor. Note that they are $T$-independent even if they are $h$-dependent; they play a role in the upper bound only when $T$ has the same order of $h$, i.e., $T=\Theta(h)$. In this case, the upper bound becomes
	\begin{equation*}
		{R_T = O \Big( d\sqrt{sT \log(dT)}+\min\{ q(\w), T\}\Big)}. 
	\end{equation*}
	By contrast, if $T$ is large, specifically, $T\ge \max\{dh, h^{2\alpha(\mu)}/(s d\log(T))\}$, the first two terms are dominated by the last one in \eqref{bound:rich}, and the upper bound reduces to
	\begin{equation*}
		{R_T = O \Big(  \sqrt{sd T \log(dT)} \Big)}. 
	\end{equation*}
	Then, our upper bound is \textit{optimal} in $d$ and $T$ (up to logarithmic factors), which follows from the lower bound $\sqrt{dT}$ shown in \citet{chu2011contextual} for linear contextual bandits.
	
	\textit{Discussion on lower bound:} \citet{ren2020dynamic} obtained a lower bound $\Omega(\sqrt{sT})$  for $s$-sparse contextual linear bandits. Taking into account the low-rank and sparse nature of our problem, one can show a lower bound of $\Omega(\sqrt{(s+d)T})$ in our case by adapting their proof. Thus, the gap between the our bound in Theorem~\ref{Th:data_rich} and this lower bound is at most a factor of $\log(dT)\min\{\sqrt{s},\sqrt{d}\}$. However, we believe that the actual gap is much smaller. We presume that a tighter lower bound can be constructed since we find that the sampling complexity of low-rank estimation using circulant measurements does not simply depends on the rank (see Sec.~\ref{challenges} for the discussion). \TiaQED 
\end{remark}




\begin{figure*}[t]
	\centering
	\includegraphics[scale=0.3]{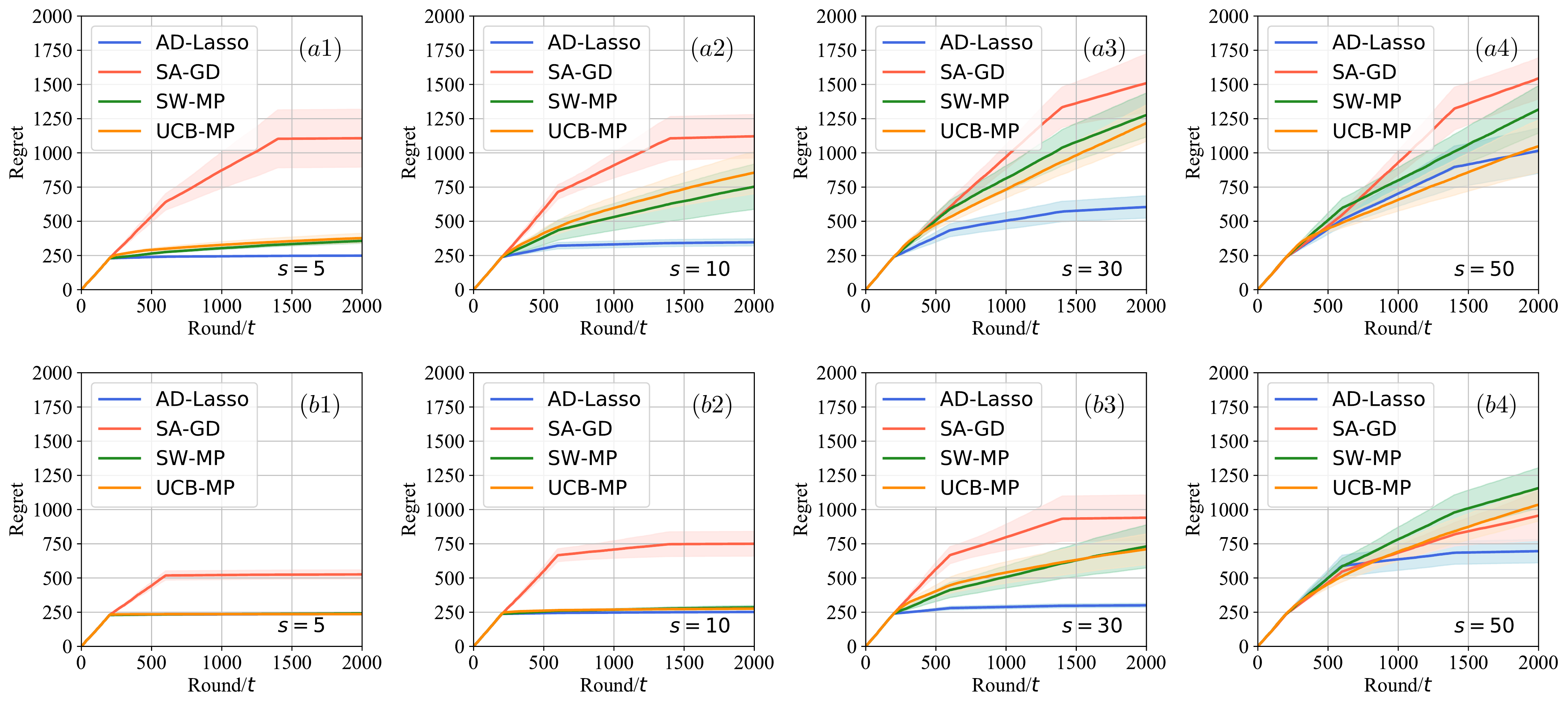}
	\caption{Performance comparison of different algorithms. (a) Flat $\w$: the non-zero entries in  $\w$ are equally spread. (b) Spiking $\w$:  the majority of the weights concentrates at only $20\%$ of the non-zeros positions in $\w$. Different sparsity of $\w$ is also considered for both cases. (Universal parameters: $T=2000,h=100,d=5$, and $\|\w\|_1=1$.)}
	\label{data_rich_diff_spars}
\end{figure*}

\newpage
\noindent\textbf{Experiments.} We perform some experiments to compare our algorithm AD-Lasso with the following three:
\begin{enumerate}
	\item \textit{Sparse-Alternating Gradient Descent} (SA-GD). The core of SA-GD is rank-1 and then sparse matrix estimation.   Based on  \eqref{model:low_rank}, SA-GD alternatively reconstructs $\BMtheta$ and $\w$ by gradient descent, and projects $\w$ to the $s$-sparse space. 
	\item \textit{Single-Weight Matching Pursuit} (SW-MP). The core of SW-MP is to locate the \textit{largest} weight in $\w$ by testing the correlation between the reward vector and the columns of the context matrix. Then, with this location information, $\BMtheta$ is estimated simply by the least-squares regression, ignoring other weights in $\w$. 
	\item \textit{UCB with Matching Pursuit} (UCB-MP). This algorithm is similar to SW-MP; the difference is that in each epoch we use UCB to update $\hat \BMtheta$ and  make decisions. 
\end{enumerate}

To facilitate fair comparison, we use the same  doubling scheme with identical epoch lengths for all the algorithms. The only difference is the method we use to estimate $\bt$ (see Appendix~\ref{experi} for more details of these algorithms). Different sparsity and reward dependence structure are considered in the experiments (see the caption in Fig.~\ref{data_rich_diff_spars}).

Our algorithm outperforms SA-GD significantly when $\w$ is highly sparse (see (a1), (a2), (b1), and (b2)). Since SA-GD is primarily reliant on rank-1 factorization, this indicates that, relative to low-rankness, sparsity plays a more dominant role in the estimation quality in line with our theory. Surprisingly, as $\w$ becomes less sparse, our algorithm can still outperform SA-GD, even in the regime $sd>d+h$. This supports the difficulty of low-rank matrix estimation with circulant measurements, which is consistent with our discussion in Sec.~\ref{challenges}. Yet, stronger theoretical analysis is desirable to formalize these findings beyond our Lemma \ref{norip}. 

AD-Lasso performs as well as SW-MP and UCB-MP, even when the weights of $\w$ are highly concentrated over few entries. When the weights are more spread out, AD-Lasso works much better, indicating that simply exploring and exploiting the largest weight becomes suboptimal.







\section{Concluding Remarks}
In this paper, we introduce a novel variation of the stochastic contextual bandits problem, where the reward depends on $s$ prior contexts, up to a time horizon of $h$. Leveraging the sparsity in the reward dependence pattern, we propose two  algorithms that account for both the data-poor and data-rich regimes. We also derive horizon-independent (up to $\log(h)$ terms) regret upper bounds for both algorithms, establishing that their sample efficiency is theoretically guaranteed. 

Our work opens up many future potential directions. For instance, the reward can depend on the prior contexts in a nonlinear fashion or sparsity pattern can vary in a data-dependent fashion. In either scenarios learning the reward dependence pattern will be more challenging. Also, beyond bandit problems, it is of interest to explore RL and control scenarios with long-term non-markovian structures where new strategies will be required. 

\section{Acknowledgement}
This work was supported in part by the Air Force grant AFOSR-FA9550-20-1-0140, NSF CCF-2046816, NSF TRIPODS II-DMS 2023166, NSF CCF-2007036, NSF CCF-2212261, NSF DMS-1839371, and Army Research Office grants ARO-78259-NS-MUR and W911NF2110312. 

  \bibliographystyle{plainnat}
\bibliography{refs_full}

\begin{thebibliography}{59}
\providecommand{\natexlab}[1]{#1}
\providecommand{\url}[1]{\texttt{#1}}
\expandafter\ifx\csname urlstyle\endcsname\relax
  \providecommand{\doi}[1]{doi: #1}\else
  \providecommand{\doi}{doi: \begingroup \urlstyle{rm}\Url}\fi

\bibitem[Abbasi-Yadkori et~al.(2011)Abbasi-Yadkori, P{\'a}l, and
  Szepesv{\'a}ri]{abbasi2011improved}
Yasin Abbasi-Yadkori, D{\'a}vid P{\'a}l, and Csaba Szepesv{\'a}ri.
\newblock Improved algorithms for linear stochastic bandits.
\newblock \emph{Advances in neural information processing systems}, 24, 2011.

\bibitem[Abbasi-Yadkori et~al.(2012)]{abbasi2012online}
Yasin Abbasi-Yadkori et~al.
\newblock Online-to-confidence-set conversions and application to sparse
  stochastic bandits.
\newblock In \emph{Artificial Intelligence and Statistics}, pages 1--9. PMLR,
  2012.

\bibitem[Adamczak(2015)]{adamczak2015note}
Radoslaw Adamczak.
\newblock A note on the hanson-wright inequality for random vectors with
  dependencies.
\newblock \emph{Electronic Communications in Probability}, 20:\penalty0 1--13,
  2015.

\bibitem[Anava et~al.(2015)]{anava2015online}
Oren Anava et~al.
\newblock Online learning for adversaries with memory: Price of past mistakes.
\newblock In \emph{Advances in Neural Information Processing Systems},
  volume~28, 2015.

\bibitem[Ariu et~al.(2022)Ariu, Abe, and Prouti{\`e}re]{ariu2022thresholded}
Kaito Ariu, Kenshi Abe, and Alexandre Prouti{\`e}re.
\newblock Thresholded lasso bandit.
\newblock In \emph{International Conference on Machine Learning}, pages
  878--928. PMLR, 2022.

\bibitem[Bastani and Bayati(2020)]{bastani2020online}
Hamsa Bastani and Mohsen Bayati.
\newblock Online decision making with high-dimensional covariates.
\newblock \emph{Operations Research}, 68\penalty0 (1):\penalty0 276--294, 2020.

\bibitem[Bickel et~al.(2009)Bickel, Ritov, and
  Tsybakov]{bickel2009simultaneous}
Peter~J Bickel, Ya’acov Ritov, and Alexandre~B Tsybakov.
\newblock Simultaneous analysis of lasso and dantzig selector.
\newblock \emph{The Annals of statistics}, 37\penalty0 (4):\penalty0
  1705--1732, 2009.

\bibitem[Bistritz et~al.(2019)Bistritz, Zhou, Chen, Bambos, and
  Blanchet]{bistritz2019online}
Ilai Bistritz, Zhengyuan Zhou, Xi~Chen, Nicholas Bambos, and Jose Blanchet.
\newblock Online exp3 learning in adversarial bandits with delayed feedback.
\newblock \emph{Advances in neural information processing systems}, 32, 2019.

\bibitem[Brown et~al.(2020)Brown, Mann, Ryder, Subbiah, Kaplan, Dhariwal,
  Neelakantan, Shyam, Sastry, Askell, et~al.]{brown2020language}
Tom Brown, Benjamin Mann, Nick Ryder, Melanie Subbiah, Jared~D Kaplan, Prafulla
  Dhariwal, Arvind Neelakantan, Pranav Shyam, Girish Sastry, Amanda Askell,
  et~al.
\newblock Language models are few-shot learners.
\newblock \emph{Advances in neural information processing systems},
  33:\penalty0 1877--1901, 2020.

\bibitem[Bubeck et~al.(2012)Bubeck, Cesa-Bianchi, et~al.]{bubeck2012regret}
S{\'e}bastien Bubeck, Nicolo Cesa-Bianchi, et~al.
\newblock Regret analysis of stochastic and nonstochastic multi-armed bandit
  problems.
\newblock \emph{Foundations and Trends{\textregistered} in Machine Learning},
  5\penalty0 (1):\penalty0 1--122, 2012.

\bibitem[Candes and Tao(2007)]{candes2007dantzig}
Emmanuel Candes and Terence Tao.
\newblock The dantzig selector: Statistical estimation when p is much larger
  than n.
\newblock \emph{The annals of Statistics}, 35\penalty0 (6):\penalty0
  2313--2351, 2007.

\bibitem[Carpentier and Munos(2012)]{carpentier2012bandit}
Alexandra Carpentier and R{\'e}mi Munos.
\newblock Bandit theory meets compressed sensing for high dimensional
  stochastic linear bandit.
\newblock In \emph{Artificial Intelligence and Statistics}, pages 190--198.
  PMLR, 2012.

\bibitem[Cella and Cesa-Bianchi(2020)]{cella2020stochastic}
Leonardo Cella and Nicol{\`o} Cesa-Bianchi.
\newblock Stochastic bandits with delay-dependent payoffs.
\newblock In \emph{International Conference on Artificial Intelligence and
  Statistics}, pages 1168--1177. PMLR, 2020.

\bibitem[Cesa-Bianchi et~al.(2018)Cesa-Bianchi, Gentile, and
  Mansour]{cesa2018nonstochastic}
Nicolo Cesa-Bianchi, Claudio Gentile, and Yishay Mansour.
\newblock Nonstochastic bandits with composite anonymous feedback.
\newblock In \emph{Conference On Learning Theory}, pages 750--773. PMLR, 2018.

\bibitem[Cesa-Bianchi et~al.(2019)Cesa-Bianchi, Gentile, Mansour, and
  Minora]{cesa2019delay}
Nicol‘o Cesa-Bianchi, Claudio Gentile, Yishay Mansour, and Alberto Minora.
\newblock Delay and cooperation in nonstochastic bandits.
\newblock \emph{Journal of Machine Learning Research}, 20:\penalty0 1--38,
  2019.

\bibitem[Chen et~al.(2021)Chen, Lu, Rajeswaran, Lee, Grover, Laskin, Abbeel,
  Srinivas, and Mordatch]{chen2021decision}
Lili Chen, Kevin Lu, Aravind Rajeswaran, Kimin Lee, Aditya Grover, Misha
  Laskin, Pieter Abbeel, Aravind Srinivas, and Igor Mordatch.
\newblock Decision transformer: Reinforcement learning via sequence modeling.
\newblock \emph{Advances in neural information processing systems},
  34:\penalty0 15084--15097, 2021.

\bibitem[Chu et~al.(2007)Chu, Li, Reyzin, and Schapire]{langford2007epoch}
Wei Chu, Lihong Li, Lev Reyzin, and Robert Schapire.
\newblock The epoch-greedy algorithm for contextual multi-armed bandits.
\newblock \emph{Advances in neural information processing systems}, 20\penalty0
  (1):\penalty0 96--1, 2007.

\bibitem[Chu et~al.(2011)]{chu2011contextual}
Wei Chu et~al.
\newblock Contextual bandits with linear payoff functions.
\newblock In \emph{Proceedings of the Fourteenth International Conference on
  Artificial Intelligence and Statistics}, pages 208--214. JMLR Workshop and
  Conference Proceedings, 2011.

\bibitem[Davenport and Romberg(2016)]{davenport2016overview}
Mark~A Davenport and Justin Romberg.
\newblock An overview of low-rank matrix recovery from incomplete observations.
\newblock \emph{IEEE Journal of Selected Topics in Signal Processing},
  10\penalty0 (4):\penalty0 608--622, 2016.

\bibitem[Foucart and Rauhut(2013)]{foucart2013mathematical}
Simon Foucart and Holger Rauhut.
\newblock A mathematical introduction to compressive sensing, 2013.

\bibitem[Gael et~al.(2020)Gael, Vernade, Carpentier, and
  Valko]{gael2020stochastic}
Manegueu~Anne Gael, Claire Vernade, Alexandra Carpentier, and Michal Valko.
\newblock Stochastic bandits with arm-dependent delays.
\newblock In \emph{International Conference on Machine Learning}, pages
  3348--3356. PMLR, 2020.

\bibitem[Garg and Akash(2019)]{garg2019stochastic}
Siddhant Garg and Aditya~Kumar Akash.
\newblock Stochastic bandits with delayed composite anonymous feedback.
\newblock \emph{arXiv preprint arXiv:1910.01161}, 2019.

\bibitem[Gyorgy and Joulani(2021)]{gyorgy2021adapting}
Andras Gyorgy and Pooria Joulani.
\newblock Adapting to delays and data in adversarial multi-armed bandits.
\newblock In \emph{International Conference on Machine Learning}, pages
  3988--3997. PMLR, 2021.

\bibitem[Hao et~al.(2020)Hao, Lattimore, and Wang]{hao2020high}
Botao Hao, Tor Lattimore, and Mengdi Wang.
\newblock High-dimensional sparse linear bandits.
\newblock \emph{Advances in Neural Information Processing Systems},
  33:\penalty0 10753--10763, 2020.

\bibitem[Hao et~al.(2021)Hao, Lattimore, and Deng]{hao2021information}
Botao Hao, Tor Lattimore, and Wei Deng.
\newblock Information directed sampling for sparse linear bandits.
\newblock \emph{Advances in Neural Information Processing Systems}, 34, 2021.

\bibitem[Hessel et~al.(2018)Hessel, Modayil, Van~Hasselt, Schaul, Ostrovski,
  Dabney, Horgan, Piot, Azar, and Silver]{hessel2018rainbow}
Matteo Hessel, Joseph Modayil, Hado Van~Hasselt, Tom Schaul, Georg Ostrovski,
  Will Dabney, Dan Horgan, Bilal Piot, Mohammad Azar, and David Silver.
\newblock Rainbow: Combining improvements in deep reinforcement learning.
\newblock In \emph{Thirty-second AAAI conference on artificial intelligence},
  2018.

\bibitem[Howson et~al.(2022)Howson, Pike-Burke, and Filippi]{howson2022delayed}
Benjamin Howson, Ciara Pike-Burke, and Sarah Filippi.
\newblock Delayed feedback in generalised linear bandits revisited.
\newblock \emph{arXiv preprint arXiv:2207.10786}, 2022.

\bibitem[Ito et~al.(2020)Ito, Hatano, Sumita, Takemura, Fukunaga, Kakimura, and
  Kawarabayashi]{ito2020delay}
Shinji Ito, Daisuke Hatano, Hanna Sumita, Kei Takemura, Takuro Fukunaga,
  Naonori Kakimura, and Ken-Ichi Kawarabayashi.
\newblock Delay and cooperation in nonstochastic linear bandits.
\newblock \emph{Advances in Neural Information Processing Systems},
  33:\penalty0 4872--4883, 2020.

\bibitem[Jin et~al.(2019)Jin, Netrapalli, Ge, Kakade, and Jordan]{jin2019short}
Chi Jin, Praneeth Netrapalli, Rong Ge, Sham~M Kakade, and Michael~I Jordan.
\newblock A short note on concentration inequalities for random vectors with
  subgaussian norm.
\newblock \emph{arXiv preprint arXiv:1902.03736}, 2019.

\bibitem[Kim and Paik(2019)]{kim2019doubly}
Gi-Soo Kim and Myunghee~Cho Paik.
\newblock Doubly-robust lasso bandit.
\newblock \emph{Advances in Neural Information Processing Systems}, 32, 2019.

\bibitem[Krahmer et~al.(2014)Krahmer, Mendelson, and
  Rauhut]{krahmer2014suprema}
Felix Krahmer, Shahar Mendelson, and Holger Rauhut.
\newblock Suprema of chaos processes and the restricted isometry property.
\newblock \emph{Communications on Pure and Applied Mathematics}, 67\penalty0
  (11):\penalty0 1877--1904, 2014.

\bibitem[Kumar et~al.(2022)]{kumar2022online}
Raunak Kumar et~al.
\newblock Online convex optimization with unbounded memory.
\newblock \emph{arXiv preprint arXiv:2210.09903}, 2022.

\bibitem[Lancewicki et~al.(2021)Lancewicki, Segal, Koren, and
  Mansour]{lancewicki2021stochastic}
Tal Lancewicki, Shahar Segal, Tomer Koren, and Yishay Mansour.
\newblock Stochastic multi-armed bandits with unrestricted delay distributions.
\newblock In \emph{International Conference on Machine Learning}, pages
  5969--5978. PMLR, 2021.

\bibitem[Lattimore and Szepesv{\'a}ri(2020)]{lattimore2020bandit}
Tor Lattimore and Csaba Szepesv{\'a}ri.
\newblock \emph{Bandit algorithms}.
\newblock Cambridge University Press, 2020.

\bibitem[Li et~al.(2019)Li, Chen, and Giannakis]{li2019bandit}
Bingcong Li, Tianyi Chen, and Georgios~B Giannakis.
\newblock Bandit online learning with unknown delays.
\newblock In \emph{The 22nd International Conference on Artificial Intelligence
  and Statistics}, pages 993--1002. PMLR, 2019.

\bibitem[Li et~al.(2010)Li, Chu, Langford, and Schapire]{li2010contextual}
Lihong Li, Wei Chu, John Langford, and Robert~E Schapire.
\newblock A contextual-bandit approach to personalized news article
  recommendation.
\newblock In \emph{Proceedings of the 19th international conference on World
  wide web}, pages 661--670, 2010.

\bibitem[Li et~al.(2017)Li, Lu, and Zhou]{li2017provably}
Lihong Li, Yu~Lu, and Dengyong Zhou.
\newblock Provably optimal algorithms for generalized linear contextual
  bandits.
\newblock In \emph{International Conference on Machine Learning}, pages
  2071--2080. PMLR, 2017.

\bibitem[Oh et~al.(2021)Oh, Iyengar, and Zeevi]{oh2021sparsity}
Min-hwan Oh, Garud Iyengar, and Assaf Zeevi.
\newblock Sparsity-agnostic lasso bandit.
\newblock In \emph{International Conference on Machine Learning}, pages
  8271--8280. PMLR, 2021.

\bibitem[Oymak et~al.(2015)Oymak, Jalali, Fazel, Eldar, and
  Hassibi]{oymak2015simultaneously}
Samet Oymak, Amin Jalali, Maryam Fazel, Yonina~C Eldar, and Babak Hassibi.
\newblock Simultaneously structured models with application to sparse and
  low-rank matrices.
\newblock \emph{IEEE Transactions on Information Theory}, 61\penalty0
  (5):\penalty0 2886--2908, 2015.

\bibitem[Pike-Burke et~al.(2018)Pike-Burke, Agrawal, Szepesvari, and
  Grunewalder]{pike2018bandits}
Ciara Pike-Burke, Shipra Agrawal, Csaba Szepesvari, and Steffen Grunewalder.
\newblock Bandits with delayed, aggregated anonymous feedback.
\newblock In \emph{International Conference on Machine Learning}, pages
  4105--4113. PMLR, 2018.

\bibitem[Qin et~al.(2022{\natexlab{a}})Qin, Menara, Oymak, Ching, and
  Pasqualetti]{Qin2022_OJCSYS}
Yuzhen Qin, Tommaso Menara, Samet Oymak, ShiNung Ching, and Fabio Pasqualetti.
\newblock Non-stationary representation learning in sequential linear bandits.
\newblock \emph{IEEE Open Journal of Control Systems}, 1:\penalty0 41--56,
  2022{\natexlab{a}}.
\newblock \doi{10.1109/OJCSYS.2022.3178540}.

\bibitem[Qin et~al.(2022{\natexlab{b}})Qin, Menara, Oymak, Ching, and
  Pasqualetti]{QinACC2022}
Yuzhen Qin, Tommaso Menara, Samet Oymak, ShiNung Ching, and Fabio Pasqualetti.
\newblock Representation learning for context-dependent decision-making.
\newblock In \emph{2022 American Control Conference (ACC)}, pages 2130--2135,
  2022{\natexlab{b}}.

\bibitem[Rauhut(2010)]{rauhut2010compressive}
Holger Rauhut.
\newblock Compressive sensing and structured random matrices.
\newblock \emph{Theoretical foundations and numerical methods for sparse
  recovery}, 9\penalty0 (1):\penalty0 92, 2010.

\bibitem[Ren and Zhou(2020)]{ren2020dynamic}
Zhimei Ren and Zhengyuan Zhou.
\newblock Dynamic batch learning in high-dimensional sparse linear contextual
  bandits.
\newblock \emph{arXiv preprint arXiv:2008.11918}, 2020.

\bibitem[Richard et~al.(2012)Richard, Savalle, and
  Vayatis]{richard2012estimation}
Emile Richard, Pierre-Andr{\'e} Savalle, and Nicolas Vayatis.
\newblock Estimation of simultaneously sparse and low rank matrices.
\newblock \emph{arXiv preprint arXiv:1206.6474}, 2012.

\bibitem[Rudelson and Vershynin(2013)]{rudelson2013hanson}
Mark Rudelson and Roman Vershynin.
\newblock Hanson-wright inequality and sub-gaussian concentration.
\newblock \emph{Electronic Communications in Probability}, 18:\penalty0 1--9,
  2013.

\bibitem[Shi et~al.(2020)]{shi2020online}
Guanya Shi et~al.
\newblock Online optimization with memory and competitive control.
\newblock In \emph{Advances in Neural Information Processing Systems},
  volume~33, pages 20636--20647, 2020.

\bibitem[Thune et~al.(2019)Thune, Cesa-Bianchi, and
  Seldin]{thune2019nonstochastic}
Tobias~Sommer Thune, Nicol{\`o} Cesa-Bianchi, and Yevgeny Seldin.
\newblock Nonstochastic multiarmed bandits with unrestricted delays.
\newblock \emph{Advances in Neural Information Processing Systems}, 32, 2019.

\bibitem[Vaswani et~al.(2017)Vaswani, Shazeer, Parmar, Uszkoreit, Jones, Gomez,
  Kaiser, and Polosukhin]{vaswani2017attention}
Ashish Vaswani, Noam Shazeer, Niki Parmar, Jakob Uszkoreit, Llion Jones,
  Aidan~N Gomez, {\L}ukasz Kaiser, and Illia Polosukhin.
\newblock Attention is all you need.
\newblock \emph{Advances in neural information processing systems}, 30, 2017.

\bibitem[Vernade et~al.(2020)Vernade, Carpentier, Lattimore, Zappella, Ermis,
  and Brueckner]{vernade2020linear}
Claire Vernade, Alexandra Carpentier, Tor Lattimore, Giovanni Zappella, Beyza
  Ermis, and Michael Brueckner.
\newblock Linear bandits with stochastic delayed feedback.
\newblock In \emph{International Conference on Machine Learning}, pages
  9712--9721. PMLR, 2020.

\bibitem[Vershynin(2018)]{vershynin2018high}
Roman Vershynin.
\newblock \emph{High-Dimensional Probability: An Introduction with Applications
  in Data Science}, volume~47.
\newblock Cambridge university press, 2018.

\bibitem[Wainwright(2019)]{wainwright2019high}
Martin~J Wainwright.
\newblock \emph{High-Dimensional Statistics: A Non-Asymptotic Viewpoint},
  volume~48.
\newblock Cambridge University Press, 2019.

\bibitem[Wang et~al.(2021)Wang, Wang, and Huang]{wang2021adaptive}
Siwei Wang, Haoyun Wang, and Longbo Huang.
\newblock Adaptive algorithms for multi-armed bandit with composite and
  anonymous feedback.
\newblock In \emph{Proceedings of the AAAI Conference on Artificial
  Intelligence}, volume~35, pages 10210--10217, 2021.

\bibitem[Wang et~al.(2018)Wang, Wei, and Yao]{wang2018minimax}
Xue Wang, Mingcheng Wei, and Tao Yao.
\newblock Minimax concave penalized multi-armed bandit model with
  high-dimensional covariates.
\newblock In \emph{International Conference on Machine Learning}, pages
  5200--5208. PMLR, 2018.

\bibitem[Wedin(1972)]{wedin1972perturbation}
Per-{\AA}ke Wedin.
\newblock Perturbation bounds in connection with singular value decomposition.
\newblock \emph{BIT Numerical Mathematics}, 12\penalty0 (1):\penalty0 99--111,
  1972.

\bibitem[Woodroofe(1979)]{woodroofe1979one}
Michael Woodroofe.
\newblock A one-armed bandit problem with a concomitant variable.
\newblock \emph{Journal of the American Statistical Association}, 74\penalty0
  (368):\penalty0 799--806, 1979.

\bibitem[Zhang et~al.(2022)Zhang, Tsuchida, and Ong]{zhang2022gaussian}
Mengyan Zhang, Russell Tsuchida, and Cheng~Soon Ong.
\newblock Gaussian process bandits with aggregated feedback.
\newblock In \emph{Proceedings of the AAAI Conference on Artificial
  Intelligence}, volume~36, pages 9074--9081, 2022.

\bibitem[Zhou et~al.(2019)Zhou, Xu, and Blanchet]{zhou2019learning}
Zhengyuan Zhou, Renyuan Xu, and Jose Blanchet.
\newblock Learning in generalized linear contextual bandits with stochastic
  delays.
\newblock \emph{Advances in Neural Information Processing Systems}, 32, 2019.

\bibitem[Zimmert and Seldin(2020)]{zimmert2020optimal}
Julian Zimmert and Yevgeny Seldin.
\newblock An optimal algorithm for adversarial bandits with arbitrary delays.
\newblock In \emph{International Conference on Artificial Intelligence and
  Statistics}, pages 3285--3294. PMLR, 2020.

\end{thebibliography}

\appendix
\onecolumn
\begin{center}
	\Huge \bf{Appendix}
\end{center}
\renewcommand{\theequation}{A\arabic{equation}}
\setcounter{equation}{0}

The appendix is organized into two parts: 

\begin{enumerate}
	\item In Part I, we are dedicated to deriving some general results, which are useful for our later analysis and also of independent interest. Specifically, we show that measurement matrices formed by subsampling $\tilde \Omega(s\log^2(s))$ rows from the following \textit{block-circulant} matrix satisfy the  Restricted Isometry Property  (RIP) for sparse vectors: 
	\begin{align*}
		\mathbf{C}=\begin{bmatrix}
			\BMxi_{n}^\top&\BMxi_{n-1}^\top&\cdots&\BMxi_{1}^\top\\
			\BMxi_{1}^\top&\BMxi_{n}^\top&\cdots&\BMxi_{2}^\top\\
			\vdots&\vdots&\ddots&\vdots\\
			\BMxi_{n-1}^\top&\BMxi_{n-2}^\top&\cdots&\BMxi_{n}^\top
		\end{bmatrix},
	\end{align*}
	where $\BMxi_1,\BMxi_2, \dots, \BMxi_n \in \R^d$ are independent and isotropic random vectors. As a generalization to the results in \citet{krahmer2014suprema}, we allow  $\BMxi_i$ to have \textit{dependent} entries. The result is summarized in Theorem~\ref{RIP:informal} of the main text. 
	
	\item In Part II, we provide supporting materials for the main problem of this paper: linear contextual bandits with long-horizon rewards.  The results we obtained in Part I will be used to prove Theorems~\ref{Th:data_poor} and \ref{Th:data_rich}. Some further experimental results will also be presented. 
\end{enumerate}

\section*{Preliminary}

\textbf{Further notations.} Given a vector $\x=[x_1,x_2,\dots,x_{nd}]^\top$, denote $\|\x\|_0 :=\sum_{i=1}^{nd}\mathbb{I}(|x_i| \neq 0)$ as its $\ell_0$-norm. We further define an $\ell_{2,1}$-norm of it, which is similar to the matrix version. Specifically, $\|\x\|^{(d)}_{2,1}=\sum_{j=1}^n\|\x_j\|_2$, where each $\x_j$ is obtained by partitioning $\x$ into $n$ blocks with equal size $d$ (i.e., $\x=[\x_1^\top,\x_2^\top,\dots, \x_n^\top]^\top$ with $\x_i\in\R^d$). Similarly, denote  $\|\x\|^{(d)}_{2,\infty}=\max_{j\in\{1,\dots,n\}}\|\x_j\|_2$  and $\|\x\|^{(d)}_{2,0}=\sum_{i=1}^{n}\mathbb{I}(\|\x_j\|_2\neq 0)$.
Given a matrix $A\in \R^{m\times n}$, $\|A\|_F$ denotes its Frobenius norm. We use $\lesssim$ and $\gtrsim$ to denote inequalities that hold up to constants/logarithmic factors.

\vspace{10pt}
A vector $\va=[a_1,a_2,\dots,a_d]^\top\in \R^d$ is called $s$-sparse if $\|\va\|_0\le s$. A vector $\vb=[\vb_1^\top,\vb_2^\top,\dots,\vb_h^\top]^\top \in \R^{hd}$, with each block $\vb_i\in\R^d$, is called $s$-block sparse if at most $s$ of $\vb_i$'s are non-zero (i.e., if $\|\vb\|_{2,0}^{(d)}\le  s$). 

\vspace{10pt}
A random variable $x\in \R$ is said to be sub-Gaussian with variance proxy $\sigma^2$ (in short $\sigma^2$-sub-Gaussian) if $\BE [x]=0$ and its moment generating function satisfies 
\begin{align*}
	\BE[\exp(tX)] \le \exp(\frac{\sigma^2 t^2}{2}), \quad\quad\forall t \in \R.
\end{align*}
A random vector $\x\in\R^d$ is said to be $\sigma^2$-sub-Gaussian if for any $u\in\R^d$ such that $\|\mathbf{u}\|\le 1$, the random variable $\mathbf{u}^\top \x$ is $\sigma^2$-sub-Gaussian. 
For a random vector $\x$ that is not zero-mean, we abuse the notation by saying it is $\sigma^2$-sub-Gaussian if $\x-\BE \x$ is  $\sigma^2$-sub-Gaussian.

\newpage

\begin{center}
	\LARGE \textbf{Part I: General Results on RIP \\of Block Circulant Measurements}
\end{center}
\section{Main Results}\label{prelim}
We first introduce a   definition that will be used later. Then, we present the main  result in this part.

\begin{definition}
	A matrix $\BMXi \in \R^{m \times d}$ is said to have the restricted isometry property of order $s$ and  level $\delta$ (in short $(s,\delta)$-RIP) if there exists $\delta \in(0,1)$ such that
	\begin{align*}
		(1-\delta)\|\va\|_2^2 \le \|\BMXi \va\|_2^2 \le (1+\delta)\|\va\|_2^2
	\end{align*}
	for all  $s$-sparse vectors $\va\in \R^d$. The smallest $\delta$ such that this inequality holds. denoted as $\delta_s$, is called the \textit{restricted isometry constant}. 
\end{definition}

\vspace{20pt}

In Part I, we are interested in the following block-circulant matrix:
\begin{align}\label{cir:original}
	\bC=\begin{bmatrix}
		\BMxi_{n}^\top&\BMxi_{n-1}^\top&\cdots&\BMxi_{1}^\top\\
		\BMxi_{1}^\top&\BMxi_{n}^\top&\cdots&\BMxi_{2}^\top\\
		\vdots&\vdots&\ddots&\vdots\\
		\BMxi_{n-1}^\top&\BMxi_{n-2}^\top&\cdots&\BMxi_{n}^\top
	\end{bmatrix} \in \R^{n \times nd},
\end{align}
where $\BMxi_1,\BMxi_2, \dots, \BMxi_n \in \R^d$ are independent random vectors, and for each $i$ it holds that $\BE \BMxi_i=\boldsymbol{0}$ and $\BE \BMxi_i \BMxi_i^\top=\boldsymbol{I}$. Let $\BMXi$ be a matrix formed by sub-sampling any $m$ ($m< n$) rows of $\bC$, i.e.,
\begin{equation}\label{subsam:xi}
	\BMXi = \frac{1}{\sqrt{m}}\bR_{\Omega} \bC,
\end{equation}
where  $\Omega \subset [n]$ with $|\Omega|=m$  and $\bR_{\Omega}\in \R^{m\times n}$ selects the $m$ rows of $\bC$ that are indexed by $\Omega$ (note that we include the term $\frac{1}{\sqrt{m}}$ here to normalize the sub-sampled matrix). It can be observed that $\bR_{\Omega}$ is the matrix formed by removing the rows indexed by the set $[n]\backslash \Omega$ from the identity matrix $\boldsymbol{I}_n$.
Now, consider the following measurement
\begin{align}\label{circ:meas}
	\y = \BMXi \va.
\end{align}
We assume that the unknown vector $\va\in\R^{nd}$ is  $s_0$-sparse, i.e., $\|\va\|_0\le s_0$. 
Our main result in this part establishes RIP of sub-sampled block circulant matrices for sparse vectors.  

\begin{theorem}\label{RIP:general}
	Consider the design matrix $\BMXi \in \R^{m \times nd}$ given by \eqref{subsam:xi}, where $\Omega$ is any subset of $[n]$ with cardinality $m$. Assume that each random vector $\BMxi_i$ satisfies the Hanson-Wright inequality
	\begin{align}\label{H-W}
		\Pr \left[|\BMxi_i^\top \bP \BMxi_i -\BE (\BMxi_i^\top  \bP \BMxi_i)|\ge t \right] 
		\le 2 \exp \left( -\frac{1}{c} \min \left\{ \frac{t^2}{k^4 \|\bP\|_F^2},\frac{t}{k^2 \|\bP\|_\op}\right\}\right)
	\end{align}
	for any positive semi-definite matrix $\bP\in \R^{d\times d}$, where $k$ is a constant and $c$ is an absolute constant.
	Then, for any $\delta,\eta\in(0,1)$, there exists a constant $c_1>0$ such that, if 
	\begin{align}\label{sp_size:cir}
		m \ge c_1 \frac{1}{\delta^2} s_0 \max \left \{\log^2 (s_0) \log^2(nd), \log \big(\frac{1}{\eta} \big) \right \},
	\end{align}
	then the restricted isometry constant of the design matrix $\BMXi$ for all $s_0$-sparse vectors $\va\in \R^{nd}$, denoted as $\delta_{s_0}$, satisfies $\delta_{s_0} \le \delta$ with probability at least $1-\eta$. 
\end{theorem}

\begin{remark}
	This theorem is similar to Theorem~{4.1} in \citet{krahmer2014suprema}. The key differences are: 1) We subsample a block-circulant matrix instead of a circulant matrix (Each row in a circulant matrix rotates one element to the right relative to the previous row; as for the block-circulant matrix in our case, each row rotates one block of size $d$ to the right); 2) We allow the first row of the design matrix $\BMXi$ to have \textit{dependent} entries (i.e., entries of each $\BMxi_i$ can be dependent), while \citet{krahmer2014suprema} assumes  them to be independent; 3) We make an additional assumption on the vectors $\BMxi_i$ in \eqref{H-W}. It worth mentioning that this assumption is a mild one. It is even milder than the classic one for Hanson-Wright inequality \cite{rudelson2013hanson}. This is because we just require $\bP$ to be positive semi-definite matrices instead of any $d \times d$ matrices. Yet, whether this assumption can be further relaxed or even removed remains an open problem. \TiaQED
\end{remark}

Here, we have derived a theorem that holds for \textit{general sparse vectors}. As it turns out later, in our bandit problem in Section~\ref{main_results}, our unknown vector $\BMphi=\w \otimes \bt$ is $s$ block-sparse with block length of $d$. We will derive a restricted eigenvalue condition for block-sparse vectors, based on the next theorem that follows from Theorem~\ref{RIP:general} straightforwardly, to facilitate our analysis. It is worth mentioning that our result in Theorem~\ref{RIP:general} can be applied to more general problems.


\begin{theorem}\label{RIP:general:FORbandits}
	Consider the design matrix $\BMXi \in \R^{m \times hd}$ is given by the Toeplitz matrix
	\begin{align*}
		\BMXi=\frac{1}{\sqrt{m}} \begin{bmatrix}
			\BMxi_{h}^\top&\BMxi_{h-1}^\top&\cdots&\BMxi_{1}^\top\\
			\BMxi_{h+1}^\top&\BMxi_{h}^\top&\cdots&\BMxi_{2}^\top\\
			\vdots&\vdots&\ddots&\vdots\\
			\BMxi_{m+h-1}^\top&\BMxi_{m+h-2}^\top&\cdots&\BMxi_{m}^\top
		\end{bmatrix}.
	\end{align*}
	Assume that the vector $\BMxi$ satisfies the assumptions in Theorem~\ref{RIP:general}.
	Then, for any $\delta,\eta\in(0,1)$, there exists a constant $c>0$ such that, if 
	\begin{align}\label{sp_size:Toep}
		m \ge c \frac{1}{\delta^2} s_0 \max \left \{\log^2 (s_0) \log^2(hd), \log \big(\frac{1}{\eta} \big) \right \},
	\end{align}
	then the restricted isometry constant of the design matrix $\BMXi$ for all $s_0$-sparse vectors $\va\in \R^{hd}$, denoted as $\delta_{s_0}$, satisfies $\delta_{s_0} \le \delta$ with probability at least $1-\eta$. 
\end{theorem}
\begin{remark}
	Here, the Toeplitz matrix $\BMXi$ can be regarded as a matrix formed by sub-sampling the \textit{first} $m$ rows and last $hd$ columns of the circulant matrix in \eqref{cir:original}. Note that, the right hand side of \eqref{sp_size:Toep} has a polylogarithmic dependence on $hd$ instead of $nd$, which is a slightly \textit{better} than \eqref{sp_size:cir}. One can show this dependence following similar steps as those in the proof of Theorem~\ref{RIP:general}.
\end{remark}

\section{Proof of Theorem~\ref{RIP:general}}\label{ciclulant:RIP}

Next, we provide the proof of Theorem~\ref{RIP:general}.
We first make some preparation by rewriting \eqref{circ:meas}  into another form, with the aim to construct the proof from a different angle. 
We partition the vector $\va\in \R^{nd}$ in \eqref{circ:meas} into $n$ blocks and rewrite it as 
\[\va=[\va_1^\top,\va_2^\top, \dots, \va_n^\top]^\top, \hspace{20pt}\text{where } \va_i\in \R^d.\]
Note that here the design matrix $\BMXi$ is formed by subsampling the first $m$ rows of $\bC$ in \eqref{circ:meas}. Then,  one can rewrite \eqref{circ:meas} into
\begin{align}\label{circ:meas:2}
	\y = \frac{1}{\sqrt{m}} \bR_{\Omega}\underbrace{\begin{bmatrix}
			\va_n^\top &\va_{n-1}^\top & \cdots&  \va_{2}^\top &\va_{1}^\top \\
			\va_1^\top &\va_{n}^\top & \cdots&  \va_{3}^\top &\va_{2}^\top \\
			\vdots&\ddots & \vdots&\vdots&\vdots\\
			\va_{n-1}^\top &\va_{n-2}^\top & \cdots&  \va_{1}^\top &\va_{n}^\top \\
	\end{bmatrix}}_{\bV\in \R^{n \times nd}} \begin{bmatrix}\BMxi_1 \\ \BMxi_2 \\ \vdots\\ \BMxi_{n} \end{bmatrix} := \bA \BMxi.
\end{align}
Recall that $\Omega \subset [n]$ with $|\Omega|=m$ and $\bR_{\Omega}\in \R^{m\times n}$ selects the $m$ rows of $\bV$ that are indexed by $\Omega$. 

Recall that we aim to show that the random matrix $\BMXi$ satisfies the restricted isometry condition for the $s_0$-sparse vector $\va$.  Equivalently, it suffices to bound the following quantity 
\begin{align}\label{RIP_const}
	C_{\CA} (\BMxi) := \sup_{\bA\in\CA} \Big| \|\bA\BMxi\|_2^2 -\BE \|\bA\BMxi\|_2^2 \Big|,
\end{align}
where  $\CA$ is the set of matrices $\bA$ generated by $\va \in \CD_{s_0}$ with 
\begin{align}\label{set:D}
	\CD_{s_0}:= \left\{\va\in \R^{nd}: \|\va\|_2\le 1, \|\va\|_0\le s_0 \right\}.
\end{align}

For a matrix $\bA$, we call the sub-matrix consisting of the $((i-1)d+1)$-th to the $id$-th columns the $i$-th \textit{block column} of $\bA$ and denote it as $\bA^i$. As for our $\bA$ in \eqref{circ:meas:2}, Fig.~\ref{fig:block_column} illustrates how block columns are defined. 

\begin{figure}[ht]
	\centering
	\includegraphics[scale=1.25]{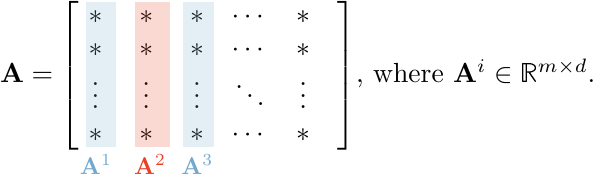}
	\caption{Illustration of the definition of block columns.}
	\label{fig:block_column}
\end{figure}

We then  define the following quantities
\begin{align*}
	&N_{\CA}(\BMxi):=\sup_{\bA \in \CA} \|\bA\BMxi\|_2, \hspace{12pt}  O_{\CA}(\BMxi):= \sup_{\bA\in\CA} \left| \sum_{i,j=1,i\neq j}^n\BMxi_i^\top (\bA^i)^\top \bA^j  \BMxi_j\right|,\\
	&D_{\CA} (\BMxi):= \sup_{\bA \in \CA} \left| \sum_{i=1}^n \left( \BMxi_i^\top (\bA^i)^\top \bA^i \BMxi_i -  \AngBra{\bA^i,\bA^i} \right) \right|,
\end{align*}
where $\AngBra{\bA^i,\bA^i}=\Tr\big((\bA^i)^\top \bA^i \big)$. 
Observe that 
\begin{align}\label{key}
	C_{\CA} (\BMxi)= \sup_{\bA \in \CA} \left| \sum_{i,j=1,i\neq j}^n\BMxi_i^\top (\bA^i)^\top \bA^j  \BMxi_j + \sum_{i=1}^n \BMxi_i^\top (\bA^i)^\top \bA^i \BMxi_i -  \AngBra{\bA^i,\bA^i} \right|\le O_{\CA}(\BMxi) + D_{\CA} (\BMxi).
\end{align}

Notice that $O_{\CA}(\BMxi)$ and $D_{\CA} (\BMxi)$ quantify the contributions of the off-block-diagonals and block-diagonals of the matrix $\bA^\top \bA$ to $C_{\CA}$, respectively. It remains to derive the bounds of the off-block-diagonal and block-diagonal terms. The bounds of these terms can be found in Lemma~\ref{bound_off_diagonal} and Theorem~\ref{bound:uniform} in Subsections~\ref{subsec:off_dia} and \ref{subsec:dia}, respectively. With the help of them, we are ready to construct the proof.

\begin{pfof}{Theorem~\ref{RIP:general}}
	We construct the proof using Lemma~\ref{bound_off_diagonal} and Theorem~\ref{bound:uniform} to bound $O_\CA(\xi)$ and $D_\CA(\xi)$, respectively. 
	
	\textbf{Step 1:} Since $\va$ is $s_0$-sparse, following similar steps as those in \citet{krahmer2014suprema}, one can derive that the quantities defined in Lemma~\ref{bound_off_diagonal} satisfy 
	\begin{align*}
		d_F(\CA) \le 1, \hspace{15pt} d_{2\to 2} (\CA) \le \sqrt{\frac{s_0}{m}},  \hspace{15pt} \text{and } \gamma_2 (\CA,\|\cdot\|_{\op}) \lesssim \sqrt{\frac{s_0}{m}} \log(s_0) \log(nd).
	\end{align*}
	Since $m \ge m_1=c_1 \frac{1}{\delta^2} \big( s_0\log^2 (s_0) \log^2(nd)+\log(\frac{1}{\eta_1}) \big)$, one can derive  from \eqref{inequ:chaos} that for any $\delta$ and $\eta_1\in(0,1)$, there exists $c_1>0$ such that 
	\begin{align}\label{off}
		\Pr \left[ O_{\CA}(\BMxi) \ge \frac{\delta}{2} \right] \le \eta_1. 
	\end{align}
	
	\textbf{Step 2:} To bound the block-diagonal term using Theorem~\ref{bound:uniform}, it remains to derive the covering number of the set $\CA$ with respect to the metric $\|\cdot\|_F$.  
	For any $\bA_1,\bA_2\in \CA$ that formed by $\va_1,\va_2 \in \CD_{s_0}$, respectively, it follows from \eqref{circ:meas:2} that  
	\begin{align*}
		\| \bA_1 - \bA_2 \|_F = \frac{1}{\sqrt{m}} \sqrt{m\|\va_1-\va_2\|^2_2} \le \| \va_1-\va_2 \|_2. 
	\end{align*}
	Then, it holds that $N(\CA,\|\cdot\|_F,u)\le N (\CD_{s_0}, \|\cdot\|_2,u)$.  From Proposition C.3 in \citet{foucart2013mathematical}, it holds that $N (\CD_{s_0}, \|\cdot\|_2,u) \le (1+\frac{2}{u})^{s_0}$. Substitute this $N(\CA,\|\cdot\|_F,u)$ into the following inequality (obtained in Theorem~\ref{bound:uniform})
	\begin{align*}
		\Pr \left[\sup_{A\in \CA } D_{\CA} (\BMxi) \ge t\right] \le 2 N\Big(\CA,\|\cdot\|_F,\frac{t}{6nd}\Big) \exp \left( -\frac{1}{ck^2}\min (m t^2 , m t ) \right),
	\end{align*}
	one can show that for any $\delta,\eta_2\in(0,1)$, there exists $c_2>0$  such that, for $m \ge m_2=c_2 \frac{1}{\delta^2} \big( s_0 \log(nd) +\log(\frac{1}{\eta_2})\big)$, it holds that 
	\begin{align}\label{diagonal}
		\Pr \left[ D_{\CA}(\BMxi) \ge \frac{\delta}{2} \right] \le \eta_2. 
	\end{align}
	Combing \eqref{off} and \eqref{diagonal} and choose $m\ge \max \{m_1,m_2\}$, one can conclude that $\Pr [C_\CA(\BMxi)\ge \delta] \le \eta$ with $\eta=\eta_1+\eta_2$, which completes the proof. 
\end{pfof}



\subsection{Bound of the off-block-diagonal term}\label{subsec:off_dia}
To show the bound of the off-block-diagonal term, the analysis is essentially the same as to that in \citet{krahmer2014suprema}. The main difference is that we derive a new decoupling inequality in Theorem~\ref{decoupling} (see there for a discussion of our contribution). We provide a complete proof here to be self-contained.
\begin{lemma}\label{bound_off_diagonal}
	The off-block-diagonal term  $O_{\CA}(\BMxi)$ satisfies 
	\begin{align}\label{inequ:chaos}
		\Pr\left[ O_{\CA} (\BMxi) \ge c_1 E +t \right] \le 2 \exp\left( -c_2\min \left\{ \frac{t^2}{V^2},\frac{t}{U}\right\} \right),
	\end{align}
	where 
	\begin{align}
		&E = \gamma_2 (\CA,\|\cdot\|_{\op}) \Big( \gamma_2 (\CA,\|\cdot\|_{\op}) +d_F(\CA)+ \sqrt{p} d_{2 \to 2} (\CA) \Big), \label{E}\\
		&V= d_{2\to 2}  (\CA) \Big( d_F(\CA)+ \gamma_2 (\CA,\|\cdot\|_{\op})  \Big), \label{V}\\
		& U=d_{2\to 2}^2 (\CA).\label{U}
	\end{align}
	Here, the quantities $d_F(\CA)$, $d_{2\to 2}(\CA)$, and $\gamma_2 (\CA,\|\cdot\|_{\op})$ satisfy
	\begin{align*}
		&d_F(\CA)=\sup_{A\in\CA}\|\bA\|_F, \\
		&d_{2\to 2} (\CA)=\sup_{\bA \in \CA} \|\bA\|_{\op}, \\
		&\gamma_2 (\CA,\|\cdot\|_{\op}) \le c \int_0^{d_{2\to 2 (\CA)}} \sqrt{ \log N(\CA,\|\cdot\|_\op,u)du},
	\end{align*} 
	where $N(\CA,\|\cdot\|_\op,u)$ is the covering number of $\CA$ with respect to the metric $(\|\cdot\|_\op,u)$.
\end{lemma}

To construct the proof of the Lemma~\ref{bound_off_diagonal}, we need the following two lemmas from \citet{krahmer2014suprema}. 
\begin{lemma}
	Let $\CA$ be a set of matrices, let $\BMxi$ be a sub-Gaussian random vector, and let $\BMxi'$ be an independent copy of $\BMxi$. Then for any $\bA\in \R^{m \times n}$, it holds that
	\begin{align*}
		\left\| \sup_{\bA\in \CA} \AngBra{\bA\BMxi,\bA\BMxi'} \right\|_{L_1} \lesssim \gamma_2 (\CA,\|\cdot\|_{\op})\cdot \|N_\CA(\BMxi)\|_{L_1} + \sup_{\bA\in \CA} \|\AngBra{\bA\BMxi, \bA\BMxi'}\|_{L_1},
	\end{align*}
	where $\|X\|_{L_1}:= \BE |X|$.
\end{lemma}
\begin{lemma}\label{copy}
	If $\BMxi'$ is an independent copy of $\BMxi$, then
	\begin{align*}
		\sup_{\bA \in \CA} \|\AngBra{\bA\BMxi,\bA\BMxi'}\|_{L_1} \lesssim \sqrt{p} d_F(\CA) d_{2 \to 2 } (\CA) + pd_{2\to 2}^2 (\CA). 
	\end{align*}
\end{lemma}

\begin{lemma}\label{exp_to_prob}
	Suppose $z$ is a random variable satisfying 
	\begin{align*}
		(\BE|z|^p)^{1/p}\le \alpha+\beta \sqrt{p}+\gamma p, \text{for all } p\ge p_0
	\end{align*}
	for some $\alpha,\beta,\gamma,p_0 >0$. Then, for $u\ge p_0$,
	\begin{align*}
		\Pr \big(|z|\ge e(\alpha+\beta\sqrt{u}+\gamma y)\big)\le e^{-u}.
	\end{align*}
\end{lemma}

\begin{pfof}{Lemma~\ref{bound_off_diagonal}}
	Applying the decoupling inequality in Theorem~\ref{decoupling} to $O_{\CA}(\BMxi)$, we have 
	\begin{align}\label{O_xi:1}
		\|O_{\CA}(\BMxi)\|_{L_1}&\le 4 \BE \sup_{\bA\in\CA} \left| \sum_{i,j=1}^n\BMxi_i^\top (\bA^i)^\top \bA^j  \BMxi'_j\right| = 4 \left\| \sup_{\bA\in \CA} \AngBra{\bA\BMxi,\bA\BMxi'} \right\|_{L_1} \nonumber\\
		&\lesssim \gamma_2 (\CA,\|\cdot\|_{2\to 2})\cdot \|N_\CA(\BMxi)\|_{L_1} + \sup_{\bA\in \CA} \|\AngBra{\bA\BMxi, \bA\BMxi'}\|_{L_1}. 
	\end{align}
	
	Following similar steps as in \citet{krahmer2014suprema}, one can show that 
	\begin{align*}
		\|N_\CA(\BMxi)\|_{L_1} \lesssim \gamma_2 (\CA,\|\cdot\|_{\op}) +d_F (\CA) + d_{2 \to 2} (\CA), 
	\end{align*}
	and from Lemma~\ref{copy} we have
	\begin{align*}
		\sup_{\bA\in \CA} \|\AngBra{\bA\BMxi, \bA\BMxi'}\|_{L_1} \lesssim d_F(\CA) d_{2\to 2} (\CA) + d^2_{2 \to 2} (\CA). 
	\end{align*}
	Susbstitutting them into \eqref{O_xi:1}, we have $\|O_{\CA}(\BMxi)\|_{L_1} \lesssim E+U+V$. Applying Lemma~\ref{exp_to_prob} completes the proof. 
\end{pfof}

\begin{theorem}[Decoupling Inequality]\label{decoupling}
	Let $\x_1,\x_2,\dots,\x_n \in \R^d$ be independent random vectors satisfying $\BE \x_i=\boldsymbol{0}$. Let $\x=(\x_1^\top, \x_2^\top,\dots,\x_n^\top)^\top$. Define 
	$
	\CB_o := \left\{ \bB\in \R^{dn \times d n} \right\}
	$
	with 
	\begin{align*}
		&\bB=\begin{bmatrix}
			\bB_{11} & \bB_{12} &\dots &\bB_{1n}\\
			\bB_{21} & \bB_{22} &\dots &\bB_{2n}\\
			\vdots & \vdots &\ddots &\vdots\\
			\bB_{n1} & \bB_{n2} &\dots &\bB_{nn}
		\end{bmatrix}, \quad\text{and}\quad \bB_{ij} \in \R^{d \times d} \text{ for all } i, j.
	\end{align*}
	Let $\x'$ is an independent copy of $\x$. Then, it holds that
	\begin{align*}
		\BE_\x \sup_{\bB \in \CB_o} \left |\sum_{i,j=1,i\neq j}^{n} \x_i^\top \bB_{ij} \x_j \right | \le \BE_{\x,\x'} \sup_{\bB \in \CB_o}  \left |4 \sum_{i,j=1}^{n} \x_i^\top \bB_{ij} \x'_j \right|.
	\end{align*}
\end{theorem}

\begin{remark}
	In \citet{rauhut2010compressive}, a similar decoupling inequality was presented, where $x_1,x_2,\dots, x_n$ are $1$-dimensional independent random variables. Here, we show that it can be generalized to multi-variate independent random vectors, without requiring the entries in each vector to be also independent.  The core of the proof  of this theorem follows from \citet{rauhut2010compressive}. 
\end{remark}

\begin{proof}
	Consider a sequence of independent random variable $\delta_1, \delta_2,\dots,\delta_n$ with each $\delta_i$ taking the values $0$ and $1$ with probability $\frac{1}{2}$. For $j\neq k$, $\BE \delta_j(1-\delta_k)=\frac{1}{4}$. Let $\delta=(\delta_1,\cdots,\delta_n)$.
	
	Since for any $\bB\in\CB$, $\bB_{ii}=\mathbf{0}$ for all $i$, we have 
	\begin{align*}
		\BE \sup_{\bB \in \CB_o} \left |\sum_{i,j=1,i\neq j}^{n} \x_i^\top \bB_{ij} \x_j \right|\le &4 \cdot  \BE_{\x} \sup_{\bB \in \CB_o} \left| \sum_{i,j=1,i\neq j}^{n} \BE[\delta_i (1-\delta_j)]  \x_i^\top \bB_{ij} \x_j \right| \\
		&\le 4 \cdot \BE_{\x,\delta} \sup_{\bB \in \CB_o} \left| \sum_{i,j=1,i\neq j}^{n} \delta_i (1-\delta_j)  \x_i^\top \bB_{ij} \x_j \right|.
	\end{align*}
	
	Next, we define the set $\CI(\delta):= \{k\in[n]: \delta_k=1\}$. Then, it follows from the Fubini’s theorem that
	\begin{align*}
		\BE \sup_{\bB \in \CB_o} \left |\sum_{i,j=1,i\neq j}^{n} \x_i^\top \bB_{ij} \x_j \right| \le  4  \cdot \BE_\delta \BE_\x \sup_{\bB \in \CB_o} \left| \sum_{i\in \CI(\delta)} \sum_{j \notin \CI(\delta)}  \x_i^\top \bB_{ij} \x_j\right|.
	\end{align*}
	For a fixed $\delta$, one can see that $\{\x_i\}_{i\in \CI(\delta)}$ and $\{\x_j\}_{j\notin \CI(\delta)}$ are independent. Therefore, we arrive at 
	\begin{align*}
		\BE \sup_{\bB \in \CB_o} \left |\sum_{i,j=1,i\neq j}^{n} \x_i^\top \bB_{ij} \x_j \right| \le  4  \cdot \BE_\delta \BE_\x \BE_{\x'} \sup_{\bB \in \CB_o} \left| \sum_{i\in \CI(\delta)} \sum_{j \notin \CI(\delta)}  \x_i^\top \bB_{ij} \x'_j\right|.
	\end{align*}
	Then, there exists $\delta_0$ such that 
	\begin{align*}
		\BE \sup_{\bB \in \CB_o} \left |\sum_{i,j=1,i\neq j}^{n} \x_i^\top \bB_{ij} \x_j \right| \le  4  \cdot \BE_\x \BE_{\x'} \sup_{\bB \in \CB_o} \left| \sum_{i\in \CI(\delta_0)} \sum_{j \notin \CI(\delta_0)}  \x_i^\top \bB_{ij} \x'_j\right|.
	\end{align*}
	Since $\BE[\x_i]=\BE[\x'_j]=\boldsymbol{0}$, we observe that 
	\begin{align*}
		&\BE_\x \BE_{\x'} \sup_{\bB \in \CB_o} \left| \sum_{i\in \CI(\delta_0)} \sum_{j \notin \CI(\delta_0)}  \x_i^\top \bB_{ij} \x'_j\right| \\
		&= \BE_\x\BE_{\x'} \sup_{\bB \in \CB_o} \left|  \sum_{i\in \CI(\delta_0)} \left( \sum_{j \notin \CI(\delta_0)}  \x_i^\top \bB_{ij} \x'_j + \sum_{j \in \CI(\delta_0)}  \x_i^\top \bB_{ij} \BE \x'_j\right) +\sum_{i\notin \CI(\delta_0)} \sum_{j=1}^{n}  \BE \x_i^\top \bB_{ij} \x'_j\right |.
	\end{align*}
	Applying the Fubini’s theorem, one has
	\begin{align*}
		\BE_\x \sup_{\bB \in \CB_o} \left |\sum_{i,j=1,i\neq j}^{n} \x_i^\top \bB_{ij} \x_j \right| \le  4  \cdot \BE_{\x,\x'} \sup_{\bB \in \CB_o} \left| \sum_{i=1}^n \sum_{j=1}^n  \x_i^\top \bB_{ij} \x'_j\right|,
	\end{align*}
	which completes the proof.
\end{proof}

\subsection{Bound of the block-diagonal term}\label{subsec:dia}

Now, we turn to bound the block-diagonal term $D_{\CA} (\BMxi)$. Define $\CB:=\{\bB=\bA^\top \bA: \bA\in \CA \in \R^{m \times nd} \}$. Recall that $\CA$ is defined after \eqref{RIP_const}. Given any $\bA\in \CA$, we have 
\begin{align*}
	\bB=\bA^\top \bA=\begin{bmatrix}
		\bB_{11}&\bB_{12}&\cdots&\bB_{1n}\\
		\bB_{21}&\bB_{22}&\cdots&\bB_{2n}\\
		\vdots&\vdots&\ddots&\vdots\\
		\bB_{n1}&\bB_{n2}&\cdots&\bB_{nn}
	\end{bmatrix}.
\end{align*}
One can see that $\bB_{ij}=\bB_{ji}$. Since $\bB$ is determined by $\bA$, we denote  $\BD(\bA)=\Blkdiag(\bB_{11},\bB_{22},\dots,\bB_{nn})$ be the block-diagonal matrix associated with $\bB=\bA^\top \bA$.  Then, $D_{\CA} (\BMxi)$ can be rewritten into 
\begin{align*}
	D_{\CA} (\BMxi)= \sup_{\bA \in \CA} \left| \sum_{i=1}^n \left( \BMxi_i^\top (\bA^i)^\top \bA^i \BMxi_i -  \AngBra{\bA^i,\bA^i} \right) \right|= \sup_{\bA \in \CA} \big| \BMxi^\top \BD(\bA) \BMxi - \Tr(\BD(\bA)) \big|.
\end{align*}

We have the following result for $D_{\CA} (\BMxi)$.



\begin{theorem}\label{bound:uniform}
	Let $\BMxi_1,\BMxi_2,\dots,\BMxi_n$ be independent random vector in $\R^d$.  Let $\BMxi=[\BMxi_1^\top, \BMxi_2^\top, \dots,\BMxi_n^\top]^\top$. Assume that $\BMxi$ satisfies the Hanson-Wright inequality \eqref{H-W}. Then, there exists a constant $k$ such that 
	\begin{align}\label{stacked:HW}
		\Pr \left[\sup_{A\in \CA } D_{\CA} (\BMxi) \ge t\right] \le 2 N\Big(\CA,\|\cdot\|_F,\frac{t}{6nd}\Big) \exp \left( -\frac{1}{ck^2}\min (m t^2 , m t ) \right),
	\end{align}
	where $N(\CA,\|\cdot\|_F,u)$ is the covering number of $u$-cover of $\CA$ with respect to the metric $\|\cdot\|_F$. 
\end{theorem}
\begin{proof}
	We construct the proof in two parts:
	\begin{enumerate}
		\item In part I, we prove 
		\begin{align}\label{HW:forward}
			\Pr \left[\big| \BMxi^\top \BD(\bA) \BMxi-\Tr(\BD(\bA)) \big| \ge t\right] \le 2 \exp \left( -\frac{1}{c k_1^2}\min (m t^2 , m t )\right).
		\end{align}
		Since it follows from the Hanson-Wright inequality \eqref{H-W} that there is $k_1>0$ such that
		\begin{align}\label{stacked:HW:norm}
			\Pr \left[\big| \BMxi^\top \BD(\bA) \BMxi- \Tr(\BD(\bA)) \big| \ge t\right] \le 2 \exp \left( -\frac{1}{c k_1^2}\min \left( \frac{t^2}{\|\BD(\bA)\|_F^2 \| },\frac{t}{\|\BD(\bA)\|_{\op}}\right) \right),
		\end{align}
		we will just show
		\begin{align*}
			&\sup_{\bA\in \CA} \|\BD(\bA)\|_F^2\le \frac{1}{m}, &\sup_{\bA\in \CA} \|\BD(\bA)\|_{\op}\le \frac{1}{m}.
		\end{align*}
		\item In part II, we complete the proof following a similar covering argument as in the proof of \citet[Theorem 9.8]{foucart2013mathematical}. 
	\end{enumerate}

	\textbf{Part I-\textit{Step 1}:} In this step, we bound $\sup_{\bA\in \CA}\|\BD(\bA)\|_F^2$. 
	For every $\bA\in \CA$, it holds that $\|\BD(\bA)\|_F^2=\sum_{i=1}^n \|B_{ii}\|_F^2$. Since $\bB=\bA^\top \bA$ and $\bA=\frac{1}{\sqrt{m}}\bR_{\Omega}V$ with $\bV$ given in  \eqref{circ:meas:2}, we have 
	\begin{align}
		\bB_{ii}= \frac{1}{m} (\bV^i)^\top \bR_{\Omega}^\top  \bR_{\Omega} \bV^i,
	\end{align}
	where $\bV^i$ is the $i$th block column of $\bV$. Recall the definition of $\bR_{\Omega}$, one can derive that 
	\begin{align}\label{diag:blocks}
		\bB_{ii}= \frac{1}{m}\sum_{j\in \Omega} (\bV^i_j)^\top \bV^i_j,
	\end{align}
	where $\bV^i_j$ is the $j$th row of $\bV^i$. 
	Due to the block-diagonal structure of $\bB$, it holds that $\|\BD(\bA)\|_F^2=\sum_{i=1}^n \|\bB_{ii}\|_F^2$. 
	Observe that
	\begin{align*}
		\sum_{i=1}^n \|\bB_{ii}\|_F^2 = \frac{1}{m^2} \sum_{i=1}^n \left( \Big\|\sum_{j\in \Omega} (\bV^i_j)^\top \bV^i_j \Big\|_F^2 \right).
	\end{align*}
	Using the expression of $\bV$ defined in \eqref{circ:meas:2}, one can derive that 
	\begin{align*}
		\sum_{i=1}^n \left( \Big\|\sum_{j\in \Omega} (\bV^i_j)^\top \bV^i_j \Big\|_F^2 \right) &\le m \left(\sum_{i=1}^n \|\va_i\va_i^\top\|_F^2 \right)+ 2 m \sum_{i\neq j} \| \va_i\va_i^\top\|_F \| \va_j\va_j^\top\|_F \\& =m (\|\va_1\|^2+\dots+\|\va_n\|^2)^2,
	\end{align*}
	where the last inequality has used the fact that $\|\va_i \va_i\|_F=\|\va_i\|^2$.
	Since $\|\va_1\|^2+\dots+\|\va_n\|^2\le 1$, it follows that
	\begin{align*}
		\sum_{i=1}^n \|\bB_{ii}\|_F^2 \le \frac{1}{m^2} \cdot m =\frac{1}{m}, 
	\end{align*}
	which proves $\sup_{\bA\in \CA} \|\BD(\bA)\|_F^2\le \frac{1}{m}$.

	\vspace{5pt}
	\textbf{Part I-\textit{Step 2}:}  In this step, we bound $\sup_{A\in \CA}\|D(\bA)\|_{\op}$. Since $\BD(\bA)$ is symmetric for the $\BD(A)$ associated with any $\bA \in \CA$, by the definition of the operator norm, we know $\|\BD(\bA)\|_{\op}=\lambda_{\max}(\BD(\bA))$. 
	Since $\BD(\bA)$ has the diagonal form $\Blkdiag(\bB_{11},\bB_{22},\dots,\bB_{nn})$, one can derive that $\lambda_{\max}(\BD(\bA))=\max\{\lambda_{\max}(\bB_{ii}):i=1,2,\dots,p\}$. Given the expressions of $\bB_{ii}$ shown in \eqref{diag:blocks}, it follows from Lemma~\ref{eig:max} that $\lambda_{\max}(\bB_{ii})\le \frac{1}{m}$ for all $i=1,2,\dots,n$, which implies that $\sup_{\bA\in \CA} \|\BD(\bA)\|_{\op}=\lambda_{\max}(\BD(\bA))\le \frac{1}{m}$.

	\vspace{15pt}
	
	\textbf{Part II:} Let $\CA_u$ be the smallest $u$-cover of the set $\CA$ with respect to the metric $\|\cdot\|_F$, and its cardinality is $N(\CA, \|\cdot\|_F, u)$. It holds that $\min_{A_u\in \CA_u} \|\bA- \bA_u\|_F \le u$ for all $\bA \in \CA$.  From the inequality \eqref{HW:forward}, one can derive that 
	\begin{align*}
		\Pr \left[ \exists \bA_u \in \mathcal \bA_u: \left | \BMxi^\top \BD(\bA_u) \BMxi- \Tr(\BD(\bA_u)) \right|> \eta \right] &\le \sum_{\bA_u \in \mathcal \bA_u} \Pr \left[ \left | \BMxi^\top \BD(\bA_u) \BMxi- \Tr(\BD(\bA_u)) \right|> \eta \right] \\
		&\le  2 N(\CA, \|\cdot\|_F, u) \exp \left( -\frac{1}{ck_1^2}\min (m \eta^2 , m \eta )\right),
	\end{align*}
	where the second inequality follows from Lemma~\ref{HW:individual_to_pack}. 
	Then, one can derive that
	\begin{align*}
		\Pr \left[ \sup_{\bA_u\in \mathcal \bA_u} \left | \BMxi^\top \BD(\bA_u) \BMxi- \Tr(\BD(\bA_u)) \right|\le  \eta \right] \ge 1- 2 N(\CA, \|\cdot\|_F, u) \exp \left( -\frac{1}{ck_1^2}\min (m \eta^2 , m \eta )\right).
	\end{align*}
	Now, observe that for any $\bA \in \CA$, it holds that
	\begin{align*}
		&\left |\BMxi^\top \BD(\bA) \BMxi-\Tr(\BD(\bA)) \right| \\
		&= \left | \BMxi^\top \BD(\bA_u) \BMxi- \Tr(\BD(\bA_u)) + \BMxi^\top \big(\BD(\bA) - \BD(\bA_u) \big) \BMxi +\Tr(\BD(\bA_u))- \Tr(\BD(\bA)) \right| \\
		&\le   \left | \BMxi^\top \BD(\bA_u) \BMxi-\Tr(\BD(\bA_u)) \right| + \left| \BMxi^\top \big(\BD(\bA) - \BD(\bA_u) \big) \BMxi  \right| +\left| \Tr \big(\BD(\bA_u)-\Tr(\BD(\bA))\big) \right|. 
	\end{align*}
	Assume that $\left | \BMxi^\top \BD(\bA_u) \BMxi -\Tr(\BD(\bA)) \right|\le \eta$, and we arrive at 
	\begin{align*}
		\left |\BMxi^\top \BD(\bA) \BMxi- \Tr(\BD(\bA)) \right| \le \eta + \|\BD(\bA) - \BD(\bA_u)\|_2 \|\BMxi\|_2^2 + nd \|\BD(\bA) - \BD(\bA_u)\|_2,
	\end{align*}
	where $\bA_u$ is chosen such that $\|\bA-\bA_u\|_F \le u$.
	Now, observe that 
	\begin{align*}
		\|\BD(\bA) - \BD(\bA_u)\|_2 \le\|\BD(\bA) - \BD(\bA_u)\|_F \le   \|\bB- \bB_u\|_F,
	\end{align*}
	and 
	\begin{align*}
		\bB- \bB_u =(\bA-\bA_u) ^\top (\bA-\bA_u) + \bA_u^\top(\bA-\bA_u) + (\bA-\bA_u)^\top \bA_u.
	\end{align*}
	Using the fact that $\|\bA_u\|_F\le 1$, we have 
	\begin{align*}
		\|\BD(\bA) - \BD(\bA_u)\|_2\le   \|\bB- \bB_u\|_F \le u^2 +2 u
	\end{align*}
	Since $\BMxi$ satisfies the Hanson-Wright inequality \eqref{H-W}, $\|\BMxi\|_2^2\le nd $ with probability at least $1-2\exp(-nd)$. Therefore, $\left |\BMxi^\top \BD(\bA) \BMxi- \Tr(\BD(\bA)) \right| \le \eta + 2 nd (u^2+2u)$. Choose $\eta=\frac{t}{2}$ and $u=\frac{t}{6nd}$, we deduce, using the fact $u^2+2u<3u$ for $u<1$, that
	\begin{align*}
		\Pr \left[ \sup_{\bA\in \mathcal A} \left | \BMxi^\top \BD(\bA) \BMxi- \Tr(\BD(\bA)) \right|\ge  t \right] \le 2 N \Big(\CA,\|\cdot\|_F, \frac{t}{6nd}\Big) \exp \left( -\frac{1}{4ck_1^2}\min (m t^2 , m t )\right),
	\end{align*}
	which completes the proof.
\end{proof}

\subsection{Useful Lemmas and Theorems}

\begin{lemma}\label{eig:max}
	Let $\va=[\va_1^\top,\va_2^\top,\dots,\va_n^\top]^\top$, where $\va_i\in \R^d$, be a vector that satisfies $\|\va\| \le 1$. Consider the matrix $\bP=\sum_{i=1}^n k_i \va_i \va_i^\top$, where $k_i\in \{0,1\}$, $i=1,2,\dots,n$. Then, the largest eigenvalue of $\bP$ satisfies $\lambda_{\max}(\bP)\le 1$ for any possible combinations of $k_i$'s. 
\end{lemma}

\begin{proof}
	It can be seen that the matrix $\bP$ satisfies $\bP^\top=\bP$. Therefore, $\lambda_{\max}(\bP)=\|\bP\|$. It then follows that 
	\begin{align*}
		\lambda_{\max}(\bP)=\|\bP\| \le \sum_{i=1}^n k_i\|\va_i \va_i^\top\|\le \sum_{i=1}^n k_i\|\va_i \va_i^\top\|_F \le \sum_{i=1}^n \|\va_i \va_i^\top\|_F.
	\end{align*}
	Observe that $\|\va_i \va_i^\top\|_F= \|\va_i\|^2$. As a consequence, we have
	\begin{align*}
		\lambda_{\max}(\bP)\le \sum_{i=1}^n \|\va_i \va_i^\top\|_F= \sum_{i=1}^n \|\va_i\|^2  =\|\va\|^2 \le 1,
	\end{align*}
	which completes the proof. 
\end{proof}

\begin{lemma}\label{HW:individual_to_pack}
	Assume $\BMxi_1,\BMxi_2,\dots,\BMxi_n \in \R^d$ are independent isotropic sub-Gaussian random vectors, and each $\BMxi_i$ satisfies
	\begin{align*}
		\Pr \left[|\BMxi_i^\top \bP \BMxi_i -\BE (\BMxi_i^\top \bP \BMxi_i)|\ge \eta \right] 
		\le 2 \exp \left( \frac{1}{c} \min \left\{ \frac{\eta^2}{k^4 \|\bP\|_F^2},\frac{\eta}{k^2 \|\bP\|_\op}\right\}\right) 
	\end{align*}
	for any positive semi-definite matrix $\bP\in\R^{d\times d}$. Consider positive semi-definite matrices $\bP_1,\bP_2,\dots,\bP_n \in\R^{d\times d}$. Let $\BD(P):=\Blkdiag(\bP_1,\bP_2,\dots,\bP_n)$ and $\BMxi=[\BMxi_1^\top,\BMxi_2^\top,\dots,\BMxi_n^\top]^\top$. Then, there exists $k_1>0$ such that
	\begin{align}
		\Pr \left[|\BMxi^\top \BD(\bP) \BMxi -\BE (\BMxi^\top \BD(\bP) \BMxi)|\ge \eta \right] 
		\le 2 \exp \left( \frac{1}{c} \min \left\{ \frac{\eta^2}{k^4 \|\BD(\bP)\|_F^2},\frac{\eta}{k^2 \|\BD(\bP)\|_\op}\right\}\right).
	\end{align}
\end{lemma}
\begin{proof}
	Due to the special block diagonal structure of $\BD(\bP)$, it can be seen that $\BMxi^\top \BD(\bP) \BMxi-\Tr(\BD(\bP))=\sum_{i=1}^n \big(\BMxi_i ^\top \bP_{i} \BMxi_i -\Tr(\bP_{i})\big)$. Each  $\BMxi_i ^\top \bP_{i} \BMxi_i -\Tr(\bP_{i})$ is a mean zero sub-exponential random variable. 
	Applying the Bernstein's inequality \cite[Chap. 2]{vershynin2018high}, one can derive that there exists $k_1>0$ such that  
	\begin{align*}
		\Pr \left[\left| \sum_{i=1}^n (\BMxi_i^\top \bP_{i} \BMxi_i -\Tr(\bP_{i})) \right| \ge t\right] \le 2 \exp \left( -\frac{1}{ck_1^2}\min \left( \frac{t^2}{\sum_{i=1}^n \|\bP_{i}\|_F^2 },\frac{t}{\sup_{i} \|\bP_{i}\|_{\op}}\right) \right).
	\end{align*}
	Using the fact that $\sum_{i=1}^n \|\bP_{i}\|_F^2 = \|\BD(\bP)\|_F^2$ and $\sup_{i} \|\bP_{i}\|_{\op} =\|\BD(\bP)\|_\op$ completes the proof. 
\end{proof}

\begin{theorem}\label{diff_of_HW}
	Let $\BMxi\in\R^d$ be an isotropic and $1$-Sub-Gaussian random vector satisfying the Hanson-Wright inequality
	\begin{align}\label{ineq:HW1}
		\Pr \left[|\BMxi^\top  \bA \BMxi -\BE (\BMxi^\top  \bA \BMxi)|\ge \eta \right]	&\le 2 \exp \big(- \frac{1}{c} \min \big\{ \frac{\eta^2}{k^4 \|\bA\|_F^2},\frac{\eta}{k^2 \|\bA\|_\op}\big\}\big), ~~~~~\forall \eta>0,
	\end{align}
	for any positive semi-definite matrix $A$. 
	Let $\BMxi'$ be an independent copy of $\BMxi$. Then, $\BMxi-\BMxi'$ also satisfies the Hanson-Wright inequality.
\end{theorem}
\begin{pfof}{Theorem~\ref{diff_of_HW}} 
	To construct the proof, we will use the following lemma from \cite{krahmer2014suprema}. 
	\begin{lemma}\label{ineq:invl:gaussian}
		Let $\x_1,\x_2,\dots,\x_d\in \BC ^m$ and $\BT\in \BC ^m$. If $\BMxi=[\xi_1,\dots,\xi_d]^\top$ is an isotropic, $1$-Sub-Gaussian random vector and $\y=\sum_{i=1}^{d} \xi_i\x_i $, then for every $p\ge 1$,
		\begin{align}\label{ineq:expectation}
			\left(\BE \sup_{\bh\in \BT}|\AngBra{\bh,\y}|^p\right)^{1/p}\le c\left(\BE \sup_{\bh\in \BT} |\AngBra{\bh,\bg}|+\sup_{\bh\in \BT} (\BE |\AngBra{\bh,\y}|^p)^{1/p}\right),
		\end{align}
		where $c$ is a constant and $\bg =\sum_{j=1}^{d} g_j\x_j$ for $g_1,g_2,\dots,g_d$ independent standard normal random variables. \TiaQED
	\end{lemma}

	Observe that 
	$
	(\BMxi-\BMxi')^\top  \bA (\BMxi-\BMxi')= \BMxi ^\top \bA \BMxi +(\BMxi') ^\top \bA \BMxi'-(\BMxi')^\top \bA \BMxi -\BMxi^\top A \BMxi' 
	$. Since the first two terms satisfy the Handson-Wright inequality, it suffices to show
	\begin{align}\label{xi:xi'}
		\Pr \left[|\BMxi^\top  \bA \BMxi'|\ge \eta \right]	&\le 2 \exp \big(- \frac{1}{c_1} \min \big\{ \frac{\eta^2}{k_1^4 \|\bA\|_F^2},\frac{\eta}{k_1^2 \|\bA\|_\op}\big\}\big), ~~~~~\forall \eta>0.
	\end{align}
	Define set $\BS=\{\bA \x:\|\x\|\le 1,\x\in\R^d\}$. Since $\BMxi$ is $1$-Sub-Gaussian, the random variable $\AngBra{\BMxi,\bA\BMxi'}$ is also Sub-Gaussian conditionally on $\BMxi'$. Therefore,  
	\begin{align*}
		\left(\BE|\AngBra{\BMxi,\bA\BMxi'}|^p\right)^{1/p} &= \left( \big(\BE_{\BMxi'}  (\BE_{\BMxi}|\AngBra{\BMxi,\bA\BMxi'}|^p)^{1/p} \big)^{p} \right)^{1/p} \stackrel{(a)}{\lesssim} \left(\BE_{\BMxi'}(\sqrt{p})^{p}\|\bA\BMxi'\|^p_2\right)^{1/p}\\
		& =\sqrt{p} \left(\BE_{\BMxi'} \sup_{\bh \in \BS} |\AngBra{\bh,\BMxi'}|^p\right)^{1/p}\\
		&\stackrel{(b)}{\lesssim} c\sqrt{p} \left(\BE \sup_{\bh\in \BS} |\AngBra{\bh,\bg}|+\sup_{\bh\in \BS}(\BE |\AngBra{\bh,\BMxi'}|^p)^{1/p}\right),
	\end{align*}
	where the inequality (a) has used the inequality $\sup_{\theta \in \mathcal S^{d-1}}(\BE |\AngBra{X,\theta}|^p)^{1/p}\le L \sqrt{p}$ for an $L$-Sub-Gaussian random vector $X$ with $\mathcal S^{n-1}$ the unit sphere in $\R^d$ (see \cite{krahmer2014suprema}), and the inequality (b) has used Lemma~\ref{ineq:invl:gaussian}.
	Next, we bound the two terms on the right-hand side of the above inequality.  For the first term, one can derive that
	\begin{align*}
		\BE \sup_{\bh\in \BS} |\AngBra{\bh,\bg}|= \BE \|\bA\bg\|_2 \le (\BE \|A\bg\|_2^2)^{1/2}=\|\bA\|_F.
	\end{align*}
	For the second term, it holds that 
	\begin{align*}
		\sup_{\bh\in \BS}(\BE |\AngBra{\bh,\BMxi'}|^p)^{1/p} =\sup_{\bz,\|\bz\|\le 1} (\BE |\AngBra{A \bz,\BMxi'}|^p)^{1/p} \lesssim \sqrt{p}\sup_{\bz,\|\bz\|\le 1}\|A\bz\|_2 \le \sqrt{p}\|A\|_\op.
	\end{align*}
	Now, one can use Lemma~\ref{exp_to_prob} to show the inequality \eqref{xi:xi'}.
\end{pfof}

\section{Circulant Matrices Do Not Obey Matrix RIP} \label{Failure_Matrix_RIP}

Let $\BMxi\in\R^p$ be a vector with i.i.d.~complex normal distribution $\CN_\CC(0,1)$. Let $\bC_{\BMxi}$ be the circulant matrix induced by $\BMxi$, i.e.,
\begin{align*}
	\bC_{\BMxi}=\begin{bmatrix}
		\xi_{p}&\xi_{p-1}&\cdots&\xi_{1}\\
		\xi_{1}&\xi_{p}^\top&\cdots&\xi_{2}\\
		\vdots&\vdots&\ddots&\vdots\\
		\xi_{p-1}&\xi_{p-2}&\cdots&\xi_{p}
	\end{bmatrix} \in \R^{p \times p}.
\end{align*}
Note that $\bC_{\BMxi}=\frac{1}{\sqrt{p}}\bF_p^*\diag(\vd)\bF_p$ where $\bF_p$ is the Discrete Fourier Transform with $\bF_p[i,j]=\omega^{ij}$ for $0\leq i,j<p$ and $\omega=e^{-\iota 2\pi/p}$. Here diagonal $\vd$ relates to the singular values of $\bC_{\BMxi}$ and obeys $\vd=\frac{1}{\sqrt{p}}\bF_p\BMxi\distas\CN_\CC(0,1)$.

Suppose $p=p_1p_2$ for some positive integers $p_1,p_2$. Let $f_k$ denote the $k$th row of $\bF_p$. Observe that, for all $1\leq k\leq p$, $f_k$ corresponds to a rank $1$ matrix $\text{mtx}(f_k)=\va\vb^\top\in \R^{p_1\times p_2}$ since it can be written as the Kronecker product
\begin{align*}
	f_k=\va\otimes\vb \in\R^{p_1 p_2}, \quad\text{where}\quad \va=\begin{bmatrix}\omega^0\\\omega^{p_2k}\\\vdots\\\omega^{(p_1-1)p_2k}\end{bmatrix},\quad \vb=\begin{bmatrix}\omega^0\\\omega^{k}\\\vdots\\\omega^{(p_2-1)k}\end{bmatrix}.
\end{align*}

\begin{lemma}
	Consider the set of vectors $\CM=(f^\ast_k/\sqrt{p})_{k=1}^p$ that are equivalent to rank-1 matrices in $\R^{p_1\times p_2}$ with unit Frobenius norm. Let $\bar{\bC}=\frac{1}{\sqrt{n}}\text{Mask}_n(\bC_{\BMxi})$ be the subsampled normalized circulant matrix with $n$ rows. Let $M_p=\max_{1\leq k\leq p} ||d_k|^2-1|$. We have that
	\begin{align}
		M_p=\sup_{\vv\in\CM} \big|\|\bar{\bC}\vv\|_2^2-1 \big|.\label{mp bound}
	\end{align}
	Specifically, for all $\delta\leq 1$, there exists a constant $0<c<1$ such that with probability at least $1-c^p$, $M_p>\delta$ and $\bar{\bC}$ does not obey $\delta$-RIP over rank $1$ matrices.
\end{lemma}

\begin{pfof}{Lemma~\ref{norip}}
	Let $\ve_k$ be the $k$th element of the standard basis. Let $\bA=\frac{1}{\sqrt{n}}\text{Mask}_n(\bF_p^\ast)$ be the subsampled normalized DFT matrix. Observe that each column $\va_k$ of $A$ has unit Euclidean norm as it contains $n$ entries with powers of $\omega$. Now, observe that 
	\[
	\frac{1}{\sqrt{n}}\text{Mask}_n(\bC_{\BMxi})\frac{f^\ast_k}{\sqrt{p}}=A\text{diag}(\vd)\frac{\bF_pf^\ast_k}{p}=A\text{diag}(\vd)\ve_k=d_k\va_k.
	\]
	Thus, $\|\bar{\bC}f^\ast_k\|_{\ell_2}=|d_k|$ which concludes the result \eqref{mp bound} when we take max/min over $1\leq k\leq p$.
	
	Now, set $c=\Pr(||d_k|^2-1|\leq 1)$. Clearly $c<1$ since $d_k\sim\CC\CN(0,1)$ is unbounded. Using independence of $d_k$'s, we find that $\Pr(M_p>1)\geq 1-c^p$. Finally, note tha $M_p>1$ implies lack of $\delta$-RIP for any $\delta\leq 1$ over rank-$1$ matrices.
\end{pfof}
	

\newpage
\begin{center}
	\LARGE \bf{Part II: Supporting Information for Bandits}
\end{center}

In Part I, we have derived some general results for RIP of block circulant measurements. In this part, we use the obtained results to analyze our bandit problem. This part is organized as follows. 
\begin{enumerate}
	\item In Section~\ref{HW:chosen}, we show that if the random context vectors at each round satisfies the Hanson-Wright inequality, any determinstically chosen context vector satisfies that, too. This property ensures that the selected context vectors in our algorithms always satisfies the Hanson-Wright inequality. 
	
	\item In Section~\ref{RE}, we use the RIP results in Part I-Section~\ref{ciclulant:RIP} to show that circulant matrices generated by non-isotropic random vectors satisfy the restricted eigenvalue condition (RE) for block-sparse vectors. 
	
	\item In Section \ref{analy:poor}, we use the results in Section \ref{RE} to analyze the regret upper bounds for our algorithms in Section~\ref{main_results} of the main text. 
	\item Finally, in Section~\ref{experi}, some details of our experiments in the main text are provided. Further, we also present some additional experiment results.
\end{enumerate}

\section{Deterministically-Chosen Context Vectors Satisfy the Hanson-Wright Inequality}\label{HW:chosen}
In our algorithms in data-poor and data-rich regime, we select the following greedy action at each round:
\begin{equation*}
	a_t \in \argmax\nolimits_{a\in [K]} \AngBra{\x_{t,a},\hat \BMtheta_{{i-1}}}, ~~~~~~i\ge 1.
\end{equation*}
The following theorem ensures that each selected context vector $\x_{t,a_t}$ also satisfies the Hanson-Wright inequality if each $x_{t,a},a\in [K]$, does. The property will play very important role in the analysis of our algorithms.  

\begin{theorem}\label{Th:HW:chosen}
	Suppose $\x_1,\x_2,\dots,\x_K\in \R^d$ are i.i.d. random vectors and $\bA\in\R^{d\times d}$.  Let $E_\x = \BE [\x_i^\top \bA \x_i]$. Assume that $x_i$ satisfies the Hanson-Wright inequality (stated below in equivalent forms)
	\begin{align}
		\Pr \Big( |\x_i^\top \bA \x_i -E_\x | \ge t\Big) \le 2 \exp\Big( -C \min \big\{\frac{t^2}{\|\bA\|_F^2},\frac{t}{\|\bA\|_\op} \big\}\Big) \label{equi_bound:1}\\
		\Longleftrightarrow \Pr \Big( |\x_i^\top \bA \x_i -E_\x|\ge \sqrt{\tau} \|\bA\|_F/\sqrt{C} +\tau \|\bA\| /C |\Big) \ge 2 e^{-\tau}. \label{equi_bound:2}
	\end{align}
	Let $\BMxi$ be a random vector deterministically chosen from the set $\{\x_a,a\in[K]\}$. The identical Hanson-Wright bound holds up to $\CO(\log(2K))$, that is,
	\begin{align*}
		\Pr \Big( |\BMxi^\top \bA \BMxi - \BE [\BMxi^\top \bA \BMxi]|\ge c\log (2K)t \Big) \le 2 \exp \Big( -C \min \big( \frac{t^2}{\|\bA\|_F^2},\frac{t}{\|\bA\|_\op} \big) \Big).
	\end{align*}
\end{theorem}
\begin{proof}
	Through union bound, we have that
	\begin{align*}
		\Pr \Big( |\BMxi^\top \bA \BMxi -E_\x|\ge t \Big) \le Q(t):= \min \left( 1, 2K \exp \Big( -C \min \big\{\frac{t^2}{\|\bA\|_F^2},\frac{t}{\|\bA\|_\op} \big\} \Big) \right).
	\end{align*}
	Denote $R = |\x_i^\top A \x_i-E_\x|$. Observe that
	\begin{align*}
		\BE[R] &= -\int_0^{\infty} t d Q(t)\\
		& = -[tQ(t)]_0^\infty +\int_0^\infty Q(t) dt\\
		& = \int_0^\infty \min \left( 1,2K \exp \Big( -C \min \big\{\frac{t^2}{\|\bA\|_F^2},\frac{t}{\|\bA\|} \big\} \Big) \right) dt\\
		& \le \int_0^\infty \min \big( 1, 2Ke^{-C\frac{t^2}{\|\bA\|_F^2}} \big) dt + \int_0^\infty \min \big( 1, 2Ke^{-C\frac{t}{\|\bA\|_\op}} \big) dt. 
	\end{align*}
	Let $T_1 = \sqrt{\log (2K)/C}\|\bA\|_F$. We can derive
	\begin{align*}
		\int_0^\infty \min \big( 1, 2Ke^{-C\frac{t^2}{\|\bA\|_F^2}} \big) dt &=T_1 + \int_{T_1}^\infty 2K e^{-C\frac{t^2}{\|\bA\|_F^2}} dt \\
		& =  T_1 + \int_{0}^\infty 2K e^{-C\frac{(t+T_1)^2}{\|\bA\|_F^2}} dt \\
		& \le  T_1 +  \int_{0}^\infty e^{-C\frac{t^2}{\|\bA\|_F^2}} dt \\
		& = T_1 + \frac{\sqrt{\pi}}{2} \|\bA\|_F/\sqrt{C} \le T_1 + \|\bA\|_F/\sqrt{C}. 
	\end{align*}
	Likewise, let $T_2 = \log(2K) \|\bA\|/C$, and we have
	\begin{align*}
		\int_0^\infty \min \big( 1, 2Ke^{-C\frac{t}{\|\bA\|}} \big) dt &=T_2 + \int_{T_1}^\infty 2K e^{-C\frac{t}{\|\bA\|}} dt\\
		&=T_2 + \int_{0}^\infty 2K e^{-C\frac{t+T_1}{\|\bA\|}} dt \\
		&=T_2 + \int_{0}^\infty e^{-C\frac{t}{\|\bA\|}} dt\\
		&= T_2 + \|\bA\|/C. 
	\end{align*}
	Combing the bounds above and using $\log(2K)\ge 1$ ($K>1$ since it is the number of arms) we arrive at 
	\begin{align*}
		\BE [R] = \BE [|\x_i^\top A \x_i-E_\x|] &\le \big(\sqrt{\log(2K)}+1 \big)\|\bA\|_F/\sqrt{C} + (\log(2K)+1) \|\bA\|/C\\
		&\le 4 \log(2K) \big( \|\bA\|_F/\sqrt{C}+\|\bA\|/C \big).
	\end{align*}
	Now, applying the equivalent bound \eqref{equi_bound:2}, we have 
	\begin{align}
		&\Pr\big( |\BMxi^\top A \BMxi -E_\x| \ge \sqrt{\tau} \|A\|_F/ \sqrt{C} +\tau \|A\|/C \big) \le 2K e^{-\tau} \nonumber \\
		\Longleftrightarrow & \Pr \big( |\BMxi^\top A \BMxi -E_\x| \ge \sqrt{\tau} \sqrt{\log(K)} \|A\|_F/ \sqrt{C} +\tau\log(K) \|A\|/C \big) \nonumber\\
		\Longrightarrow &  \Pr \big( |\BMxi^\top A \BMxi -E_\x| \ge \sqrt{\tau} \sqrt{\log(2K)} \|A\|_F/ \sqrt{C} +\tau\log(2K) \|A\|/C \big). \label{xi_x:bound}
	\end{align}
	Let $E_{\BMxi} = \BE [\BMxi^\top \bA \BMxi]$, and we can derive that
	\begin{align*}
		|E_{\BMxi} -E_\x| \le |\x_i^\top A \x_i-E_\x|\le 4 \log(2K) \big( \|\bA\|_F/\sqrt{C}+\|\bA\|/C \big). 
	\end{align*}
	Combining this with \eqref{xi_x:bound}, we obtain that, with probability at $1-2e^{\tau}$,
	\begin{align*}
		|\BMxi^\top \bA \BMxi -E_{\BMxi}| &\le \underbrace{\sqrt{\tau} \log (2K) \|A\|_F /\sqrt{C} +\tau \log(2K) \|A\|/C}_{\text{Probability}} + \underbrace{4 \log(2K) \big( \|A\|_F / \sqrt{C} +\|A\|/C \big)}_{\text{Expectation}}\\
		&\le 2 \log (2K) \Big( \sqrt{\tau +16} \|A\|_F /\sqrt{C} + (\tau + 16) \|A\|/C \Big). 
	\end{align*}
	This implies that, with probability at least $1-2 e^{\tau}$, for an absolute constant $c>0$, we have
	\begin{align*}
		|\BMxi^\top \bA \BMxi -E_{\BMxi}|  \le c \log(2K) \Big( \sqrt{\tau } \|A\|_F /\sqrt{C} + \tau  \|A\|/C \Big).  
	\end{align*}
	Now, using the equivalence between \eqref{equi_bound:1} and \eqref{equi_bound:2}, we prove 
	\begin{align*}
		\Pr \big( |\BMxi^\top \bA \BMxi -E_{\BMxi}|\ge c\log(2K)t  \big) \le 2 \exp \Big( -C \min \big\{\frac{t^2}{\|\bA\|_F^2},\frac{t}{\|\bA\|_\op} \big\} \Big).
	\end{align*}
\end{proof}

\section{From RIP to Restricted Eigenvalue Condition (RE)}\label{RE}

In Theorem~\ref{RIP:general:FORbandits}, we have shown that under a certain assumption, the block Toeplitz matrix of the following form
\begin{align}\label{Teoplitz}
	\BMXi=\frac{1}{\sqrt{m}}\begin{bmatrix}
		\BMxi_{h}^\top&\BMxi_{h-1}^\top&\cdots&\BMxi_{1}^\top\\
		\BMxi_{h+1}^\top&\BMxi_{h}^\top&\cdots&\BMxi_{2}^\top\\
		\vdots&\vdots&\ddots&\vdots\\
		\BMxi_{m+h-1}^\top&\BMxi_{m+h-2}^\top&\cdots&\BMxi_{m}^\top
	\end{bmatrix}.
\end{align}
satisfies RIP if $m \ge O(s_0\log^2(s_0)\log^2(hd))$, where each $\xi_i$ is assumed to be isotropic Sub-Gaussian.

However, as for our contextual bandit problem, the context vectors at each step, $\x_{t,a},a\in[K]$, are not isotropic. From Assumption~\ref{ass:cv_concentr} and our algorithms in Section~\ref{main_results}, one can derive from Theorem~\ref{Th:HW:chosen} that the measurement matrix generated by selected context vectors in our bandit problem has the following form
\begin{align}\label{matrix:RE}
	\BMPsi = \BMXi 
	\begin{bmatrix}
		\sqrt{\Sigma} & & & \\
		& \sqrt{\Sigma} & & \\
		& & \ddots & \\
		& & & \sqrt{\Sigma}
	\end{bmatrix}+\onebb\bmu^\top,
\end{align}
where $\lambda_{\max}(\Sigma)\ge\lambda_{\min}(\Sigma)>\sigma>0$. This measurement matrix $\BMPsi$ has taken into account that the covariance matrix of the selected context vectors, $\x_{t,a_t}$, is $\Sigma$ (which is not simply an identity matrix), and they have a non-zero mean $\bmu$. As shown in \eqref{cic:equation}, the core of our bandit problem becomes to learn the $s$-block-sparse vector $\BMphi$ from the following noisy measurement
\begin{equation}
	\y= \BMPsi \BMphi+\bEps.
\end{equation}

In this section, we will show that $\BMPsi$ satisfies a restricted eigenvalue condition for  $s$-block-sparse vectors if $\BMXi$ satisfies RIP.  First, let us define the restricted eigenvalue condition for block sparse vectors. 

\begin{definition}[RE for block sparse vectors\footnote{Recall that given $\x=[\x^\top_1,\x^\top_2,\dots,\x^\top_h]^\top$ with each $\x_i\in\R^d$, the $\ell_{2,1}^{(d)}$ norm is defined as $\|\x\|_{2,1}^{(d)}=\sum_{i=1}^{h}\|\x_i\|_2$. The block support set $\CS$ of $\x $ is defined as $\CS:=\{i\in\{1,2,\dots,h\}:\|\x_i\|_2\neq 0\}$. Further, given a set $Q\subseteq \{1,2,\dots,h\}$, $\va_Q=[\tilde{\va}^\top_1, \tilde{\va}^\top_2, \dots,\tilde{\va}^\top_h]$ with $\tilde{\va}_i= {\va}_i$ for all $i\in Q$ and $\tilde{\va}_i= 0$ for any other $i$.}]\label{Def:RE}
	Define the set $\mathbb{C}^{(d)}_{\alpha} (S):=\{ \va \in \R^{hd}: \|\va_{S^c}\|^{(d)}_{2,1} \le \alpha \|\va_S\|^{(d)}_{2,1}\}$, where $\CS$ is the block support set of the $s$-block-sparse vector $\va \in \R^{hd}$ and $\CS^c$ is the  complement of $S$. The matrix $\bM\in \R^{ m \times d}$ satisfies the \textit{restricted eigenvalue} (RE) condition over $\CS$ with parameter $\kappa$ and  $\alpha$ if
	\begin{align*}
		&\frac{1}{m}\|\bM \Delta\|_2^2 \ge \kappa  \|\Delta\|_2^2, &\text{for all\;} \Delta\in \mathbb{C}_{\alpha}^{(d)} (\CS).
	\end{align*}
\end{definition}

\begin{remark}
	In Chap. 7 of \cite{wainwright2019high}, the definition of restricted eigenvalue condition for sparse vectors is provided. Here, we generalize that definition to block-sparse vectors. 
\end{remark}

The following theorem is the main result in this section. 

\begin{theorem}\label{Theorem:RE}
	Assume that the matrix $\BMXi$ defined in \eqref{Teoplitz} satisfies the block RIP\footnote{A design matrix that satisfies RIP for general $sd$-sparse vectors certainly obeys RIP for $s$-block-sparse vectors with block size $d$.} with restricted isometry constant $\delta_{2s}$ with probability $1-p$ (for some $p>0$), i.e., 
	\begin{align*}
		(1-\delta_{2s})\|\va\|_2^2 \le \|\BMXi \va\|_2^2 \le (1+\delta_{2s})\|\va\|_2^2,
	\end{align*}
	for any $2s$-block sparse vector $\va\in \R^{hd}$. Suppose $m>\frac{s(\sqrt{d}+\sqrt{\log h}+t)^2}{1-\delta_{2s}(1+9\lambda_{\max}^2/\lambda_{\min}^2)}$, where $\lambda^2_{\max}$ and $\lambda^2_{\min}$ are the largest and smallest eigenvalue of $\Sigma$, respectively. Then, with probability $1-p-2e^{-ct^2}$ the matrix $\BMPsi$ defined in \eqref{matrix:RE} satisfies the restricted eigenvalue condition over the block sparse support set $\CS$, i.e., there is $\kappa>0$ such that
	\begin{align}\label{RE:Psi}
		&\|\BMPsi \Delta\|_2^2 \ge \kappa \|\Delta\|_2^2, &\text{for all } \Delta \in \BC_{3}^{(d)} (\CS). 
	\end{align}
\end{theorem}

\begin{proof} Let $\BMPsi'=\BMPsi-\onebb\bmu^\top$ be the zero-mean component. Then, we construct the proof in the following two steps: 
	\begin{enumerate}
		\item First, we prove that there exists $\kappa'>0$ such that $\|\BMPsi' \Delta\|_2^2 \ge \kappa' \|\Delta\|_2^2$ for all $\Delta \in \BC_{3}^{(d)} (\CS)$; 
		\item Then, we finalize the proof by taking into account the non-zero mean part $\onebb\bmu^\top$ and showing \eqref{RE:Psi}. 
	\end{enumerate}
	
	\textbf{Step 1:} We first introduce some notations. Denote $\CH:=\{1,2,\dots,h\}$. Let $\CS_0=\CS$ be the set of block indices given by the $s$ blocks of $\va$ that have the largest $\ell_2$ norm, $\CS_1$ be the set of block indices associated with the $s$ blocks of $\Delta_{\CH\backslash \CS_0}$ that have the largest $\ell_2$ norm.  The other sets $\CS_2,\dots,\CS_{k}$ are defined in the same fashion ($\CS_k$ may contain less than $s$ blocks).  Rewrite $\BMXi=[\BMXi_1,\BMXi_2,\dots,\BMXi_h]$ with $\BMXi_i\in \R^{m\times d}$.  We define $\BMXi_{\CS_i}$, $i=0,1,\dots,k$, in a similar way. 
	
	Observe that 
	\begin{equation*}
		\|\BMPsi' \Delta\|_2^2 = \BMXi \begin{bmatrix} \sqrt{\Sigma} & & & \\
			& \sqrt{\Sigma} & & \\
			& & \ddots & \\
			& & & \sqrt{\Sigma}
		\end{bmatrix} \Delta := \BMXi \tilde{\Delta}.
	\end{equation*}
	Since $\lambda^2_{\min}$ and $\lambda^2_{\max}$ are the smallest and largest eigenvalues of $\Sigma$, respectively, it holds that 
	\begin{equation*}
		\lambda_{\min}\| \Delta_{\CS^c} \|_{2,1}^{(d)}\le \|\tilde \Delta_{\CS^c}\|_{2,1}^{(d)} \le \lambda_{\max}\| \Delta_{\CS^c}\|_{2,1}^{(d)},
	\end{equation*}
	and 
	\begin{equation*}
		\lambda_{\min}\| \Delta_{\CS} \|^{(d)}_{2,1}\le \|\tilde \Delta_{\CS}\|^{(d)}_{2,1} \le \lambda_{\max}\|\tilde \Delta_{\CS}\|^{(d)}_{2,1},
	\end{equation*}
	Then, one can derive that 
	\begin{equation}\label{new:alpha}
		\tilde{\Delta} \in \{\va \in \R^{hd}: \|\va_{\CS^c}\|_{2,1}^{(d)} \le \alpha \|\va_\CS\|_{2,1}^{(d)} \}:= \mathbb{C}^{(d)}_{ \alpha} (\CS). 
	\end{equation}
	where $ \alpha={3\lambda_{\max}}/{\lambda_{\min}}$. Now, it remains to show $\|\BMXi \tilde{\Delta}\|_2^2 \ge \kappa' \|\tilde{\Delta}\|_2^2$ for all $\tilde{\Delta} \in \mathbb{C}_{ \alpha}^{(d)} (\CS)$. 	
	
	Next, one can derive that 
	\begin{align}
		\|\BMXi \tilde{\Delta}\|_2^2  =& \AngBra{ \BMXi \tilde{\Delta}_{},\BMXi \tilde{\Delta} }= \AngBra{\BMXi \tilde{\Delta}_{\CS_0\cup \CS_1},\BMXi \tilde{\Delta}_{\CS_0\cup \CS_1}} + 2 \AngBra{ \BMXi \tilde{\Delta}_{(\CS_0\cup \CS_1)^c},\BMXi \tilde{\Delta}_{\CS_0\cup \CS_1}} + \AngBra{ \BMXi \tilde{\Delta}_{(\CS_0\cup \CS_1)^c},\BMXi \tilde{\Delta}_{(\CS_0\cup \CS_1)^c}}\nonumber\\
		= &\AngBra{\BMXi \tilde{\Delta}_{\CS_0\cup \CS_1},\BMXi \tilde{\Delta}_{\CS_0\cup \CS_1}} + 2 \AngBra{ \sum_{i=2}^{k} \BMXi \tilde{\Delta}_{\CS_i},\BMXi \tilde{\Delta}_{\CS_0\cup \CS_1}} + \AngBra{\sum_{i=2}^{k} \BMXi \tilde{\Delta}_{\CS_i},\sum_{i=2}^{k} \BMXi \tilde{\Delta}_{\CS_i}} \nonumber\\
		= &\AngBra{\BMXi \tilde{\Delta}_{\CS_0\cup \CS_1},\BMXi \tilde{\Delta}_{\CS_0\cup \CS_1}} + \sum_{j=2}^{k} \AngBra{\BMXi \tilde{\Delta}_{\CS_i},\BMXi \tilde{\Delta}_{\CS_i}} + 2 \AngBra{ \sum_{i=2}^{k} \BMXi \tilde{\Delta}_{\CS_i},\BMXi \tilde{\Delta}_{\CS_0\cup \CS_1}} + \sum_{i\neq j} \tilde{\Delta}_{\CS_i}^\top \BMXi ^\top \BMXi \tilde{\Delta}_{\CS_j} \label{RE:step0}
	\end{align}
	Since $\BMXi$ is $(2s,\delta_{2s})$-RIP, it holds that 
	\begin{align}
		&\AngBra{\BMXi \tilde{\Delta}_{\CS_0\cup \CS_1},\BMXi \tilde{\Delta}_{\CS_0\cup \CS_1}} \ge (1-\delta_{2s}) \|\tilde{\Delta}_{\CS_0\cup \CS_1}\|_2^2, \label{RE:step1}\\
		&\AngBra{\BMXi \tilde{\Delta}_{\CS_i},\BMXi \tilde{\Delta}_{\CS_i}} \ge (1-\delta_{2s}) \|\tilde{\Delta}_{\CS_i} \|_2^2. \label{RE:step2}
	\end{align}
	Furthermore, $\BMXi$ being $(2s,\delta_{2s})$-RIP implies that $\|\BMXi_{\mathcal{Q}_1}^\top\BMXi_{\mathcal{Q}_2} \|\le \delta_{2s}$ for any non-overlapping sets $\mathcal{Q}_1$ and $\mathcal{Q}_2$. Therefore, it follows that
	\begin{equation}\label{RE:step3}
		|\AngBra{\BMXi \tilde{\Delta}_{\CS_i},\BMXi \tilde{\Delta}_{\CS_0\cup \CS_1}}| = | \AngBra{\BMXi_{\CS_i} \tilde{\Delta}_{\CS_i},\BMXi_{\CS_0\cup \CS_1} \tilde{\Delta}_{\CS_0\cup \CS_1}} | \le \delta_{2s} \|\tilde{\Delta}_{\CS_i}\|_2 \|\tilde{\Delta}_{\CS_0\cup \CS_1}\|_2,
	\end{equation}
	and 
	\begin{equation}\label{RE:step4}
		|\tilde{\Delta}_{\CS_i}^\top \BMXi ^\top \BMXi \tilde{\Delta}_{\CS_j}|=|\tilde{\Delta}_{\CS_i}^\top \BMXi_{\CS_i} ^\top \BMXi_{\CS_j} \tilde{\Delta}_{\CS_j}|\le \delta_{2s} \|\tilde{\Delta}_{\CS_i}\|_2 \|\tilde{\Delta}_{\CS_j}\|_2.
	\end{equation}
	Substituting \eqref{RE:step1}-\eqref{RE:step4} into \eqref{RE:step0} yields 
	\begin{align}
		\|\BMXi \tilde{\Delta}\|_2^2  \ge & (1-\delta_{2s}) \|\tilde{\Delta}_{\CS_0\cup \CS_1}\|_2^2 + (1-\delta_{2s}) \|\tilde{\Delta}_{\CS_i} \|_2^2 \nonumber \\
		&- 2 \delta_{2s} \big(\sum_{i=2}^{k} \|\tilde{\Delta}_{\CS_i}\|_2\big) \|\tilde{\Delta}_{\CS_0\cup \CS_1}\|_2 - \delta_{2s} \big(\sum_{i=2}^{k} \|\tilde{\Delta}_{\CS_i}\|_2 \big) \big(\sum_{i=2}^{k} \|\tilde{\Delta}_{\CS_i}\|_2 \big) \nonumber\\
		\ge & (1-\delta_{2s}) \|\tilde{\Delta}\|_2^2 \nonumber \\
		&- 2 \delta_{2s} \big(\sum_{i=2}^{k} \|\tilde{\Delta}_{\CS_i}\|_2\big) \|\tilde{\Delta}_{\CS_0\cup \CS_1}\|_2 - \delta_{2s} \big(\sum_{i=2}^{k} \|\tilde{\Delta}_{\CS_i}\|_2 \big) \big(\sum_{i=2}^{k} \|\tilde{\Delta}_{\CS_i}\|_2 \big). \label{RE:step5}
	\end{align}
	Note that for $i\ge 2$, we have
	\begin{equation*}
		\|\tilde{\Delta}_{\CS_i}\|_2 \le \sqrt{s} \|\tilde{\Delta}_{\CS_i}\|_{2,\infty}^{(d)} \le \frac{1}{\sqrt{s}} \|\tilde{\Delta}_{\CS_{i-1}}\|_{2,1}^{(d)},
	\end{equation*}
	and subsequently, one can derive
	\begin{align}
		\sum_{i=2}^{k} \|\tilde{\Delta}_{\CS_i}\|_2 &\le \frac{1}{\sqrt{s}} \sum_{j=1}^{k-1} \|\tilde{\Delta}_{\CS_{j}}\|^{(d)}_{2,1} \le \frac{1}{\sqrt{s}} \| \tilde{\Delta}_{\CS_0^C} \|_{2,1}^{(d)} \nonumber\\ &\stackrel{(a)}{\le}  \frac{1}{\sqrt{s}} \alpha \| \tilde{\Delta}_{\CS_0} \|_{2,1}^{(d)}  \nonumber\\
		&\le    \alpha \| \tilde{\Delta}_{\CS_0} \|_2, \label{RE:step6}
	\end{align}
	where the inequality (a) follows from \eqref{new:alpha}. Substituting \eqref{RE:step6} into \eqref{RE:step5} one has 
	\begin{align*}
		\|\BMXi \tilde{\Delta}\|_2^2  \ge & (1-\delta_{2s}) \|\tilde{\Delta}\|_2^2 - 2 \delta_{2s} \alpha \| \tilde{\Delta}_{\CS_0} \|_2 \|\tilde{\Delta}_{\CS_0\cup \CS_1}\|_2 - \delta_{2s} \alpha^2  \| \tilde{\Delta}_{\CS_0} \|_2^2\\		
		\stackrel{(a)}{\ge} & \big(1-\delta_{2s}(1+\alpha)^2 \big)  \|\tilde{\Delta}\|_2^2,
	\end{align*}
	where the inequality (a) has used the facts that $\| \tilde{\Delta}_{\CS_0} \|_2\le \|\tilde{\Delta}\|_2$ and $\|\tilde{\Delta}_{\CS_0\cup \CS_1}\|_2 \le \|\tilde{\Delta}\|_2$. Letting $\kappa':=1-\delta_{2s}(1+\alpha)^2 $ completes Step 1. 
	
	\vspace{20pt}
	
	\textbf{Step 2:} First, we have the following straightforward lemma. 
	\begin{lemma} \label{Re:inter:lemma}
		With probability at least $1-2e^{-ct^2}$, for all $\Delta \in \BC_{3}^{(d)} (\CS)$, there exists $C>0$ such that
		\[
		|\onebb^\top \Psi'\ab|\leq C\sqrt{s}(\sqrt{d}+\sqrt{\log h}+t).
		\]
	\end{lemma}
	
	\begin{proof} Let $\z^\top=\onebb^\top \Psi'$. Observe that $i$'th block of $\z_i\in\R^d$ is simply the sum of $m$ subgaussian vectors in $\R^d$. Through standard concentration arguments (see \citet{vershynin2018high}), this implies that, with probability $1-2e^{-ct^2}$, 
		\[
		\|\z_i\|_2\leq c(\sqrt{d}+t).
		\]
		Since there are $h$ such blocks, union bounding, we obtain, with the same probability,
		\[
		\|\z\|_{2,\infty} ^{(d)} =\sup_{i\in [h]}\tn{\z_i}\leq c(\sqrt{d}+\sqrt{\log h}+t).
		\]
		Now, observe that, for any $\Delta\in \BC_{3}^{(d)} (\CS)$ with unit $\ell_2$ norm, we have 
		\begin{align*}
			|\z^\top \Delta| \le \|\Delta\|_{2,1}^{(d)} \|\z\|_{2,\infty}^{(d)}.
		\end{align*}
		Since $\Delta\in \BC_{3}^{(d)} (\CS)$, one can derive that $\|\Delta\|_{2,1}^{(d)}\le (1+3)\|\Delta_\CS\|_{2,1}^{(d)}\le 4 \sqrt{s}$. Then, it follows that  
		\begin{equation*}
			|\z^\top \Delta| \le 4 c \sqrt{s}(\sqrt{d}+\sqrt{\log h}+t),
		\end{equation*}
		which establishes the desired result.
	\end{proof}	
	
	Recall that $m\gtrsim m_0=\frac{4s(\sqrt{d}+\sqrt{\log h}+t)^2}{\mu}$ with $\mu=1-\delta_{2s}(1+9\lambda_{\max}^2/\lambda_{\min}^2)$. To proceed, we claim that correlation coefficient $\rho$ between $\onebb$ and $\Psi'\ab$ is bounded above by $0.5$ for all choices of $\ab$. To see this,
	\begin{align*}
		\rho(\onebb,\Psi'\ab)&=\frac{\onebb^\top \Psi'\ab}{\sqrt{m}\tn{\Psi'\ab}}\\
		&\stackrel{(a)}{\le} \frac{C\sqrt{s}(\sqrt{d}+\sqrt{\log h}+t)}{\sqrt{m}\sqrt{\mu}}\\
		&\leq 0.5,
	\end{align*}
	where the inequality (a) has used the results in Step 1 and Lemma~\ref{Re:inter:lemma}. 
	
	With this, we can conclude the proof as follows. First note that, if $\ab,\bb$ are two vectors with correlation at most $0.5$, we have $\tn{\ab+\bb}\geq c(\tn{\ab}+\tn{\bb})$ for some constant $c>0$. With this in mind, we write
	\begin{align*}
		\tn{\Psi\ab}^2&=\tn{\Psi'\ab+\onebb \bmu^\top \ab}^2\\
		&\geq c(\tn{\Psi'\ab}^2+\tn{\onebb \bmu^\top \ab}^2)\\
		&\geq c\tn{\Psi'\ab}^2
	\end{align*}
	Substituting the inequality we established in Step 1 completes the proof. 
\end{proof}

\section{Supplementary Material for Section~\ref{main_results}}\label{analy:poor}
We first introduce some useful definitions and results for our analysis. 
Consider the block-sparse linear regression model:
\begin{align*}
	\y= \BMPsi \va+\bEps,
\end{align*}
where $\va\in \R^{hd}$ and $\|\va\|_{0}^{(d)} = s \ll d$, $\BMPsi\in \R^{m\times  hd}$, and the noise $\varepsilon_i$ in $\bEps=[\varepsilon_1,\dots,\varepsilon_m]$ is independent $1$-sub-Gaussian. Let $\CS$ be the block support set of $\va$.  Define the Lasso program:
\begin{align}\label{Append:Lasso}
	\hat \va =\argmin _{\va\in \R^{d}}\Big(\frac{1}{2m} \|\y- \BMPsi \va\|_2^2+\lambda_m \|\va\|_{2,1}^{(d)}\Big)
\end{align}
with the regularization parameter $\lambda_m$.

The following theorem provides a bound on the $\ell_2$-error between the solution $\hat \va$ and the true $\va$. 

\begin{theorem}\label{bound:Lasso}
	Assume that the matrix $\Xi$ satisfies the restricted eigenvalue condition over the support set $\CS$ with parameters $(\kappa,3)$. Then, any solution of the Lasso program \eqref{Append:Lasso} with regularization parameter lower bounded $\lambda_m \ge 2 \|\frac{\BMPsi ^\top \bEps}{m}\|_{2,\infty}^{(d)}$ satisfies 
	\begin{align*}
		& \|\hat \va -\va\|_2 \le \frac{3}{\kappa} \sqrt{s} \lambda_m, &\|\hat \va -\va\|_1 \le 4 \sqrt{s} \|\hat \va -\va\|_2.
	\end{align*}
\end{theorem}

\begin{remark}
	A similar result in found in Theorem 7.13 in \cite{wainwright2019high}, where a bound for sparse vector recovery is provided. In our case, we aim to reconstruct a block-sparse vector. The only differences between our theorem and that one in \cite{wainwright2019high} are: 1) in our case, $\BMPsi$ is required to satisfy RE for block-sparse vectors, and 2) the regularization parameter is required to be bounded by a $\ell_{2,\infty}$ norm, instead of simply $\ell_{\infty}$ norm, of $\frac{\BMPsi ^\top \bEps}{n}$. Our proof presented as follows is a generalization of Theorem. 7.13 in \cite{wainwright2019high}. 
\end{remark}

\begin{proof}
	Denote $\Delta=\hat \va -\va$. The first step is to show that the error vector satisfies $\Delta\in \BC_{3}^{(d)}(\CS)$ under the condition $\lambda_m\ge 2 \|\frac{\BMPsi ^\top \bEps }{m}\|_{2,\infty}^{(d)}$. Towards this end, define $L(\va,\lambda_m)=\frac{1}{2m}\|\y- \BMPsi \va\|_2^2+\lambda_m \|\va\|_{2,1}^{(d)}$. Since $\hat \va$ satisfies \eqref{Append:Lasso}, one has
	\begin{equation*}
		L(\hat \va,\lambda_m) \le L( \va,\lambda_m) = \frac{1}{2m} \|\bEps\|_2^2 + \lambda_m \|\va\|_{2,1}^{(d)}.
	\end{equation*}
	As $L(\hat \va,\lambda_m)=\frac{1}{2m}\|\y- \BMPsi \hat \va\|_2^2+\lambda_m \|\hat \va\|_{2,1}^{(d)}$, it then can be derived that 
	\begin{equation}\label{bound:step1}
		\frac{1}{2m} \|\BMPsi \Delta\|_2^2 \le \frac{\bEps^\top \BMPsi \Delta}{m} + \lambda_m\Big(  \|\va\|_{2,1}^{(d)} -  \|\hat \va\|_{2,1}^{(d)} \Big).
	\end{equation}
	Since $\va$ is $s$-block sparse, it holds that $\|\va\|_{2,1}^{(d)} =\|\va_\CS\|_{2,1}^{(d)} $, and $\|\hat \va\|_{2,1}^{(d)}= \|\- \hat \va -\va +\va\|_{2,1}^{(d)} = \|\Delta_\CS+\Delta_{\CS^c} +\va_\CS\|_{2,1}^{(d)}=\|\Delta_\CS +\va_\CS \|_{2,1}^{(d)}+\|\Delta_{\CS^c}\|_{2,1}^{(d)}$. Subsequently, we arrive at   
	\begin{align*}
		\|\va\|_{2,1}^{(d)} - \|\hat \va\|_{2,1}^{(d)}  = \|\va_{\CS}\|_{2,1}^{(d)} - \|\Delta_\CS +\va_\CS \|_{2,1}^{(d)}-\|\Delta_{\CS^c}\|_{2,1}^{(d)}.
	\end{align*}
	Then, it follows from \eqref{bound:step1} that
	\begin{align*}
		\frac{1}{m} \|\BMPsi \Delta\|_2^2 &\le 2 \frac{\varepsilon^\top \BMPsi \Delta}{m} + 2 \lambda_m \Big( \|\va_{\CS}\|_{2,1}^{(d)} - \|\Delta_\CS +\va_\CS \|_{2,1}^{(d)}-\|\Delta_{\CS^c}\|_{2,1}^{(d)} \Big)\\
		& \le 2\|\BMPsi ^\top \varepsilon/m\|_{2,\infty}^{(d)} \|\Delta\|_{2,1}^{(d)} +2 \lambda_m \Big(  \|\Delta_\CS\|_{2,1}^{(d)}-\|\Delta_{\CS^c}\|_{2,1}^{(d)} \Big).
	\end{align*}
	Under the condition $\lambda_m\ge 2 \|\frac{\BMPsi ^\top \bEps}{m}\|_{2,\infty}^{(d)}$, we arrive at 
	\begin{align*}
		\frac{1}{m} \|\BMPsi \Delta\|_2^2 \le \lambda_m \Big( 3\|\Delta_\CS\|_{2,1}^{(d)} -  \|\Delta_{\CS^c}\|_{2,1}^{(d)} \Big),
	\end{align*}
	which implies that $\Delta \in \BC_{3}^{(d)}(\CS)$. Applying the restricted eigenvalue condition $\frac{1}{m} \|\BMPsi \Delta\|_2^2 \ge \kappa \|\Delta\|_2^2$ we have 
	\begin{align*}
		\kappa \|\Delta\|_2^2 \le 	\frac{1}{m} \|\BMPsi \Delta\|_2^2 \le 3 \lambda_m \|\Delta_\CS\|_{2,1}^{(d)}\le 3 \lambda_m \sqrt{s} \|\Delta_\CS\|_2 \le 3 \lambda_m \sqrt{s} \|\Delta\|_2.
	\end{align*}
	Using the definition of $\BC^{(d)}_{3}(\CS)$, the inequality $\|\hat \va -\va\|_1 \le 4 \sqrt{s} \|\hat \va -\va\|_2$ follows straightforwardly. 
\end{proof}

\begin{theorem}[\cite{wedin1972perturbation}]\label{Wedin}
	Consider two matrices $M$ and $\hat M \in\R^{m \times n}$.  Consider the following singular value decomposition (SVD) of them
	\begin{align*}
		&\bM=\begin{bmatrix} 
			\bU_1& \bU_2
		\end{bmatrix}  \begin{bmatrix} 
			\Sigma_1& 0\\ 0 & \Sigma_2 
		\end{bmatrix} \begin{bmatrix} 
			\bV_1^\top \\  \bV_2^\top 
		\end{bmatrix}, & \hat \bM=\begin{bmatrix} 
			\hat \bU_1& \hat \bU_2
		\end{bmatrix}  \begin{bmatrix} 
			\hat \Sigma_1& 0\\ 0 & \hat \Sigma_2 
		\end{bmatrix} \begin{bmatrix} 
			\hat \bV_1^\top \\  \hat \bV_2^\top 
		\end{bmatrix},
	\end{align*}
	where $\Sigma_1$ and $\hat \Sigma_1 \in \R^{r \times r}$. Let  $\sigma_1,\dots,\sigma_k$ and $\hat \sigma_1,\dots, \hat \sigma_k$, with $k=\min \{m,n\}$, be the singular values of $M$ and $\hat M$, respectively. 
	If $\min_{1\le i \le r}\sigma_i\ge \mu$ and $\min_{1\le i \le r, r+1 \le j \le k}|\sigma_i-\hat \sigma_j|\ge \mu$, then
	\begin{align*}
		\|\sin \bm{\theta} (\bU_1,\hat \bU_1)\|_F^2 + \|\sin \bm{\theta} (\bV_1,\hat \bV_1)\|_F^2 \le \frac{\|\bU_1^\top \Delta\|_F^2 +\|\Delta \bV_1\|_F^2}{\mu^2},
	\end{align*}
	where $\Delta = \bM-\hat \bM$.
\end{theorem}

\subsection{Definition of Pseudo-regret}\label{def:regret}

To show why the pseudo-regret can be defined as 
\begin{align}\label{regret}
	R_T = \left [\sum_{t=1}^{T}\sum_{i=0}^{h-1}{w_i} \AngBra{\x_{t,a^*_t}-\x_{t,a_t},\theta}  \right],
\end{align}
we provide  an alternative interpretation for our problem setting. 

Let $\bar r(x_{t,a_t}):=\sum_{i=0}^{h-1}w_{i}\x_{t,a_{t}}^\top \theta $, and one can see that $\bar r(\x_{t,a_t})$ is the total reward generated by the action $a_t$. However, the agent does not receive the entire reward immediately after taking this action at round $t$. Instead, the reward spreads over at most $h$ rounds. At round $k \in \{t,t+1,\dots,t+h-1\}$, the portion of the reward that the agent receives from the action $a_t$ is $w_{k-t}\x_{t,a_t}^\top \theta$. Consequently, one can see that the total reward over the course of $T$ rounds, $\sum_{t=1}^Tr_t$, can be calculated by 
\begin{align}
	\sum_{t=1}^T r_t= \sum_{t=1}^T\bar r(\x_{t,a_t})-r,
\end{align}
where $r$ is the reward contributed by the actions $a_{T-h+1},a_{T-h+2},\dots,a_{T}$ that should be received by the agent after the horizon $T$. Since $r$ is bounded and independent of $T$, to maximize $\sum_{t=1}^T r_t$, it is equivalent to choose an action $a_t\in \CA$ that maximizes $\bar r(\x_{t,a_t})$ at each round $t$. Since $\sum_{i=0}^{h-1}w_{i}$ is a constant, the optimal action at each round is simply $a_t=\arg \max_{a \in \CA} \x_{t,a}^\top \theta $. Therefore, the regret can be defined as \eqref{regret}.

\subsection{Proof of Theorem~\ref{Th:data_poor}}

\begin{theorem}\label{theta_hat}
	In Algorithm~\ref{alg:initial}, let $L = c  sd \log^2 (sd) \log^2 (h d)$. Recall that $\BMphi_{\CS_i}$ is a vector consisting of the first $2^{i-1}L$ entries of $\BMphi$ and $\hat \BMtheta_i$ is the estimate of $\BMtheta_i$ at each doubling cycle $i$. Then, it holds, with probabiliy at least $1-\gamma$, that
	\begin{align}\label{accu:theta}
		&\left|\sin \bm{\Theta} \Big(\hat \BMtheta_i,\frac{\BMtheta}{\|\BMtheta\|} \Big) \right| \le 
		\begin{cases}
			\frac{6 \kappa d \sqrt{2s}}{ \|\w_{\CS_i}\|_2} \left( \sqrt{\frac{2 \log(2^i L/\gamma)}{2^{i-1} L}}\right), & \text{if }  \|\w_{\CS_i}\|_2^2 \neq 0,\\
			\frac{\pi}{2} & \text{if }  \|\w_{\CS_i}\|_2^2 = 0.
		\end{cases}    
	\end{align} 
\end{theorem}

\begin{remark}
	Note that the constant $c$ in this theorem has at most a polylogarithmic dependence on $K$; specifically, $c\lesssim \log^2(K)$, where $K$ is the number of arms. This is because, according to Theorem~~\ref{Th:HW:chosen}, the context vectors chosen by our algorithm satisfy the Hanson-Wright inequality, but the constant $C$ in \eqref{equi_bound:1} decreases at most $c' \log^2(K)$ times. 
\end{remark}

\begin{proof}
	We construct the proof in two steps:
	
	\begin{enumerate}
		\item First, we use Theorem~\ref{bound:Lasso} to show that
		\begin{align}\label{err:phi_partial}
			\|\hat \BMphi_{\CS_i} -\BMphi_{\CS_i}\|_2 \lesssim  6 \kappa d \sqrt{s} \left( \sqrt{\frac{2 \log(2^i L/\gamma)}{2^{i-1} L}}\right).
		\end{align}
		\item Then, we apply Theorem~\ref{Wedin} to \eqref{err:phi_partial} and show \eqref{accu:theta}.  
	\end{enumerate}
	
	\textbf{Step 1:} Under Assumption~\ref{ass:cv_concentr}, since $L \gtrsim sd\log^2(sd) \log^2(hd)$, each element of $P_i'$ and $P_i''$ given in \eqref{matrices} satisfies the Hanson-Wright inequality. Also, our choice of $P_i'$ and $P_i''$ (separated by $2^{i-1}L$ rows) ensures that   $P_i''$ is an independent copy of $P_i'$. Then, one can use Theorems~\ref{diff_of_HW} and~\ref{Theorem:RE} to show  that the matrix $\bar \bP_i=P_i''-P_i$ satisfies the restricted eigenvalue condition for block-sparse vectors with parameters $(\kappa,3)$. Then, it remains to show $\lambda_i \ge 2 \|\frac{1}{2^{i-1} L }\bar \bP_i^\top \bNeps[i] \|_{2,\infty}$ with $\bNeps[i]=\bar \bQ_i \BMphi_{\bar \CS_i}+\bEps[i]$ defined in \eqref{parti_measure}. To do that, we will prove 
	\begin{equation}\label{bound:regularization}
		\|\frac{1}{2^{i-1} L }\bar \bP_i^\top \bNeps[i] \|_{2,\infty} \lesssim d\sqrt{\frac{2 \log(2^i d L/\gamma)}{2^{i-1} L}}.
	\end{equation}
	
	Observe that
	\begin{align}
		\left \|\frac{1}{2^{i-1}L } \bar\bP_i^\top \bNeps[i] \right \|_{2,\infty}^{(d)} \le \frac{1}{2^{i-1} L }\Big(\|\bar \bP_i^\top \bar \bQ_i  \BMphi_{\bar\CS_i}\|_{2,\infty}^{(d)}+ \| \bar \bP_i^\top \bEps[i]  \|_{2,\infty}^{(d)} \Big).\label{noise:bound:overall}
	\end{align}
	For concise notation, let $m_i:= 2^{i-1}L$. Recall that $\bar \bP_i^\top \in \R^{m_id \times m_i}$, denote the $k$th block row\footnote{Given a matrix $A\in \R^{md\times n}$, one can partition its rows into $m$ blocks with equal size $d$, and the $j$th block row, $j=1,2,\dots,m$, is the matrix formed by the $j$th block.} of $\bar \bP_i^\top$ as $\bar \bP_i^{[k]}=[\x_1^{[k]},\x_2^{[k]},\dots,\x_{m_i}^{[k]}]$ with each $\x_{j}^{[k]}\in \R^{d}$ (since each entry in $\bar \bP_i$ is the difference between two selected context vectors, one can derive that $\|\x_{j}^{[k]}\|_\infty \le 2$).  It holds that 
	\begin{align}
		\|\bar \bP_i^\top \bEps[i]  \|_{2,\infty}^{(d)} &= \max_k \Big\{ \|\bar \bP_i^{[k]} \bEps[i]\|_2  \Big\} = \max_k \Big\{ \big\|\sum_{j=1}^{m_i} \x_{j}^{[k]} \bEps_{j}[i] \big\|_2\Big\} \nonumber\\
		&\le \max_k \Big\{\sqrt{d}  \big\|\sum_{j=1}^{m_i} \x_{j}^{[k]} \bEps_{j}[i] \big\|_{\infty}\Big \} \nonumber\\
		&=  \max_k \Big\{\sqrt{d} \max_{\ell\in[d]} \big\{ \big|\sum_{j=1}^{m_i} [\x_{j}^{[k]}]_\ell \cdot \bEps_{j}[i] \big|  \big \} \Big \}. \label{bound:noise:1}
	\end{align}
	
	Since $\bEps_{1}[i],\bEps_{2}[i],\dots,\bEps_{m_i}[i]$ are independent sub-Gaussian random variables, using the fact that $\big| [\x_{j}^{[k]}]_\ell \big|<2$ for any $j,K$ and $\ell$, one can derive that
	\begin{align*}
		\Pr \Big(   \big|\sum_{j=1}^{m_i} [\x_{j}^{[k]}]_\ell \cdot \bEps_{j}[i] \big|   \ge \eta \Big) \le 2 \exp \big( - \frac{\eta^2}{8 m_i} \big).
	\end{align*}
	Let $\eta=2\sqrt{2 m_i \log(2d m_i/\gamma)}$ with $0<\gamma<1$, and one can derive that with probability at least $1-\gamma$, it holds that
	\begin{align*}
		\max_k \max_{\ell\in[d]} \big\{ \big|\sum_{j=1}^{m_i} [\x_{j}^{[k]}]_\ell \cdot \bEps_{j}[i] \big|  \big \} \le 2 \sqrt{2 m_i \log(2d m_i/\gamma)}.
	\end{align*}
	Substituting this inequality into \eqref{bound:noise:1} yields that, with probability at least $1-\gamma$,
	\begin{align}
		\|\bar \bP_i^\top \varepsilon[i]  \|_{2,\infty}^{(d)}  \le 2 \sqrt{2^i d L \log(2^id L/\gamma)}. \label{bound:true:noise}
	\end{align}
	
	Further, denote the $j$th row of $\bar \bQ_i \phi_{\bar \CS_i}$ as $\nu_j$, and it follows that
	\begin{align*}
		\|\bar \bP_i^\top \bar \bQ_i \phi_{\bar \CS_i}\|_{2,\infty}^{(d)}&= \max_k \Big\{ \Big\| \sum_{j=1}^{m_i} \x_{j}^{[k]} \nu_j  \Big\|_2\Big\}.
	\end{align*}
	Since $\|\x_{j}^{[k]}\|_\infty \le 2$ for all $j$ and $k$,  the random vectors $\x_{1}^{[k]},\dots,\x_{m_i}^{[k]}$ are norm-subGaussian satisfying $\Pr (\|x^{[k]}_j\|_2\ge t ]\le 2 e^{-\frac{t^2}{4d}}$. It follows from Corollary 7 \cite{jin2019short} that there exists an absolute constant $c>0$ so that
	\begin{align}\label{fake_noise:bound:1}
		\Big\| \sum_{j=1}^{m_i} \x_{j}^{[k]} \nu_j  \Big\|_2 \le c\sqrt{d \sum_{j=1}^{m_i}\nu_j^2 \log(4d/\gamma)}
	\end{align}
	holds with probability at least $1-\gamma$. 
	Since $\BMphi=\w\otimes \BMtheta$ and $\w$ is $s$-sparse, $\BMphi$ is at most $s$-block-sparse. Then, $\BMphi_{\bar\CS_i}$ is at most $s$-block-sparse. Denote $\BMphi_{\bar \CS_i}:=[\psi_1^\top,\psi_2^\top,\dots,\psi_r ^\top ]^\top$ where $\psi_k \in \R^d$ and $r$ is the appropriate dimension. Then, each row of $\bar \bQ_i \BMphi_{\bar \CS_i}$ has the form
	$
	\sum _{\ell =1}^r \BMxi_\ell^\top \psi_\ell,
	$
	where $\BMxi_\ell$ is the related difference between the selected context vectors in $\bar\bP''_i$ and $\bar\bP'_i$ (see \eqref{matrices}).	Since $\|\x_{t,a}\|_\infty\le 1$, 
	one can derive that $| \BMxi_\ell^\top \psi_\ell |  \le \| \BMxi_\ell\|_\infty \|\psi_\ell\|_1  \le \| \psi_\ell\|_1$. Then, we have 
	\begin{align*}
		\left| \sum _{\ell =1}^r \BMxi_\ell^\top \psi_\ell \right| \le \sum _{\ell =1}^r \|\psi_\ell\|_1 \le \sqrt{d},
	\end{align*}
	where the second inequality has used the fact that $\sum _{\ell =1}^r \|\psi_\ell\|_1 \le \sum _{i=1}^h w_i \sqrt{d}\|\theta \|_2 \le \sqrt{d}$. Subsequently, we arrive at 
	\begin{align*}
		|\nu_j|\le \sqrt{d}, \hspace{40pt} \forall j=1,\dots,m_i. 
	\end{align*}
	Substituting this inequality into \eqref{fake_noise:bound:1} and using similar arguments in \cite{jin2019short}, one can derive that
	\begin{align}\label{bound:fake:noise}
		\|\bar \bP_i^\top \bar \bQ_i  \phi_{\bar \CS_i}\|_{2,\infty}^{(d)} \le c d \sqrt{m_i \log(2d m_i/\gamma)}\le c d \sqrt{2^i L \log(2^id L/\gamma)}.
	\end{align}
	holds with probability at least $1-\gamma$.
	
	Substituting \eqref{bound:true:noise} and \eqref{bound:fake:noise} into \eqref{noise:bound:overall} shows \eqref{bound:regularization}. Then, applying Theorem~\ref{bound:Lasso}, one can show that the solution to the Lasso program~\eqref{Lasso:initial} satisfies \eqref{err:phi_partial}.

	\vspace{12pt}
	\textbf{Step 2:}   Recall from the definition of $\hat \Phi_{\CS_i}$ (resp., $\Phi_{\CS_i}$) that its $i$th column is defined by the $((i-1)d+1)$-th to the $id$-th entries of $\hat \phi$ (resp., $\phi_{\CS_i}$).  The inequality \eqref{err:phi_partial} implies that
	\begin{align}\label{err:Phi}
		\|\hat \Phi_{\CS_i} -\Phi_{\CS_i} \|_F \le 6 \kappa d \sqrt{s} \left( \sqrt{\frac{2 \log(2^id L/\gamma)}{2^{i-1} L}}\right).
	\end{align}
	Consider the SVD for $ \Phi_{\CS_i}$ and $\hat \Phi_{\CS_i}$
	\begin{align*}
		&\Phi_{\CS_i}=\begin{bmatrix} 
			\bU_1& \bU_2
		\end{bmatrix}  \begin{bmatrix} 
			\sigma_1& 0\\ 0 & \Sigma_2 
		\end{bmatrix} \begin{bmatrix} 
			\bV_1^\top \\  \bV_2^\top 
		\end{bmatrix}, & \hat \Phi_{\CS_i}=\begin{bmatrix} 
			\hat \bU_1& \hat \bU_2
		\end{bmatrix}  \begin{bmatrix} 
			\hat \sigma_1& 0\\ 0 & \hat \Sigma_2 
		\end{bmatrix} \begin{bmatrix} 
			\hat \bV_1^\top \\  \hat \bV_2^\top 
		\end{bmatrix}.
	\end{align*}
	Since $\Phi_{\CS_i}=\BMtheta \w_{\CS_i}^\top$, it can be seen that $U_1=\frac{\theta}{\|\theta\|}$. Applying Theorem~\ref{Wedin} yields that 
	\begin{align*}
		\left\|\sin \bm{\Theta} \Big(\hat \theta,\frac{\theta}{\|\theta\|} \Big) \right\|_F^2 &\le \|\sin \bm{\theta} (\hat \bU_1, \bU_1)\|_F^2 + \|\sin \bm{\theta} (\hat \bV_1, \bV_1)\|_F^2 \\ 
		&\le \frac{\|\bU_1^\top \Delta\|_F^2 +\|\Delta \bV_1\|_F^2}{\mu^2}\\
		& \le \frac{2 \|\Delta\|_F^2 }{\mu^2},
	\end{align*}
	where $\mu= \min\{\sigma_1,  \min_{2 \le j \le \min\{d,2^{i-1}L\}}|\sigma_1-\hat \sigma_j|\}$ and $\Delta=\hat \Phi_{\CS_i}- \Phi_{\CS_i}$. From Weyl's Theorem (see Theorem~1 in \cite{wedin1972perturbation}), it holds that $|\sigma_j-\hat \sigma_j|\le \|\hat \Phi_{\CS_i} -\Phi_{\CS_i}\|_2$ for any $j$. When $2 \le j \le \min\{d,2^{i-1}L\}$, we have $|\sigma_j-\hat \sigma_j|=|\hat \sigma_j|\le \|\hat \Phi_{\CS_i} -\Phi_{\CS_i} \|_2$. Therefore, one can derive that $\mu\ge \sigma_1-\|\hat \Phi_{\CS_i} -\Phi_{\CS_i} \|_2$. Then, it follows from \eqref{err:Phi} that 
	\begin{align*}
		\left\|\sin \bm{\Theta} \Big(\hat \theta,\frac{\theta}{\|\theta\|} \Big) \right\|_F^2 \le \frac{72 \kappa^2 sd^2}{\sigma_1^2 } \left( {\frac{2 \log(2^i d L/\gamma)}{2^{i-1} L}} \right).
	\end{align*}
	Since $\Phi_{\CS_i}=\BMtheta \w_{\CS_i}^\top$, it holds that $\sigma_1=\|\BMtheta\|_2 \|\w_{\CS_i}\|_2$. If $\|\w_{\CS_i}\|_2\neq 0$, we have
	\begin{align}\label{accu:theta:1}
		&\left|\sin \bm{\Theta} \Big(\hat \BMtheta,\frac{\BMtheta}{\|\BMtheta\|} \Big) \right| \le \frac{6 d\sqrt{2} \kappa \sqrt{s}}{ \|\BMtheta\|_2 \|\w_{\CS_i}\|_2} \left( \sqrt{\frac{2 \log(2^i d L/\gamma)}{2^{i-1} L}}\right),
	\end{align}
	where the fact that $\sin \bm{\Theta} \big(\hat \theta,\frac{\theta}{\|\theta\|} \big)$ is a scalar has been used.  If $\|\w_{\CS_i}\|_2 = 0$,  $\sin \bm{\Theta} \big(\hat \theta,\frac{\theta}{\|\theta\|} \big)$ can be as large as $\frac{\pi}{2}$, which completes the proof. 
\end{proof}

Now, we are now ready to prove Theorem~\ref{Th:data_poor}. 

\begin{pfof}{Theorem~\ref{Th:data_poor}}
	The proof consists of two steps:
	\begin{enumerate}
		\item We show that at each round $t$, taking a sub-optimal action $a_t=\argmax_{a\in [K]} \AngBra{x_{t,a},\hat \BMtheta_t}$ will result in a total regret that only depends on the angle distance between $\hat \BMtheta_t$ and the true $\BMtheta$, i.e., $\sin \bm{\Theta} \big(\hat \theta,\frac{\theta}{\|\theta\|} \big)$. 
		\item We apply Theorem~\ref{theta_hat} and sum up the regret accumulated in each epoch of the doubling scheme and show how the distribution of the weights in the vector $\w$ affect the final regret. 
	\end{enumerate}
	
	\textbf{Step 1:} Since $w_i>0$ for all $i=1,2,\dots,h$, it holds that $\|\w\|_1=\sum_{i=1}^h w_i$. Recall that $a_t^*=\argmax_{a\in [K]}\AngBra{\x_{t,a},\BMtheta}$, and the total reward of the action $a_t \in [K]$ at round $t$ contributes to the cumulative reward is  $\bar r(\x_{t,a_t})=\|\w\|_1\x_{t,a_{t}}^\top \theta $. Therefore, the regret generated by taking the action $a_t \in [K]$ at round $t$ can be calculated by 
	\begin{align*}
		\bar r(\x_{t,a_t^*}) -\bar r(\x_{t,a_t}) =\|\w\|_1\x_{t,a^*_{t}}^\top \theta -\|\w\|_1\x_{t,a_{t}}^\top \theta.
	\end{align*}
	Let $\bar \theta =\frac{\theta}{\|\theta\|_2}$. Then, it follows that 
	\begin{align*}
		\bar r(\x_{t,a_t^*}) -\bar r(\x_{t,a_t}) =\|\w\|_1 \|\theta\|_2 (\x_{t,a^*_{t}}^\top \bar \theta -\x_{t,a_{t}}^\top \bar \theta).
	\end{align*}
	Observe that 
	\begin{align*}
		\x_{t,a^*_{t}}^\top \bar \BMtheta -\x_{t,a_{t}}^\top \bar \BMtheta = \x_{t,a^*_t}^\top (\bar \BMtheta-\hat \BMtheta)-\x_{t,a_t}^\top (\bar \BMtheta-\hat \BMtheta) +(\x_{t,a^*_t}-\x_{t,a_t})^\top \hat \BMtheta.
	\end{align*}
	Since $\x_{t,a_{t}}=\argmax_{a\in [K]}\AngBra{\x_{t,a},\hat \theta}$ is the estimated optimal action, it holds that  $\AngBra{\x_{t,a_t},\hat \BMtheta} \ge \AngBra{\x_{t,a_t^*},\hat \BMtheta}$, which implies
	\begin{align*}
		\x_{t,a^*_{t}}^\top \BMtheta -\x_{t,a_{t}}^\top \BMtheta \le x_{t,a^*_t}^\top (\bar \BMtheta-\hat \BMtheta)-\x_{t,a_t}^\top (\bar \BMtheta-\hat \BMtheta)=(\x_{t,a^*_t}-\x_{t,a_t})^\top  (\bar \BMtheta-\hat \BMtheta).
	\end{align*}
	By assumption, $\|\x_{t,a}\|_2\le 1$ for all $a\in [K]$, $\|\BMtheta\|_2 \le 1$,  and $\|\w\|_1\le 1$. Consequently, it follows that
	\begin{align*}
		\bar r(\x_{t,a_t^*}) -\bar r(\x_{t,a_t})\le 2 \|\bar \BMtheta-\hat \BMtheta\|_2.
	\end{align*}
	Next, we provides the bound for $\|\bar \BMtheta-\hat \BMtheta\|_2$. Let $\alpha=:\sin \bm{\Theta} (\hat \BMtheta, \bar \BMtheta )$, and one can derive that 
	\begin{align*}
		\|\bar \BMtheta-\hat \BMtheta\|_2 \le 2 \big| \sin(\frac{\arcsin \alpha}{2}) \big|\le c |\alpha|,
	\end{align*}
	It can be derived that $2 \big| \sin(\frac{\arcsin \alpha}{2}) \big|\le c |\alpha|$ for some constant $c$ when $|\alpha|$ is small, and $2 \big| \sin(\frac{\arcsin \alpha}{2}) \big| =\sqrt{2}$ when $|\alpha|=\frac{\pi}{2}$. Therefore, we have
	\begin{align}\label{regret:chuncks}
		\bar r(\x_{t,a_t^*}) -\bar r(\x_{t,a_t})\le 
		\begin{cases}
			2 c  |\sin \bm{\Theta} (\hat \BMtheta, \bar \BMtheta )|, &\text{if }  \alpha \text{ is small},\\
			2 \sqrt{2} &  \text{if }  \alpha=\frac{\pi}{2}.
		\end{cases}
	\end{align}
	
	\textbf{Step 2:} Recall that in Algorithm~\ref{alg:initial} the doubling  sequence is 
	\begin{align*}
		&T_1,T_2,\dots,T_m, &\text{where } T_i=4(2^{i}-1)L,
	\end{align*}
	and $\hat\BMtheta_i$ is estimated at each $T_i$. Recall that the regret satisfies
	\begin{align*}
		R_{T} &=\left [\sum_{t=T_1+1}^{T}\sum_{i=0}^{h-1}{w_i} \AngBra{\x_{t,a^*_t}-\x_{t,a_t},\theta}  \right] \le \sum_{t=1}^{T} \big(\bar r(\x_{t,a_t^*}) -\bar r(\x_{t,a_t}) \big).
	\end{align*}
	
	Let $T_0=0$ and denote $\Tilde{R}_k=\sum_{T_{k-1}+1}^{T_{k}} (\bar r(\x_{t,a_t^*}) -\bar r(\x_{t,a_t}) )$ for $k=1,2,\dots,m$. It can be seen that $\Tilde{R}_1 \le 4L$. For $k=2,3,\dots, m$, it follows from Theorem~\ref{theta_hat} and \eqref{regret:chuncks} that if  $\|\w_{\CS_{k-1}}\|_2 \neq 0$, 
	\begin{align*}
		\Tilde{R}_k&\le \frac{6 \kappa d \sqrt{2s}}{ \|\w_{\CS_{k-1}}\|_2} \left( \sqrt{\frac{2 \log(2^{k-1} L/\gamma)}{2^{k-2} L}} \right) (T_k-T_{k-1}) \\
		&= \frac{24 \kappa d \sqrt{2s}}{ \|\w_{\CS_{k-1}}\|_2} \left( \sqrt{{2 \log(2^{k-1} d L/\gamma)}}\right) \sqrt{2^k L},
	\end{align*}
	and $\Tilde{R}_k\le 2^{k+1}L$ if $\|\w_{\CS_{k-1}}\|_2 = 0$. Assume that $k_1$ is the first $k$ such that $\|\w_{\CS_{k}}\|_2 \ge \mu$ with $0\le \mu \le 1$. Then, the regret satisfies 
	\begin{align*}
		R_T=\sum_{k=1}^m\Tilde{R}_k&\le \sum_{k=1}^{k_1} 2^{k+1} L + \frac{24  \kappa d \sqrt{2s}}{ \mu} \sum_{k=k_1+1}^{m} \left( \sqrt{{2 \log(2^{k-1} d L/\gamma)}} \right) \sqrt{2^k L}\\
		&\le 4(2^{k_1}-1)L\cdot 2\sqrt{2} + \frac{24 \kappa  d \sqrt{2s}}{ \mu}\left( \sqrt{{2 \log(T/\gamma)}} \right) \frac{(\sqrt{2})^{m+1}-(\sqrt{2})^{k_1+1}}{\sqrt{2}-1} \sqrt{L}.
	\end{align*}
	Since $T_i=4(2^{m}-1)L=T$, one can calculate that $m=\log_2(\frac{T}{4L}+1)$, which implies that  
	\begin{align*}
		R_T \le 8\sqrt{2}(2^{k_1}-1)L + \frac{24 \kappa d \sqrt{s}}{ (\sqrt{2}-1)\mu}\left( \sqrt{{2 \log(dT/\gamma)}} \right) \sqrt{T}.
	\end{align*}
	For any $\mu\in(0,1)$ and any $k_1$ such that $4(2^{k_1}-1)L\le T$, there always exists $\alpha(\mu)\in (0,1)$ such that
	\begin{align*}
		h^{\alpha(\mu)} = 4(2^{k_1}-1)L.  
	\end{align*}
	Consequently, it holds that 
	\begin{align}\label{form:full_poor}
		R_T \le 2\sqrt{2} \min\{h^{\alpha(\mu)},T\} + \frac{24\kappa d}{(\sqrt{2}-1)\mu}\left( \sqrt{2sT \log(dT/\gamma)}\right),
	\end{align}
	which completes the proof.
\end{pfof}

\subsection{Proof of Theorem~\ref{Th:data_rich}}

\begin{theorem}\label{theta_hat:rich}
	In Algorithm~\ref{alg:rich}, let $\hat \BMphi[j]$ be the estimate of $\BMphi$ at the $j$th epoch for the phase when $t> h$ (i.e., $\hat \phi[j]$ is the solution to the Lasso program~\ref{Lasso:rich}). Let $\hat \BMtheta_j$ be the estimate of $\BMtheta$ at the same epoch. Then, it holds, with probabiliy at least $1-\gamma$, that
	\begin{align}\label{accu:theta:rich}
		\left|\sin \bm{\Theta} \Big(\hat \BMtheta_j,\frac{\BMtheta}{\|\BMtheta\|} \Big) \right| \le 6 \kappa \sqrt{2sd}  \sqrt{\frac{2 \log(2^j d h/\gamma)}{2^{j-1} h}}. 
	\end{align} 
\end{theorem}
The proof of this theorem is very similar to that of Theorem~\ref{theta_hat} and thus is omitted here. 

\begin{pfof}{Theorem~\ref{Th:data_rich}}
	Recall that in Algorithm~\ref{alg:rich} the doubling sequence is 
	\begin{align*}
		&\tilde T_1,\tilde T_2,\dots,\tilde T_p, &\text{where } \tilde T_j=2(2^{j}-1)h,
	\end{align*}
	and $\hat\BMtheta_j$ is estimated at each $\tilde T_j$. Following similar reasoning to Step 1 in the proof of Theorem~\ref{Th:data_poor}, one can show that for any $t$ in each epoch $j$ it holds that
	\begin{align*}
		\bar r(\x_{t,a_t^*}) -\bar r(\x_{t,a_t})\le 2 c  |\sin \bm{\Theta} (\hat \BMtheta_{j-1}, \bar \BMtheta )|.
	\end{align*}
	The regret can be divided into two parts: 1) the initial phase of $h$ rounds, and 2) the later phase where $t>h$. The regret in the initial phase follows from Theorem~\ref{Th:data_poor}, which is 
	\begin{align}\label{regr:h}
		R_h \le  2\sqrt{2} h^{\alpha(\mu)} + \frac{12\kappa}{(\sqrt{2}-1)\mu}\left( \sqrt{2sdh \log(dh/\gamma)}\right).
	\end{align}
	After the initial phase, from Theorem~\ref{theta_hat}, the estimated $\hat \BMtheta$ satisfies
	\begin{align*}
		\left|\sin \bm{\Theta} \Big(\hat \BMtheta,\frac{\BMtheta}{\|\BMtheta\|} \Big) \right| \le 6 \kappa \sqrt{2sd} \left( \sqrt{\frac{2 \log(2^m d L/\gamma)}{2^{m-1} L}} \right).
	\end{align*}
	Since $m$ is the largest integer such that $T_m=2(2^{m}-1)L\le h$, it holds that
	\begin{align*}
		\left|\sin \bm{\Theta} \Big(\hat \BMtheta,\frac{\BMtheta}{\|\BMtheta\|} \Big) \right| \le 12 \kappa \sqrt{sd} \left( \sqrt{\frac{2 \log(dh/\gamma)}{h}}\right).
	\end{align*}
	Therefore, the regret accumulated in the first $2h$ rounds in the second phase, denoted by $R_{h+1:3h}$, satisfies
	\begin{align}\label{regr:h+1}
		R_{h+1:3h} \le 2h \cdot 12 \kappa \sqrt{sd} \left( \sqrt{\frac{2 \log(dh/\gamma)}{h}}  \right) = 24 \kappa \sqrt{2sd h\log(dh/\gamma)}. 
	\end{align}
	Applying Theorem~\ref{Th:data_rich}, the total regret over the horizon $T$ can be bounded as
	
	Then, applying Theorem~\ref{Th:data_rich}, one can derive that
	\begin{align*}
		R_T &\le R_h +R_{h+1:3h}+ \sum_{j=2}^p 2 (\tilde T_j-\tilde T_{j-1})2 c  |\sin \bm{\Theta} (\hat \BMtheta_{j-1}, \bar \BMtheta )|\\
		& \le R_h +R_{h+1:3h} + \sum_{j=2}^p 2^{j}h\cdot 6 \kappa \sqrt{2sd}    \sqrt{\frac{2 \log(2^j d h/\gamma)}{2^{j-2} h}},\\
		& \le R_h +R_{h+1:3h} + 48  \kappa \sqrt{sd} \sum_{j=2}^p \sqrt{(2^{j-1}) h \log(2^j d h/\gamma)}, \\
		& \le R_h +R_{h+1:3h} +   48  \kappa \sqrt{sd\log(dT/\gamma)} \sum_{j=2}^p \sqrt{2^{j-1}h}.
	\end{align*}
	Since $p$ is the largest integer such that $\tilde T_j=(2^{p-1}+p)h\le T$, one can derive that $\sum_{j=2}^p \sqrt{2^{j-1}h}\le \sqrt{T}$. Combining this with \eqref{regr:h} and \eqref{regr:h+1} one has
	\begin{align*}
		R_T \le &2\sqrt{2} h^{\alpha(\mu)} + \frac{12\kappa}{(\sqrt{2}-1)\mu} \sqrt{2sdh \log(h/\gamma)} \\
		&+ 24 \kappa \sqrt{2sd h\log(d h/\gamma)} +   48  \kappa \sqrt{sd T \log(d T/\gamma)},
	\end{align*}
	which completes the proof. 
\end{pfof}

\newpage
\section{Experiments: Details and Further Results}\label{experi}

\subsection{Details of Algorithms}
In this section, we first provide some further details of the Sparse-Alternating Gradient Descent (SA-GD), Single-Weight Match Pursuit (SW-MP), and UCB with Match Pursuit (UCB-MP) presented in Section~\ref{sec:rich}. 
As we mentioned in Section~\ref{sec:rich}, all algorithms use the same doubling trick as our adaptive-Lasso to enable fair comparison. The only difference is that they learn $\BMtheta$ and make decisions in different ways.

\vspace{15pt} 

\textbf{SA-GD.} To elaborate how SA-GD works, we rewrite \eqref{CB:high-d} into
\begin{equation*}
	r_t=\BMtheta^\top  Z_{t}\w +\varepsilon_t,
\end{equation*}
where $Z_{t}=[\BMxi_{t},\BMxi_{t-1},\dots,\BMxi_{t-h+1}]$. Recall that the $i$th epoch starts at $t=\tilde T_{i-1}+1$ and ends at $\tilde T_i$. We then define the loss function in the $i$th epoch as
\begin{align*}
	f_i(\BMtheta,\w) = \sum_{t=\tilde T_{i-1}+1}^{\tilde T_j} (r_t-\BMtheta^\top  Z_{t}\w)^2.
\end{align*}
Then, $\bt$ and $\BMphi$ are estimated by 
\begin{align*}
	&\hat \BMtheta^{(k+1)}= \BMtheta^{(k)} - \beta \nabla f_{i,\bt} (\bt^{(k)},\w^{(k)}),\\
	&\hat \w^{(k+1)}= \w^{(k)} - \beta \nabla f_{i,\w} (\bt^{(k)},\w^{(k)}),\\
\end{align*}
alternatively, where $f_{i,\bt}$ and $f_{i,\w}$ are the gradients of $f_i$ with respect to $\bt$ and $\w$, respectively. We choose a threshold $\epsilon$. When the loss $f_i(\hat \bt^{(k^*)},\hat \w^{(k^*)})\le \epsilon$, the alternative gradient descent algorithm stops and we project $\w^{(k^*)}$ to the $s$-sparse subspace. 

\vspace{15pt} 

	We wish to mention that, due to our intricate problem setting, most existing algorithms will perform worse than ours since they do not learn the long-horizon reward pattern. For instance, for the special case where $w_0=0$ (where $\w=[w_0,w_1,\dots,w_h]^\top$), it is not hard to see that the classic UCB and OFUL \cite{chu2011contextual,abbasi2011improved} will fail since nothing can be learned. To enable more fair comparison, we equip the classic algorithms with the ability to estimate the location of the largest weight in $\w$, which we will elaborate below. 

\vspace{15pt} 

\textbf{SW-MP.} The single-weight matching pursuit algorithm employs the same idea as the classic matching pursuit. Instead of estimating all non-zero values in a vector, we only look for the largest non-zero, and that is why we call this algorithm single-weight matching pursuit.

In each epoch, the rewards are generated by
\begin{figure}[ht]
	\centering
	\includegraphics[scale=1]{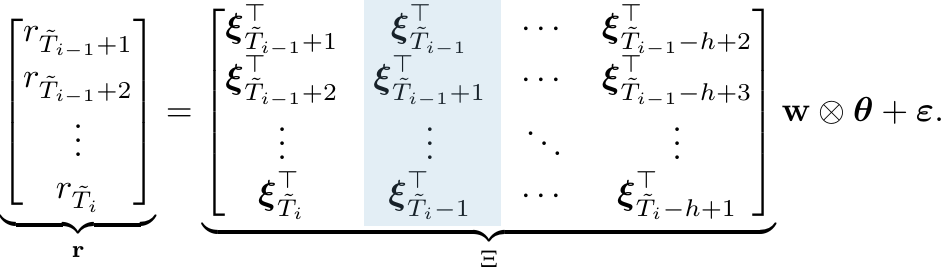}
\end{figure}

The first step of SW-MP is to estimate the location of the largest $w_i$ in the weight vector $\w$. To do that, it tests which block column (see the above  shaded area for an example of a block column) are most correlated with the reward vector $\br$. Let $\Xi^{k}$ be the $k$th block column. Then, we solve the optimization problem at the end of each epoch
\begin{align*}
	k^* = \argmax_{k\in\{1,2,\dots,h\}} \frac{1}{\|\Xi^{k}\|}\| \br^\top \Xi^{k}\|
\end{align*}
to find the index of the largest weight in $\w$. Then, we estimate $\bt$ by solving the following least squares problem
\begin{align*}
	\hat \bt = \argmin_{\bt \in \R^d} \|\br - \Xi^{k^*} \bt\|_2^2. 
\end{align*}
This estimated $\hat \bt$ will be used for decision-making in the next epoch. Note that, as expected, this algorithm can be inaccurate when there are many non-zeros in $\w$, which is also evidenced in (a3), (a4), (b3), and (b4) in Fig.~\ref{data_rich_diff_spars}. 

\vspace{15pt} 
\textbf{UCB-MP.} This algorithm shares some similarities as SW-MP. At the end of each epoch $i$, we use the same approach as SW-MP to estimate the location of the largest $w_i$ in the weight vector $\w$, i.e., 
\begin{align*}
	k^* = \argmax_{k\in\{1,2,\dots,h\}} \frac{1}{\|\Xi^{k}\|}\| \br^\top \Xi^{k}\|,
\end{align*}

Using this learned location, we use UCB \cite{chu2011contextual} to learn $\BMtheta$ and make decisions at the next epoch. Specifically, at each round, we update the confidence bound only using the chosen context vector and the corresponding reward $k^*$ rounds ago.  

\subsection{Further Results}

We also perform some further experiments to to complement those in Section~\ref{sec:rich}.

First, we consider the case where the weight vector $\w$ is uniformly randomly generated. Different sparsity of $\w$ is investigated. As shown in Fig.~\ref{fig:randm_w}, our algorithm still performs the best most of the time. Interestingly, it outperforms SA-GD even if $sd$ is much larger than $h+d$. This indicates that leveraging the low-rank structure is problematic. These observations, again, support our discussions about circulant matrices in low-rank matrix recovery in Sections~\ref{challenges} and \ref{sec:rich}. 

\begin{figure*}[ht]
	\centering
	\includegraphics[scale=0.35]{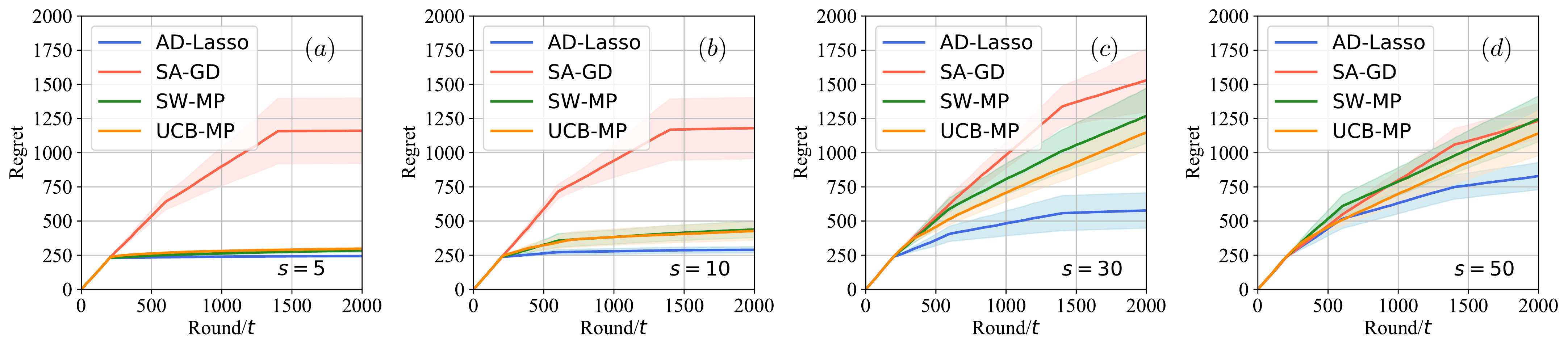}
	\caption{Performance comparison of different algorithms. Here, the $s$-sparse vector $\w$ is uniformly randomly generated and then normalized. (Universal parameters: $T=2000,h=100,d=5$, and $\|\w\|_1=1$. Shaded areas represent standard error.)}
	\label{fig:randm_w}
\end{figure*}

Further, we investigate the situation where there is only $1$ nonzero entry in $\w$. This case corresponds to contextual linear bandits with unknown delays. It can be seen from Fig.~\ref{fig:only_1_w} that our algorithm performs as nearly well as SW-MP and UCB-MP, indicating the power of our algorithm in locating unknown delays. As expected, low-rank recovery algorithm  performs poorly. 

\begin{figure*}[ht]
	\centering
	\includegraphics[scale=0.35]{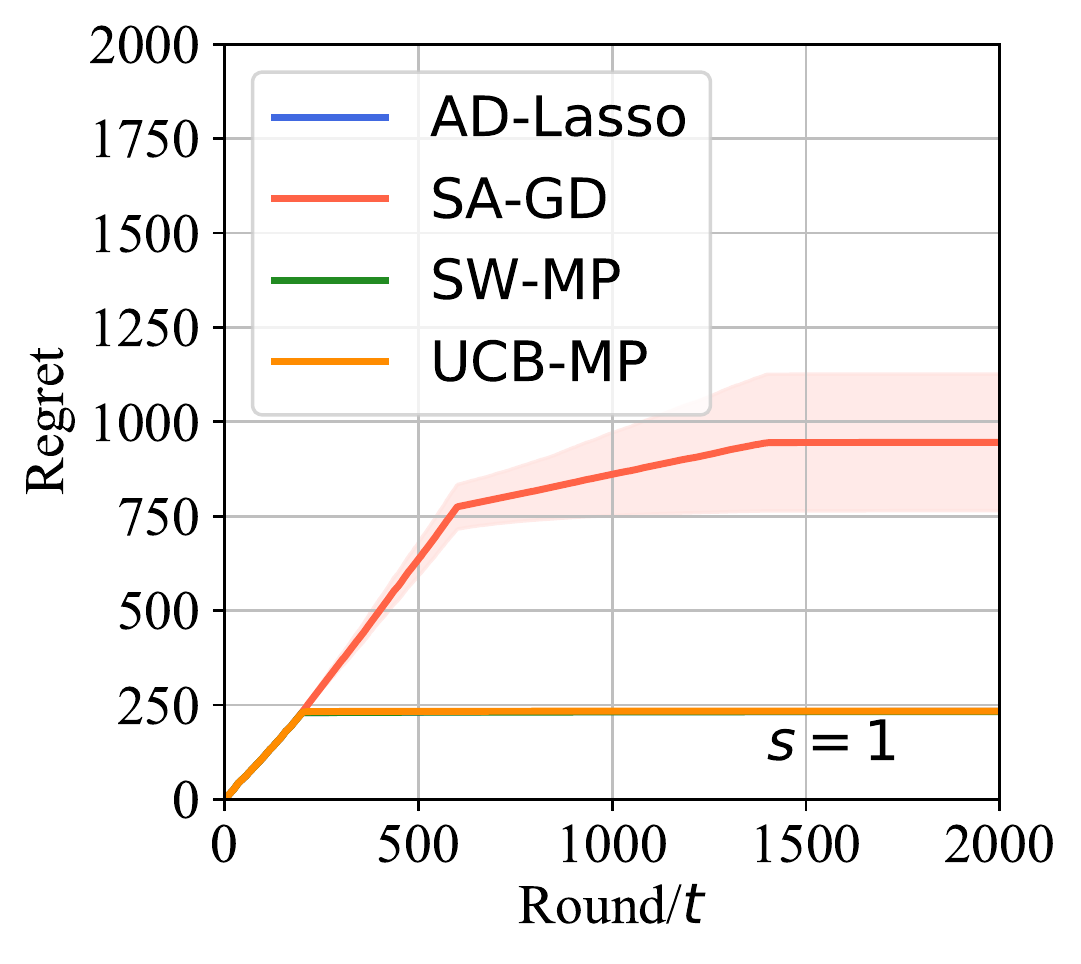}
	\caption{Performance comparison of different algorithms. Here, $\w$ contains only one non-zeros entries. ($T=2000,h=100,d=5$, and $\|\w\|_1=1$.)}
	\label{fig:only_1_w}
\end{figure*}

\end{document}